\documentclass[11pt]{article} 

\usepackage[letterpaper,margin=1in]{geometry}

\usepackage[letterpaper,margin=1in]{geometry}

\newif\ifred
\redtrue 

\usepackage[T1]{fontenc}
\usepackage{microtype}
\usepackage{verbatim}
\usepackage{titlesec}
\titleformat*{\paragraph}{\bfseries}

\usepackage{booktabs}
\usepackage{multirow}

\usepackage{amsthm,amsmath,amssymb,amsfonts,amssymb,mathtools}
\usepackage{amsmath,amssymb,amsfonts,amssymb,mathtools}
\usepackage{xcolor}
\usepackage{empheq}
\usepackage{dsfont}
\usepackage{mathrsfs}
\usepackage{graphicx}
\usepackage{enumitem}
\usepackage{aliascnt} 
\usepackage{xspace}
\usepackage{bbm}

\usepackage{bm}

\usepackage[most]{tcolorbox}

\newtcolorbox{openquestionbox}{
  colback=gray!10,
  colframe=black,
  boxrule=0.8pt,
  arc=3mm,
  left=6pt,
  right=6pt,
  top=6pt,
  bottom=6pt
}

\newcommand{\inote}[1]{\footnote{{\bf [[Ilias: {#1}\bf ]] }}}
\newcommand{\dnote}[1]{\footnote{{\bf [[Daniel: {#1}\bf ]] }}}

\usepackage{caption}
\usepackage{tikz}
\usetikzlibrary{patterns}
\usetikzlibrary{patterns}
\usetikzlibrary{arrows,shapes,automata,backgrounds,petri,positioning}
\usetikzlibrary{shadows}
\usetikzlibrary{calc}
\usetikzlibrary{spy}
\usetikzlibrary{angles, quotes}
\usetikzlibrary{arrows.meta}

\usetikzlibrary{matrix}
\usepackage{pgf,pgfplots}
\usepackage{pgfmath,pgffor}
\pgfplotsset{compat=1.17}

\usepackage{framed}
 \usepackage{algorithm}
\usepackage[noend]{algpseudocode}

\definecolor[named]{ACMBlue}{cmyk}{1,0.1,0,0.1}
\definecolor[named]{ACMYellow}{cmyk}{0,0.16,1,0}
\definecolor[named]{ACMOrange}{cmyk}{0,0.42,1,0.01}
\definecolor[named]{ACMRed}{cmyk}{0,0.90,0.86,0}
\definecolor[named]{ACMLightBlue}{cmyk}{0.49,0.01,0,0}
\definecolor[named]{ACMGreen}{cmyk}{0.20,0,1,0.19}
\definecolor[named]{ACMPurple}{cmyk}{0.55,1,0,0.15}
\definecolor[named]{ACMDarkBlue}{cmyk}{1,0.58,0,0.21}

\usepackage[colorlinks,citecolor=blue,linkcolor=magenta,bookmarks=true]{hyperref}
 \usepackage[nameinlink]{cleveref}
\creflabelformat{ineq}{#2{\upshape(#1)}#3}
\crefname{sub}{Subsection}{Subsection}
\creflabelformat{Subsection}{#2{\upshape(#1)}#3}
\crefname{sdp}{SDP}{SDP}
\creflabelformat{sdp}{#2{\upshape(#1)}#3}
\crefname{lp}{LP}{LP}
\creflabelformat{lp}{#2{\upshape(#1)}#3}
\usepackage{aliascnt}
\crefname{ineq}{Inequality}{Inequality}
\creflabelformat{ineq}{#2{\upshape(#1)}#3}
\crefname{sub}{Subsection}{Subsection}
\creflabelformat{Subsection}{#2{\upshape(#1)}#3}
\crefname{sdp}{SDP}{SDP}
\creflabelformat{sdp}{#2{\upshape(#1)}#3}
\crefname{lp}{LP}{LP}
\creflabelformat{lp}{#2{\upshape(#1)}#3}



\newtheorem{theorem}{Theorem}[section]
\newtheorem{lemma}[theorem]{Lemma}

\newtheorem{definition}[theorem]{Definition}

\newtheorem{fact}[theorem]{Fact}
\newtheorem{question}{Open Question}
\theoremstyle{definition}

\newtheorem{remark}[theorem]{Remark}

\crefname{claim}{claim}{claims}

\newcommand{\eqdef}{\coloneqq}

\newcommand{\vol}{\mathrm{Vol}}

\newcommand{\lp}{\left}
\newcommand{\rp}{\right}
\newcommand\norm[1]{\left\| #1 \right\|}

\renewcommand\vec[1]{\mathbf{#1}}
\DeclareMathOperator*{\pr}{\mathbf{Pr}}
\DeclareMathOperator*{\E}{\mathbf{E}}
\newcommand{\proj}{\mathrm{proj}}

\DeclareMathOperator*{\argmin}{argmin}
\DeclareMathOperator*{\argmax}{argmax}

\newcommand{\bx}{\mathbf{x}}

\renewcommand{\u}{\mathbf{u}}
\newcommand{\bz}{\mathbf{z}}

\newcommand{\err}{\mathrm{err}}

\newcommand{\B}{\mathbb{B}} 
\renewcommand{\H}{\mathcal H }

\newcommand{\R}{\mathbb{R}}
\newcommand{\s}{\mathbb{S}}
\newcommand{\Z}{\mathbb{Z}}
\newcommand{\N}{\mathbb{N}}
\newcommand{\eps}{\epsilon}
\newcommand{\dtv}{d_{\mathrm TV}}
\newcommand{\poly}{\mathrm{poly}}

\newcommand{\polylog}{\mathrm{polylog}}
\newcommand{\var}{\mathbf{Var}}

\newcommand{\sign}{\mathrm{sign}}

\newcommand{\D}{D}

\newcommand{\Ind}{\mathds{1}}

\newcommand{\x}{\vec x}
\newcommand{\w}{\vec w}
\newcommand{\g}{\vec g}
\renewcommand{\a}{\vec a}
\renewcommand{\b}{\vec b}
\renewcommand{\v}{\vec v}

\newcommand{\opt}{\mathrm{opt}}

\newcommand{\iid}{{i.i.d.}\ }
\newcommand{\abs}[1]{\lp| #1 \rp|}
\newcommand{\A}{\mathcal{A}}

\newcommand{\card}[1]{|#1|}

\def\colorful{1}

\ifnum
\newcommand{\new}[1]{{\color{red} #1}}

\else
\newcommand{\new}[1]{{#1}}

\fi

\renewcommand\Pr{\pr}

\newcommand{\twolines}[2]{%
\begin{tabular}{@{}l@{}}
#1 \\ #2
\end{tabular}
}

\begin{document}
\title{Robust Learning of Multi-Index Models under Logconcave Distributions 
via Weak Partial Learners}

%


\author{
Ilias Diakonikolas\thanks{Supported by NSF Medium Award CCF-2107079, 
ONR Award number N00014-25-1-2268, and a Guggenheim Fellowship.}\\
University of Wisconsin-Madison\\
\texttt{ilias@cs.wisc.edu}
\and
Daniel M. Kane\thanks{Supported by NSF Medium Award CCF-2107547.}\\
University of California, San Diego\\
\texttt{dakane@cs.ucsd.edu}
\and
Mingchen Ma\thanks{Supported by NSF Award  CCF-2144298 (CAREER).}\\
University of Wisconsin-Madison\\
\texttt{mingchen@cs.wisc.edu}
}

\maketitle

\begin{abstract}%

We develop a new algorithmic framework for learning Boolean concept classes
under arbitrary logconcave distributions. Our framework applies
to broad families of multi-index models and yields significantly faster robust learners
than were previously known in these settings.
It also extends naturally to the settings of partial classification and testable learning.
As consequences, we obtain substantially improved learners
for several fundamental concept classes, including functions of halfspaces
and multiclass linear classifiers, under logconcave marginals.

Our starting point is a conceptually new boosting framework based on \emph{weak partial learners}.
Classical boosting-based approaches either require distribution-free weak learners,
which are often computationally intractable, or yield only weak approximation guarantees.
We show instead that it suffices to design weak learners for \emph{partial} classification problems,
where labels may be undefined on a nontrivial fraction of the domain.
This provides a robust and computationally tractable route to strong agnostic learning guarantees.
The main technical challenge is that boosting requires a weak learner that succeeds
under conditionings of the initial distribution that are not guaranteed to be logconcave.
To address this, we develop a recursive moment-based partitioning procedure:
given a logconcave distribution \(X\) and a conditional version \(X'\),
we partition the space into polynomially many convex regions so that,
on most regions, the low-degree central moments of \(X\) and \(X'\) approximately agree.
We show that in any good region, a well-behaved multi-index model 
will correlate with a low-degree polynomial, which provides a weak hypothesis on that region.
Combining these weak hypotheses over the regions yields the required weak learner.

A key ingredient of our approach is an efficient method for detecting moment discrepancy.
Specifically, we show how to find a low-dimensional subspace
that captures any non-trivial discrepancy between \(X\) and \(X'\), if one exists.
Our approach combines spectral analysis of moment tensors with recent progress on isoperimetry
for logconcave measures.
This discrepancy-finding procedure may be of independent interest.
\end{abstract}

\thispagestyle{empty}

\newpage

\setcounter{page}{1}

\section{Introduction}

The main motivation of this work was to develop efficient learners 
for fundamental classes of Boolean and multiclass functions---including multiclass linear classifiers and intersections of halfspaces---under mild distributional assumptions
and potentially 
in the presence of adversarial label noise. In the process, 
we developed a new general learning algorithm 
that is both label noise-tolerant and significantly more efficient 
than the state of the art. 
Before we describe our contributions in detail, we start with some background on specific settings that motivated this work. 

\paragraph{Multi-Index Models and Efficient Learnability} 
A function is called a Multi-Index Model (MIM) if it depends only on the projection 
of its inputs onto a ``low-dimensional'' subspace. 
Formally, we have: 


\begin{definition}[Multi-Index Model] \label{def:mim}
A function $f: \R^d \to \cal \mathcal{Y}$ is called a $K$-MIM 
if there exists a subspace $U \subseteq \R^d$ of dimension at most $K$ 
such that $f$ depends only on {the projection onto} $U$, 
i.e., $f(\x) = f ( \proj_U \x)$,  
where $\proj_U \x$ is the projection of $\x$ on $U$. Equivalently, 
$f$ is of the form $f(\x) = g(\mathbf{W} \x)$ for a link function $g: \R^K \to \cal{Y}$ 
and a $K \times d$ weight matrix $\mathbf{W}$.
\end{definition}

In several settings of interest, 
the dimension $K$ of the hidden subspace is 
significantly smaller than the ambient dimension $d$  
and is often thought of as bounded above by a constant.
Throughout this work, we focus on the case that the range of our functions is discrete, in most cases Boolean.

Multi-index models~\cite{Friedman:1980tu, Huber85-pp, Li91, 
HL93,xia2002adaptive, Xia08} are motivated by the typical assumption
that real high-dimensional data are inherently structured, specifically 
that the relationship between the features is of a lower-dimensional nature.
MIMs capture several well-studied concept classes, both discrete and real-valued, 
including halfspaces, functions of halfspaces, multiclass linear classifiers, 
low-dimensional polynomial threshold functions, and natural classes of neural 
networks. The efficient learnability of each of these concept classes has been
a topic of intense investigation in the theoretical computer science and machine 
learning communities over several decades. 
More recently, there has been a flurry of research activity aiming at developing
general algorithmic techniques that are applicable to broad families 
of MIMs---albeit under strong distributional assumptions.
This progress notwithstanding, our understanding of the complexity of learning MIMs 
(both for specific special cases and in more general settings) remains limited. 
This motivates the general research direction explored in this work:

\vspace{-0.1cm} 

\begin{openquestionbox}
\begin{center}
{\em  Can we develop efficient and noise-tolerant learners for broad classes of MIMs\\ 
under mild distributional assumptions?
}
\end{center}
\end{openquestionbox}

\vspace{-0.1cm}


Our main result answers this question in the affirmative, 
when the underlying distribution is logconcave and the labels are adversarially corrupted. 
Additionally, our algorithmic results extend to {\em partial} classification and 
to the {\em testable learning} framework.

For concreteness, we define agnostic PAC learning under logconcave distributions.

\begin{definition}[Agnostic  Classification under Logconcave Marginals] \label{def:agnostic-lc} 
Let $\mathcal Y$ be finite and $\cal{H}$ be a class of functions 
from $\R^d$ to $\mathcal Y$.
Given i.i.d.\ samples $(\x,y)$ 
from a distribution $D$ on $\R^d\times\mathcal Y$ 
such that (a) the $\x$-marginal, $D_{\x}$, of $D$ is an unknown logconcave distribution,  
and (b) no assumptions are made on the labels, 
the goal is to output a hypothesis 
$h: \R^d \to \mathcal Y$ that, with high probability, 
has small 0-1 error 
$\mathrm{err}_D(h)\eqdef \pr_{(\x,y) \sim D}[h(\x) \neq y]$, 
as compared to the optimal 0-1 error of the class ${\cal H}$, i.e., 
$\opt_D(\mathcal{H}) \eqdef \inf_{f \in \mathcal{H}}\mathrm{err}_D(f)$.
\end{definition}

To demonstrate the limitations of current techniques, we next discuss the state-of-the-art for three basic learning tasks.

\paragraph{Example 1: Robustly Learning Halfspaces}
A halfspace is any function $h: \R^d \to \{ \pm 1\}$ of the form 
$h(\x) = \sign(\vec w \cdot \x+t)$, where $\vec w \in \R^d$ and $t \in \R$.
The task of learning halfspaces is one of the most extensively studied 
problems in machine learning~\cite{Rosenblatt:58, Novikoff:62,
MinskyPapert:68,  FreundSchapire:97, Vapnik:98, CristianiniShaweTaylor:00}. 
While halfspaces are efficiently learnable in the distribution-free (realizable) 
PAC model~\cite{Valiant:84} (see, 
e.g.,~\cite{MT:94}), the algorithmic problem becomes computationally intractable in the 
presence of adversarial label noise~\cite{Daniely16,Tiegel23}.
This has motivated a long line of work on developing efficient robust learners 
under natural distributional assumptions. For the special 
case of {\em homogeneous} halfspaces (corresponding to zero threshold), 
\cite{ABL17} gave a $\poly(d/\eps)$ time learner with 
error $O(\opt)+\eps$ under logconcave marginals, 
where $\opt$ is the optimal error in the class\footnote{A constant-factor approximation is best possible 
for fully-polynomial time algorithms, due to known intractability results even for Gaussian marginals~\cite{DKPZ21,diakonikolas2023near}.}. 
The efficient learnability of {\em general} (i.e., not necessarily homogeneous) 
halfspaces under logconcave marginals has remained open since. 
For the special case of Gaussian marginals,~\cite{DKS18a,diakonikolas2022learning,diakonikolas2024active} gave an efficient constant-factor approximation. However, for the general logconcave case, the best known 
efficient algorithm achieves error $\tilde{\Omega}(\opt^{1/2})+\eps$; 
see, e.g.,~\cite{DKS18a}. 
On the other hand, the $L_1$-polynomial regression algorithm~\cite{KKMS:08} achieves optimal error of $\opt+\eps$ with sample and computational 
complexity $d^{2^{\poly(1/\eps)}}$. Concretely, this leads to the following open question.

\begin{openquestionbox}

\begin{question} \label{q:LTFs}
Is there a robust learner for {\em general} halfspaces under logconcave marginals 
with error $O(\opt)+\eps$ beating the complexity of polynomial regression? 
\end{question}
\end{openquestionbox}



As a corollary of our general algorithmic approach, 
we answer this question in the affirmative (see Theorem~\ref{thm:LTFs-intro}). 

\vspace{-0.25cm}

\paragraph{Example 2: Learning Intersections of Halfspaces}
An intersection of $K$ halfspaces is any  
function $f:\R^d\to\{\pm 1\}$ for which there exist $K$ halfspaces 
$h_i(\x) = \sign(\vec w^{(i)}\cdot\x+t_i)$ such that 
\( f(\x) = 1 \) if and only if \( h_i(\x) = 1 \) for all \( i \in [K]\). 
In sharp contrast to the special case of a single halfspace, our understanding
of the efficient learnability of this class remains limited---even for $K=2$ 
and in the presence of clean labels. Specifically, even for this special case, 
it remains a major open question whether efficient learning is possible 
in the distribution-free setting. The best known distribution-free algorithm 
requires sub-exponential time~\cite{alman2026learningfunctionshalfspaces} 
and prior work has given plausible evidence of hardness; 
see, e.g.,~\cite{KS:08, Sherstov:09}. 
 On the other hand, for $K = O(1)$, efficient learners have been developed 
with respect to natural distributions, including the Gaussian and the 
uniform distribution on the hypercube; see, 
e.g.,~\cite{blum1997learning,KOS:04,arriaga2006algorithmic, KlivansServedio:04coltmargin, KlivansLT09, Vempala10,Vempala10a,DiakonikolasMRT25}. 
One of the strongest known results~\cite{Vempala10} 
(see also~\cite{KlivansLT09}) gives an efficient (realizable) algorithm 
for any constant number of halfspaces under logconcave 
distributions \emph{assuming all component halfspaces pass 
through the mean of the distribution}
\footnote{
This is an important assumption without which their algorithmic approach fails.}. 
The only known upper bound for the general case---even for clean labels---is 
the $L_1$-polynomial regression method, which incurs sample and computational 
complexity $d^{2^{\poly(K/\eps)}}$. Concretely, this leads to the following open question.

\begin{openquestionbox}
\begin{question} \label{q:inter}
Is there a learner for intersections of {\em two} halfspaces under logconcave marginals 
beating the complexity of polynomial regression? 
\end{question}
\end{openquestionbox}

As a corollary of our general approach, 
we answer this question in the affirmative (see Theorem~\ref{thm:functions-LTFs-intro}). 
As a bonus, our algorithm succeeds for any function of $K = O(1)$ many halfspaces, 
and is robust to adversarial label noise (achieving {\em any} desired constant factor approximation).  

\paragraph{Example 3: Robust Multiclass Linear Classification}
A Multiclass Linear Classifier (MLC) is any function 
$f: \R^d \to [K]$ of the form
$f(\x) = \argmax_{i\in [K]}( \w^{(i)} \cdot \x  + t_i)\;,$
where $\w^{(i)}  \in \R^d$ and $t_i \in \R$.
Multiclass linear classification (MLC) is a textbook machine learning problem~\cite{SB-book} generalizing the task of learning halfspaces (corresponding to $K=2$). 
In the realizable PAC model, this task is
efficiently solvable to any desired error $\eps$ 
with sample size $O(dK/\eps)$ via linear programming. 
On the other hand, despite extensive interest, 
see, e.g.,~\cite{NDRT13,WLT18, Lipton18a, Song-etal-survey20, GWBL20,diakonikolas2025statistical,diakonikolas2026polynomial}, 
our understanding of the computational complexity of learning 
remains limited in the presence of adversarial label noise. 
The recent result~\cite{DiakonikolasIKZ25} gave an algorithm 
that succeeds under the Gaussian distribution, has complexity 
$\poly(d) 2^{\poly(K/\eps)}$, and guarantees error of $C \opt+\eps$, 
for some universal constant $C>1$. Note that the Gaussianity assumption 
is crucial. Moreover, the underlying complexity scales exponentially with 
both $K$ and $1/\eps$. As in previous examples, for logconcave distributions, 
the only known algorithm is $L_1$-regression which incurs complexity 
$d^{\poly(K/\eps)}$ even for the special case of the Gaussian distribution. 
 Concretely, this leads to the following open 
question.


\begin{openquestionbox}
\begin{question} \label{q:mlc}
Is there a robust learner for $K$-MLC under logconcave marginals 
beating the complexity of polynomial regression? 
\end{question}
\end{openquestionbox}

We answer this question in the affirmative (see Theorem~\ref{thm:MLC-intro}). 
Specifically, 
we give two algorithms for this task: the first algorithm achieves error $\widetilde{O}(K \opt)+\eps$ 
and runs in $\poly(d, K, 1/\eps)$ time, and the second achieves error $O(K) \, \opt+\eps$ with
sample complexity $\poly(d, K, 1/\eps)$ and running time $\poly(d) \, 2^{\polylog(K/\eps)}$.


\subsection{Our Results} \label{ssec:results}


Before we state our algorithmic contributions in detail, we emphasize that all our upper bounds in this work are obtained via a \emph{single unified algorithm}. In Section~\ref{sssec:gen-result}, we describe our general algorithmic result that is applicable to a broad family of MIMs, and its testable learning extension. In Section~\ref{sssec:app-results}, we describe some concrete algorithmic implications for fundamental concept classes.


\subsubsection{General Algorithmic Result} \label{sssec:gen-result}

Our algorithm can be used to learn essentially any non-pathological $K$-MIM, although its 
sample and computational complexities will depend on 
some complexity measure of the underlying concept class. 
In particular, we can learn any function that admits $L_1$-sandwiching polynomials 
(see Definition~\ref{def:class-sandwich})
with respect to any logconcave distribution. 
Results of~\cite{kane2013learning} imply that this applies 
to any function with bounded effective surface area.




\begin{theorem}[Robust Learner for General MIMs, Informal] \label{thm:gen-intro}
Let $\H$ be a family of $K$-MIMs over $\R^d$,  
admitting $L_1$-sandwiching polynomials achieving error $c>0$ with degree $\ell = \ell(c)$. 
There is an algorithm such that for any distribution $D$ over 
$\R^d \times \{\pm 1\}$ with $D_\x$ logconcave, and any 
$\eps\in (0,1)$, it has the following behavior: the algorithm draws 
$\poly_{\ell}(d) \poly_{K,c}(1/\eps)$ samples from $D$, 
it runs in time $\poly_{\ell}(d) 2^{\polylog_{K, c}(1/\eps)}$, 
and with high probability outputs a hypothesis $\hat{h}$ such that 
$\err_{D}(\hat{h})\le (1+O(c)) \, \opt_D(\H) + \eps$. 
\end{theorem}

\noindent The reader is referred to Theorem~\ref{th sandwhiching} in 
Section~\ref{sec sandwhiching polynomial} for the detailed formal statement 
that also covers the partial learning setting. For our other generic families of functions covered by our main result, see Theorem~\ref{th generic family} in Section~\ref{sec generic class}.


\vspace{-0.3cm}

\paragraph{Discussion} 
We now provide a few remarks regarding Theorem~\ref{thm:gen-intro}. First, we note that 
our general MIM family includes a range of well-studied concept classes, 
including arbitrary functions of $K$ halfspaces, 
$K$-dimensional constant degree PTFs, and others. 
Specifically, all of these classes admit sandwiching polynomials whose degree $\ell$ 
is only a function of $K$ and the desired accuracy.
Therefore, as long as $K = O(1)$, Theorem~\ref{thm:gen-intro} implies a constant-factor robust learner for each of these classes with sample complexity $\poly(d/\eps)$ and computational complexity $\poly(d) 2^{\polylog(1/\eps)}$. 

An additional feature of our result is that (for concept classes with the $L_1$-
sandwiching property) we can obtain {\em any} desired constant factor approximation 
by increasing the complexity by a polynomial factor. In particular, this allows us 
to obtain a Quasi-PTAS for the agnostic learning task (with quasi-polynomial dependence on 
$\eps$ in the runtime). It is also worth noting that our general algorithm is implementable in the 
Statistical Query (SQ) model~\cite{Kearns:98} with complexity (i.e., number of queries and inverse of 
accuracy) $\poly_{K}(d/\eps)$. The quasi-polynomial dependence on $\eps$ in the runtime is 
due to a brute-force search over a space of dimension $\polylog_K(1/\eps)$.




\vspace{-0.3cm}

\paragraph{Extension to Testable Learning}
Interestingly, the algorithm establishing Theorem~\ref{thm:gen-intro} also applies to the 
testable learning framework of~\cite{rubinfeld2023testing}, 
as long as the marginal logconcave distribution is known to the algorithm. 
In testable learning, the goal is to be able to test whether the underlying distributional assumptions for your algorithm are met (or met well enough) rather than having to assume them. In particular, a testable learner should consist of  (see Definition~\ref{def:testable}) 
a tester-learner pair satisfying the following properties: 
(a) if the tester accepts, the learner outputs a hypothesis 
and a certificate that it achieves near-optimal error, and 
(b) it is highly unlikely that the tester rejects if the data satisfies 
the underlying distributional assumptions. 
The reader is referred to Theorem~\ref{th testable} in Section~\ref{sec proof main} for our generic 
testable learning result. This result provides significantly more efficient testable learners
for a range of concept classes with near-optimal error guarantees.

\vspace{-0.3cm}


\paragraph{Extension to Partial Learning}
Furthermore, our results also apply to agnostic partial learning, a stronger but natural extension of the standard agnostic learning setting. In agnostic partial learning, we receive samples $(\x,y)$ from a distribution $D$ over $\R^d \times \{0,\pm 1\}$, with logconcave marginal, 
where a sample with $y=0$ can be thought of as an example that goes unclassified. 
The goal of an agnostic partial learner is to output 
a hypothesis that does well over the examples that are classified in our data. In other words, even if only a small fraction of the examples are classified (whose marginal is logconcave conditioned on an event with a suitably small probability),
our algorithm is still able to learn with qualitatively the same complexity (with a 
necessary dependence on the inverse of the labeled mass). See \Cref{thm: main-weak-to-strong}.

\subsubsection{Algorithmic Applications for Concrete Concept Classes} \label{sssec:app-results}

Our first algorithmic application is for the task of robustly learning general halfspaces under logconcave marginals, thereby answering \Cref{q:LTFs} in the affirmative.

\begin{theorem}[Robustly Learning General Halfspaces] \label{thm:LTFs-intro}
Let $\H$ be the class of halfspaces on $\R^d$.
There are two algorithms for learning distributions $D$ over $\R^d \times \{\pm 1\}$ 
with unknown logconcave $\x$-marginal $D_\x$. 
Each draws $\poly(d/\eps)$ samples from $D$, 
and has the following performance guarantees:
\begin{enumerate}[leftmargin=*]
\item The first algorithm runs in time $\poly(d/\eps)$ and with high probability 
outputs a hypothesis with error $\widetilde{O}(\opt)+\eps$, where $\opt = \opt_D(\H)$. 
\item The second algorithm runs in time $\poly(d) 2^{\polylog(1/\eps)}$ and with high 
probability outputs a hypothesis with error $C \opt+\eps$, where $C>1$ is a universal constant. 
\end{enumerate}
\end{theorem}

The reader is referred to Theorems~\ref{th halfspace quasi}
and~\ref{th halfspace polytime} in Section~\ref{sec partial linear} 
for more detailed formal statements that also cover the partial PAC learning setting.

Moreover, the testable version of Theorem~\ref{thm:LTFs-intro} answers another 
basic open problem in the literature. Specifically, even for the special case of Gaussian marginals, the best known polynomial-time error guarantee for testable learning of halfspaces was $\tilde{\Omega}(\sqrt{\opt})+\eps$~\cite{DKLZ24}. Our learner works for any fixed logconcave distribution and achieves near-optimal 
error of $\tilde{O}(\opt)+\eps$ 
(or $O(\opt)+\eps$ with a quasi-polynomial time dependence on $1/\eps$).

\medskip

Our second algorithmic application is for the task of learning functions of halfspaces under logconcave marginals, thereby answering \Cref{q:inter} in the affirmative.

\begin{theorem}[Robustly Learning Functions of Halfspaces] \label{thm:functions-LTFs-intro}
Let $\mathcal{H}_{K, d}$ be the concept class of functions of $K$ halfspaces on $\R^d$.  
There is an algorithm that for every distribution $D$ 
over $\R^d \times \{\pm 1\}$ 
with unknown logconcave $\x$-marginal $D_\x$ and any given $c>0$, 
draws  $\poly_{K,c}(d/\eps)$ samples from $D$, 
runs in time $\poly_{K,c}(d) 2^{\polylog_{K,c}(1/\eps)}$, 
and with high probability outputs a hypothesis with error $(1+c) \opt+\eps$, 
where $\opt = \opt_D(\H_{K, d})$. 
\end{theorem}

The reader is referred to Theorem~\ref{th functions of halfspaces} 
in Section~\ref{ssec:qptas-apps} 
for a more detailed formal statement that also 
covers the partial PAC learning setting.

We emphasize that the upper bound of Theorem~\ref{thm:functions-LTFs-intro} is new even for 
the very special case of learning intersections of two halfspaces in the realizable PAC 
model (i.e., $K=2$ and $\opt=0$). As a bonus, the upper bound applies to any function of any 
constant number of halfspaces and is robust to adversarial label noise. Specifically, 
for $K = O(1)$ we obtain a constant factor approximation to the optimal error 
with sample complexity $\poly(d/\eps)$ and computational complexity 
$\poly(d/\eps) 2^{\polylog(1/\eps)}$. Moreover, we can achieve a factor $(1+c)$ 
approximation, for any desired parameter $c>0$, by increasing the complexity 
as a function of $c$. 


\medskip

Our third algorithmic application, obtained via a refinement of our generic technique, 
is for the task of robustly learning multiclass linear classifiers under logconcave marginals, 
thereby answering \Cref{q:mlc} in the affirmative.

\begin{theorem}[Robustly Learning Multiclass Linear Classifiers] \label{thm:MLC-intro}
There are two algorithms for learning the concept class $\mathcal{M}_{K, d}$ of $K$-MLC on $\R^d$. 
Each algorithm draws $\poly(d \, K/\eps)$ samples from a distribution $D$ over $\R^d \times [K]$ 
with unknown logconcave $\x$-marginal $D_\x$, and has the following performance guarantees:
\begin{enumerate}[leftmargin=*]
\item The first algorithm runs in time $\poly(d\, K/\eps)$ and with high probability 
outputs a hypothesis with error $\widetilde{O}(K) \, \widetilde{O}(\opt)+\eps$, 
where $\opt = \opt_D(\mathcal{M}_{K, d})$. 
\item The second algorithm runs in time $\poly(d) 2^{\polylog(K/\eps)}$ and with high probability 
outputs a hypothesis with error $O(K) \, \opt+\eps$.  
\end{enumerate}
\end{theorem}

The reader is referred to Theorem~\ref{th multiclass}
 in Section~\ref{sec multiclass} 
for a more detailed formal statement 
that also covers the partial PAC learning setting.
For comparison, the prior work~\cite{DiakonikolasIKZ25} gave an algorithm 
for the special case of the Gaussian distribution with complexity 
$\poly(d) 2^{\poly(K/\eps)}$, and error of $O(\opt)+\eps$. 
Our algorithm achieves a somewhat weaker robustness guarantee 
with substantially better complexity and under much weaker distributional assumptions.



\begin{table}[t]
\centering
\small
\setlength{\tabcolsep}{2pt}
\renewcommand{\arraystretch}{1.3}

\begin{tabular}{l |l| l| l| l| l}
\toprule
\textbf{Result} & \textbf{Function Class} & \textbf{Distribution} & \textbf{Time Complexity} & \textbf{Error} & \textbf{Notes} \\
\midrule

Theorem~\ref{th sandwhiching}
& \twolines{$L_1$-sandwiching}{$K$-MIMs} 
& Logconcave
& $d^{O(\ell(c))} 2^{\polylog_{K,c}(1/\eps)}$  
& $ (1+c)\opt+\eps$ 
& \ Partial, Testable   \\ \hline
\cite{KKMS:08}
& \twolines{$L_1$-Approximable}{Classes}  
& General 
& $d^{O(\ell(\eps))}$ 
& $ \opt+\eps$ 
& \    \\ \hline
Theorem~\ref{th functions of halfspaces}
& \twolines{Functions of}{$K$ Halfspaces} 
& Logconcave
&  $\poly_{K,c}(d) 2^{\polylog_{K,c}(1/\eps)}$  
& $ (1+c)\opt+\eps$ 
& \ Partial, Testable   \\ \hline
\cite{kane2013learning}
& \twolines{Functions of}{$K$ Halfspaces} 
& Logconcave
& $d^{2^{\poly(K/\eps)}}$ 
& $ \opt+\eps$ 
& \    \\ \hline
\cite{Vempala10}
& \twolines{Intersections of $K$}{Centered Halfspaces} 
&  \twolines{Isotropic}{Logconcave}  
& $O(dK/\eps)^K$ 
& $ \eps$ 
& \ Realizable   \\ \hline
\cite{DKS18a}
& \twolines{Intersections of $K$}{General Halfspaces} 
&  Gaussian   
& $\poly(dK/\eps)+(K/\eps)^{O(K^2)}$ 
& $ \tilde O(K^{4/11}\opt^{1/11})+\eps$ 
&    \\ \hline
\cite{DiakonikolasIKZ25}
& \twolines{Intersections of $K$}{General Halfspaces} 
& Gaussian
& $\poly(d)2^{\poly(K/\eps)}$ 
& $ K\tilde O(\opt)+\eps$ 
& \   \\ \hline
Theorem~\ref{th multiclass}
& $K$-MLC 
& Logconcave
& $\poly(dK/\eps)$ 
& $ \tilde O(K\opt)+\eps$ 
& \ Partial, Testable  \\ \hline
\cite{DiakonikolasIKZ25}
& $K$-MLC 
& Gaussian
& $\poly(d)2^{\poly(K/\eps)}$ 
& $ O(\opt)+\eps$ 
& \   \\ \hline

Theorem~\ref{th halfspace polytime}
& General LTFs 
& Logconcave 
& $\poly(d/\eps)$ 
& $\tilde O(\opt)+\eps$ 
& \ Partial, Testable  \\ \hline

\cite{ABL17} 
& Homogeneous LTFs 
& Logconcave 
& $\poly(d/\eps)$ 
& $O(\opt)+\eps$ 
& \     \\ \hline

\cite{DKS18a} 
& General LTFs 
& Logconcave 
& $\poly(d/\eps)$ 
& $\tilde O (\opt^{1/2})+\eps$ 
& \  \twolines{$O(\opt)+\eps$}{Gaussian}   \\ \hline

\cite{DKLZ24} 
& General LTFs
& Gaussian 
& $\poly(d/\eps)$ 
& $\tilde O(\opt^{1/2})+\eps$ 
& \ Testable  \\
\bottomrule
\end{tabular}

\caption{Summary and comparison of some of our results on agnostic learning under logconcave distributions. All our theorems listed in the table have sample complexity polynomial in $d/\eps$ with degree depending on $K,c,$ or $\ell(c)$.}

\label{tab:results}
\end{table}

\subsection{Technical Overview} \label{ssec:techniques}


In this section, we give a detailed overview of the techniques developed in this paper.

\paragraph{Weakly Learnable Classes}

A natural method to address such learning problems is the method of moments. 
We note that for nice families of distributions, 
there always exists a low-degree polynomial $p$ so that its correlation with the optimal hypothesis $h^*$,
$\E[p(\x)h^*(\x)]$ is not too small. As we can easily compute moments of $D$, 
it is not hard to find such a $p$ if it exists. Furthermore, thresholding this polynomial can be used as a weak learner---a method to produce a hypothesis with error rate 
slightly less than $1/2$.
Furthermore, this weak learner can be made robust in two ways. 
For one, as long as the noise level over the input distribution $D$ 
is not too large, $p(\x)$ also has non-trivial correlation with the noisy label $y$, 
and thus yields a weak hypothesis over $D$.
Furthermore, one can often somewhat relax the condition on the marginal distribution. 
In particular, so long as the marginal is ``close'' to logconcave in a reasonable sense, 
that should suffice to produce this weak learner. 
In particular, for our purposes, it will be enough to have 
a distribution whose \emph{central moments} nearly match those of a logconcave distribution.
These observations inspire us to consider a broad family of classes, 
formalized in Definition~\ref{def moment-based learner} and termed 
{\em $(k,\Delta,\beta,C)$-weakly learnable} classes. 
Roughly speaking, a class is weakly learnable if given any distribution $D$ with low noise rate, whose first $k$ central moments approximately match those of some logconcave distribution, one can efficiently compute a hypothesis whose error is slightly better than $1/2$.
As we will see in Theorem~\ref{th generic family}, essentially any reasonable class of multi-index models satisfies this assumption for some setting of the parameters.

\paragraph{Agnostic Partial Learning and Boosting from Weak Partial Learners}
Given that we assume above that our function class 
has some kind of weak learner, it is natural to attempt 
to use a boosting algorithm to strengthen this to a stronger learner.
Unfortunately, in order for boosting to work, one typically needs a weak learning algorithm that works not just over the original marginal distribution, but also over the marginals of distributions used by the boosting algorithm, which can be much less structured than a logconcave distribution.
Fortunately, if we use a smooth boosting algorithm, 
see, e.g.,~\cite{Servedio:03jmlr, gavinsky2003optimally,Feldman:10ab,KalaiK09}, 
these marginals will always consist of the original distribution 
conditioned on some not-too-unlikely event (see Fact~\ref{th distribution specific boosting}).
This motivates us to consider a stronger but natural generalization 
of the original problem that we term agnostic {\em partial} learning (see Definition~\ref{def partial learning}). 
In agnostic partial learning, we receive samples $(\x,y)$ from a distribution $D$ over $\R^d \times \{0,\pm 1\}$, with logconcave marginal, 
where a sample with $y=0$ can be thought of as an example that goes unclassified. 
The goal of an agnostic partial learner is to output 
a hypothesis that does well over the examples that are classified in our data. 
We show that, to do boosting, it suffices to have a \emph{weak partial learner
(see Definition~\ref{def:partial-weak}). That is a learner that will always return a hypothesis with error somewhat less than $1/2$ when given a partial classification problem with error not too big and the fraction of non-zero labels not too small.}
In Theorem~\ref{th condition boosting}, we show that given a weak partial learner that works when the noise rate is less than $1/C$, we can use it to produce a hypothesis with error $(C/2)\opt+\eps$.


\paragraph{Efficient Weak Partial Learning via Space Partitioning and Discrepancy Finder}
Unfortunately, even if our function class is itself weakly learnable, 
this does not mean we can immediately obtain a weak partial learner. 
This is because if we let $\x$ and $\x'$ be the marginals over all samples and over classified samples, respectively, we are not guaranteed that 
they have nearly matching moments. 
Without this property, it is not clear that our class 
being $(k,\Delta,\beta,C)$-weakly learnable suffices to produce a weak hypothesis. 
As a result, simply running a weak learning algorithm on the classified data 
does not in general suffice. 
This would only work if $\x$ and $\x'$ had nearly matching central moments.

So what can we do if these moments do not match?
In this case, there must be a direction $\v$ 
and an integer $i \le k$ such that the expectations of 
$(\v \cdot \x)^i$ and $(\v \cdot \x')^i$ substantially disagree.
This means that the distributions of $(\v\cdot\x)$ and $(\v \cdot \x')$ must be somewhat different.
This, in turn, implies that if we can find such a direction $\v$ through some efficient 
algorithm---that we term {\em discrepancy finder} (see Definition~\ref{def direction finding oracle})---then
we can split the entire space into two halves  
such that the fraction of classified samples differs substantially across halves.
After the split, we can recursively run the above procedure 
over each half until a direction with moment discrepancy cannot be found; 
or we reach a part, where the probability that $\x'$ falls into it is too small. 
By doing this repeatedly, we partition the whole space into many convex regions $R$, 
with $\Pr[y\neq 0 \mid R]$ starting to become more and more extreme. 
However, since $\x'$ is obtained from $\x$ by conditioning on some event 
of probability $\alpha$, the ratio $\Pr[\x'\in R]/\Pr[\x\in R]$ will never exceed $1/\alpha$.
This fact, along with a careful analysis of the partitioning scheme, 
allows us to show that we end up with a polynomial number of regions, 
so that most of the mass of $\x$ belongs to regions where the conditional moments 
of $\x$ and $\x'$ nearly match (see Lemma~\ref{lm partition}).
As a result, by calling the two algorithmic modules---namely, the weak learner 
and the discrepancy finder---polynomially many times, we get a weak partial hypothesis over $D$, 
as long as the noise rate is not too high and the mass 
of non-zero labels is not too small (see Theorem~\ref{th decision tree k moment}).

\paragraph{Discrepancy Finder under Logconcave Distributions via Dimension Reduction}
Given the discussion thus far, we already have the ideas necessary 
to design an efficient robust learning algorithm in the Statistical Query model, 
where given the moments of $\x$ and $\x'$ we can existentially determine 
if there is a direction of significant mismatch. 
Unfortunately, computational difficulties arise in designing 
an actual algorithm for the discrepancy finder.
In particular, given distributions $\x$ and $\x'$,
the discrepancy finder is asked to either certify that $(\v\cdot\x)$ and $(\v\cdot\x')$ 
approximately match their moments for each $\v$, 
or to find a direction $\v$, in which these moments differ substantially.
While it is not difficult to approximate the expectation of $(\v\cdot\x)^k-(\v\cdot\x')^k$ 
as an explicit polynomial in $\v$,
finding an approximate maximum of this polynomial is believed to be computationally hard in general. 
Fortunately, the structure of the distributions in question will allow us to 
build such a discrepancy finder in quasi-polynomial time (see Theorem~\ref{th direction finding}). 
We provide an overview of our algorithm to do so in the following paragraphs.

Our first observation is that if $\x$ is sampled from a logconcave distribution 
and $\x'$ is $\x$ conditioned on an event of probability at least $\alpha$, 
then their moments cannot differ arbitrarily. In particular, 
we show in Lemma~\ref{lm Q function learning} that there must be a degree-$k$ polynomial 
$Q$ with relatively small $L_2$ norm over $\x$ 
such that for any degree-$k$ polynomial $p$, $\E[p(\x)Q(\x)] = \E[p(\x)-p(\x')]$. 
How can we exploit this fact?
For the purpose of motivation, we first cover the simpler case 
where $\x$ is drawn from a standard Gaussian distribution. 
In this case, we can expand $Q(\x)$ in the Hermite basis, yielding a sequence of tensors $\mathbf{T}_0,\dots,\mathbf{T}_k$ with small Frobenius norm, 
so that for every $\v \in \s^{d-1}$ we can write $\E_{\x \sim D_\x}[Q(\x)(\v \cdot \x)^k]$ 
as a linear combination of $\mathbf{T}_k\cdot \v^{\otimes i}$ for $i\in [k]$.
Then, viewing $\mathbf{T}_i$ as a bilinear form over $\R^d \times \R^{d^{i-1}}$, 
and using a singular-value decomposition, one can efficiently find a subspace 
$V$ (of $\poly(\log(1/\alpha)/\Delta)$ dimension) 
so that the moment-discrepancy in any direction $\u$ is not too much larger than 
the discrepancy in the direction $\v = \mathrm{proj}_V(\u).$ We can then spend time $2^{O(\dim(V))}$ (which is $\mathrm{quasipoly}(1/\alpha)$ if $\Delta$ is constant) to do an exhaustive search over $V$ for directions of large discrepancy.

Unfortunately, the above argument critically exploits 
many features of the Gaussian distributions 
(such as the fact that moments in orthogonal directions are independent) 
that do not hold for a general logconcave distribution.
However, the basic idea of dimension-reduction can still be made to work.
In particular, we wish to find a subspace $V$ with a not-too-large dimension 
so that the maximum of the multi-linear form 
$\E[Q(\x)\prod_{i=1}^k(\u^{(i)}\cdot\x)]$ over any unit vectors $\u^{(i)}$ 
is at most a constant multiple of the maximum 
of $\E[Q(\x)\prod_{i=1}^k(\v^{(i)}\cdot\x)]$, with $\v^{(i)} \in V$, 
plus a small error $\Delta$ (see Lemma~\ref{lm subspace v3}).
We achieve this in two steps. 
First, we notice that $\E[Q(\x)p(\x)]$ (as a linear function of $p(\x)$) 
can be written as a linear function of the expectations of $p$'s derivative $\nabla^i p(\x)$.
Making use of recent progress on the KLS conjecture, 
we can show that the size of the corresponding tensors cannot be too large. 
Combining this fact with a spectral analysis, 
we give an efficient algorithm (see Lemma~\ref{lm subspace finding v1})
to find a subspace $U$ so that if all of the vectors $\u^{(i)}$s are orthogonal to $U$, then $\E[Q(\x)\prod_{i=1}^k(\u^{(i)}\cdot\x)]$ is absolutely small. 

The next step is to use this as a subroutine to recursively build our target subspace $V$. 
To do this, we decompose $\u^{(i)} = \u^{(i)}_U + \u^{(i)}_{U^\perp}$. 
This allows us to decompose $\E[Q(\x)\prod_{i=1}^k(\u^{(i)}\cdot\x)]$ into $2^k$ terms 
$\E[Q(\x)\prod_{i=1}^k(\v^{(i)}\cdot\x)]$, with each $\v^{(i)}$ either in $U$ or in $U^\perp$. 
As a result, we only need to consider the case where one of the vectors is in $U$. 
Let $\g^{(1)},\dots,\g^{(N)}$ be an orthonormal basis of $U$. 
If we apply the subroutine in the first step to $Q(\x)(\g^{(i)}\cdot\x)$ for each $i$, 
this allows us to get another subspace $U'$ such that when each $\u^{(i)}$ is orthogonal to $U'$, 
then $\E[Q(\x)(\g^{(i)}\cdot\x)\prod_{i=1}^{k-1}(\u^{(i)}\cdot\x)]$ is small. 
In particular, this means that if $\u^{(1)}\in U$ and $\u^{(2)},\ldots,\u^{(k)}$ are orthogonal to $U'$, 
then $\E[Q(\x)\prod_{i=1}^k(\u^{(i)}\cdot\x)]$ will be small.
By repeating this approach for $k$ iterations, and taking $V$ as the span of the subspaces found in the process, 
we obtain that the maximum $\E[Q(\x)\prod_{i=1}^k(\u^{(i)}\cdot\x)]$ will not be much larger 
if all the unit vectors are taken from $V$. This completes the overview of the analysis of our learning algorithm.

\paragraph{Algorithmic Application (I): Generic Family and Learning via Sandwiching Polynomials}
To make use of our algorithm template, it remains to show that there are 
natural function classes that satisfy our condition of being $(k,\Delta,\beta,C)$-weakly learnable 
for some reasonable sets of parameters. In particular, we argue that essentially
any non-pathological family of $K$-MIMs is $(k,\Delta,\beta,C)$-weakly learnable 
for an appropriate choice of parameters. We note that the assumption of being a $K$-MIM 
here is critical, as nearly matching central moments will imply that nearly all moments 
along the critical directions will nearly match.

In order to construct such a weak learner, suppose that for the optimal hypothesis $h^*$ in the class, 
we can find a polynomial $p$, such that $\E[p(\x)] = 0$ and $p(\x)$ is positive whenever $h^*(\x)=1$. 
In this case, we have $\E[h^*(\x)p(\x)] = \E[(h^*(\x)+1)p(\x)]>0$. 
Since this $p$ then correlates with the indicator set of $h^*(\x) = 1$, 
it is not hard to see that some appropriate threshold of $p$ will necessarily produce a weak learner.

We claim that very mild conditions on $h^*$ will allow us to construct such a $p$. 
For example, supposing that $\x$ is isotropic, as long as there is a not-too-small ball $B$ 
not-too-far from the origin so that $h^*$ is $-1$ on $B$, 
we can construct (by standard techniques) a polynomial that approximates 
the indicator function of $B$. Subtracting this from a small constant, 
we produce a polynomial that is mean $0$ and positive whenever $h^*$ is. 
As we need something similar to hold for \emph{every} logconcave $\x$ (not just isotropic ones), 
we will need the slightly stronger condition (see Definition~\ref{def:nice-MIMs}) 
that every ellipsoid $E$ contains a (not too much) smaller ellipsoid $E'$ 
so that $h^*$ is constant on $E'$. We note that this should hold for essentially any reasonable $K$-MIM.


Unfortunately, the parameters obtained through this natural approach 
will usually lead to poor approximation ratios. 
However, if we make the stronger assumption that $h^*$ has sandwiching polynomials  
(i.e., polynomials $p^u$ and $p^l$ with $p^u(\x) \geq h^*(\x) \geq p^l(\x)$ pointwise and $\E[p^u(\x) - p^l(\x)] \leq c$ for any small $c>0$; see Definition~\ref{def:class-sandwich}),
then a much better error guarantee can be obtained (see Theorem~\ref{th sandwhiching}). 
To see this, $\E[p^u(\x')y]$ will approximate $\E[h^*(\x)y] = 1-2\opt = c$ 
within error at most $\E[p^u(\x')-p^l(\x')] \approx \E[p^u(\x)-p^l(\x)]$, 
which can be made as a small multiple of $c$. 
This means that we can find a weak hypothesis, 
as long as $\E[h^*(\x')y]$ is more than a small constant,
which happens when $\opt$ is only slightly less than $1/2$. 
In particular, this implies that for every $K$-MIM 
that admits $L_1$-sandwiching polynomial approximation, 
we can learn it in time $\poly(d) 2^{\polylog(1/\eps)}$ to 
error $(1+c)\opt+\eps$, for any small constant $c>0$, under any logconcave distribution. 
In particular, by leveraging the sandwiching polynomial construction from \cite{kane2013learning}, 
this immediately gives a QPTAS for many fundamental classes, 
such as functions of halfspaces and low-dimensional low-degree polynomial threshold functions 
(see Section~\ref{ssec:qptas-apps}).

\paragraph{Algorithmic Application (II): Halfspaces and Multiclass Linear Classifiers}
As another important application, our algorithmic framework also gives a way to tackle more challenging multiclass classification problems. A natural way to solve a multiclass classification problem is to break it into $\binom{K}{2}$ binary classification problems to learn hypotheses $h_{ij}(\x)$ that tell whether an example is more likely to have label $i$ or label $j$ and combine them together.
We show in Lemma~\ref{lm reduction multi-binary} that it suffices to design 
a semi-agnostic partial learning algorithm for the class of halfspaces over logconcave distributions.
In particular, due to the structure of the class of halfspaces, 
we can use a simple Chow-parameter estimation method to obtain a weak hypothesis, 
as long as the moments of the marginal distribution match some logconcave marginal. 
In particular, in Theorem~\ref{th halfspace quasi}, we show that if the first four moments 
approximately match, then we can agnostically partially learn a halfspace 
to error $O(\opt)+\eps$ in $\poly(d) 2^{\polylog(1/\eps)}$ time, 
which in turn implies a quasi-polynomial time algorithm for multiclass linear 
classification with error $O(K\cdot\opt)+\eps$. 
The quasi-polynomial dependence on $\eps$ is due to the fact that we need to implement 
a discrepancy finder for the $4$th moment. On the other hand, if we only use the first two moments, 
we can always find the moment discrepancy in polynomial time.
For the halfspace learning problem (Theorem~\ref{th halfspace polytime}), this fact 
gives us a polynomial-time algorithm that achieves error $\tilde O(\opt)+\eps$;  
this implies a polynomial-time algorithm 
for multiclass linear classification with error $\tilde O(K\cdot\opt)+\eps$.

\paragraph{Testable Learning}
Finally, we point out that our algorithmic framework naturally fits 
the testable learning framework~\cite{rubinfeld2023testing} 
(see Definition~\ref{def:testable}). Roughly speaking, given a known logconcave distribution $\x\sim D_B$ 
and sample access from an arbitrary distribution $D$ over $\R^d \times \{0,\pm 1\}$, 
a testable learner is asked to either reject or output a hypothesis. 
We want to guarantee that the probability that the algorithm accepts 
but outputs a bad hypothesis is small; and that when $D_\x = D_B$, the algorithm will accept with high probability.
Towards this end, we notice that our boosting algorithm will succeed with no extra assumption, 
as long as the weak partial learner succeeds. Furthermore, our weak partial learner is guaranteed to work, 
as long as we can partition the space into regions $R$ so that most of the mass of examples with non-zero label, 
$\x'$, belong to regions on which the moments of $(\x'\mid R)$ and $(\x \mid R)$ approximately match.
This implies that when we create the partition for the weak partial learner, 
we can replace the global marginal distribution $D_\x$ by $D_B$. 
If within the sample and time budget, the partition is not created successfully, 
we simply reject; otherwise, the partition is created successfully, 
and we obtain a weak partial hypothesis. In particular, if $D_\x = D_B$, we are likely to accept 
(see Theorem~\ref{th testable}).

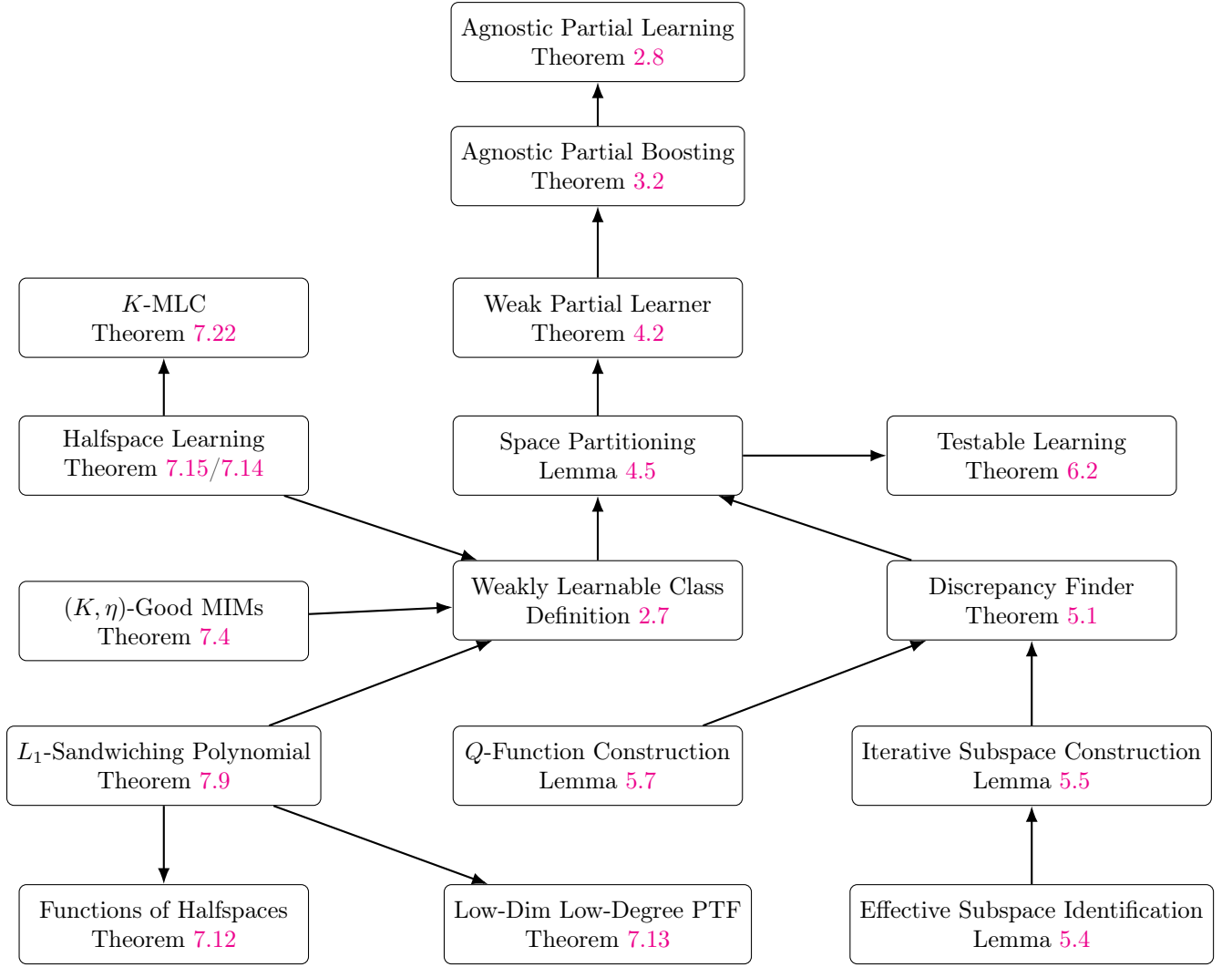
\begin{figure}[t]
\centering
\begin{tikzpicture}[
    >=Latex,
    font=\small,
    box/.style={
        draw,
        rounded corners=3pt,
        align=center,
        minimum width=4.2cm,
        minimum height=1.15cm,
        inner sep=4pt
    },
    qbox/.style={
        draw,
        rounded corners=3pt,
        align=center,
        minimum width=3.6cm,
        minimum height=1.15cm,
        inner sep=4pt
    },
    arrow/.style={->, thick}
]

\node[box] (agnostic)  at (0,8.8) {Agnostic Partial Learning\\Theorem~\ref{thm: main-weak-to-strong}};
\node[box] (boosting)  at (0,7.0) {Agnostic Partial Boosting\\Theorem~\ref{th condition boosting}};
\node[box] (weak)      at (0,4.8) {Weak Partial Learner\\Theorem~\ref{th decision tree k moment}};
\node[box] (partition) at (0,2.8) {Space Partitioning\\Lemma~\ref{lm partition}};

\node[box] (mlc)      at (-6.3,4.8) {$K$-MLC\\Theorem~\ref{th multiclass}};
\node[box] (testable) at ( 6.3,2.8) {Testable Learning\\Theorem~\ref{th testable}};

\node[box] (weakclass) at (0,0.7) {Weakly Learnable Class\\Definition~\ref{def moment-based learner}};
\node[box] (disc)      at ( 6.3,0.7) {Discrepancy Finder\\Theorem~\ref{th direction finding}};

\node[box] (halfspace) at (-6.3,2.8)  {Halfspace Learning\\Theorem~\ref{th halfspace polytime}/\ref{th halfspace quasi}};
\node[box] (margin)    at (-6.3,0.4) {$(K,\eta)$-Good MIMs\\Theorem~\ref{th generic family}};
\node[box] (sandwich)  at (-6.3,-1.7) {$L_1$-Sandwiching Polynomial\\Theorem~\ref{th sandwhiching}};
\node[box] (function)  at (-6.3,-4.0) {Functions of Halfspaces\\Theorem~\ref{th functions of halfspaces}};

\node[box] (iter)      at (6.3,-1.7) {Iterative Subspace Construction\\Lemma~\ref{lm subspace v3}};
\node[box] (critical)  at (6.3,-4.0) {Effective Subspace Identification\\Lemma~\ref{lm subspace finding v1}};
\node[box] (ptf)  at (0,-4.0) {Low-Dim Low-Degree PTF\\Theorem~\ref{thm:low-dim-PTFs}};

\node[box] (qfunc) at (0,-1.7)
{$Q$-Function Construction\\Lemma~\ref{lm Q function learning}};

\draw[arrow]  (boosting) -- (agnostic);
\draw[arrow]  (weak) -- (boosting);
\draw[arrow]  (partition) -- (weak);

\draw[arrow] (partition) --  (testable);
\draw[arrow]  (weakclass) --  (partition);
\draw[arrow]   (disc) -- (partition);

\draw[arrow]  (margin)--(weakclass);
\draw[arrow]  (halfspace)--(weakclass);
\draw[arrow]  (halfspace)--(mlc);
\draw[arrow]  (sandwich) -- (weakclass);
\draw[arrow]  (sandwich) -- (function);
\draw[arrow]  (sandwich) -- (ptf);
\draw[arrow] (iter) -- (disc);
\draw[arrow]  (critical) -- (iter);
\draw[arrow]  (qfunc)--(disc);

\end{tikzpicture}
\caption{Overview of the proof structure and dependencies. Directed edges indicate logical dependencies: an arrow 
$A\to B$ means that the result at 
$B$ relies on or invokes the result at 
$A$.}
\label{fig:overview-flowchart}
\end{figure}


\subsection{Organization} \label{ssec:org}

In Section~\ref{sec:setup},  
we begin with some preliminaries followed by
a formal definition of an agnostic partial classifier 
and its resulting error metrics 
(Definitions~\ref{def partial learning} and~\ref{def:global-error}). 
We then state 
our general algorithmic result (Theorem~\ref{thm: main-weak-to-strong}) 
efficiently translating our notion of weak learning (Definition~\ref{def moment-based learner}) 
to agnostic partial learning. 
The subsequent sections develop the necessary machinery to prove 
Theorem~\ref{thm: main-weak-to-strong}.
In Section~\ref{sec boosting}, we define the notion of weak {\em partial} learning (Definition~\ref{def:partial-weak}), 
cover the generic results on boosting from a weak partial learner, 
and show how it can be used to efficiently 
produce agnostic partial learners (Theorem~\ref{th condition boosting}).  
Section~\ref{sec decision tree} develops our iterative spatial decomposition 
and shows (Theorem~\ref{th decision tree k moment}) how it can be used 
to produce a weak partial learner 
for function classes satisfying our weak learning condition (Definition~\ref{def moment-based learner}). 
In Section~\ref{sec:moment-discr}, we provide our algorithm 
for computing directions of moment discrepancy for logconcave distributions. 
We combine the above results to prove Theorem~\ref{thm: main-weak-to-strong}, 
and state our testable learning extensions in Section~\ref{sec proof main}. 
Finally, in Section~\ref{sec application} we describe our concrete applications, 
showing that our assumptions are satisfied by a wide variety of multi-index models, 
and in particular functions of halfspaces and multiclass linear classifiers.

\section{Setup and Statement of the Main Result} \label{sec:setup}

\paragraph{Notation}
For $n \in \Z_+$, we let $[n] = \{1,2,\dots,n\}$. We use small boldface characters for vectors and capital bold characters for matrices and tensors.
For $\x \in  \R^d$ and $i \in [d]$, $\x_i$ denotes the $i$-th coordinate of $\x$ and $\norm{\x}_2 = \sqrt{\sum_{i=1}^d (\x_i^2)}$ is the $\ell_2$ norm of $\x$. 
Let $V$ be a subspace, we use $\proj_V(\x)$ to denote the projection of $\x$ onto subspace $V$, we will also use $\x_V$ for short to denote $\proj_V(\x)$.
We use $\s^{d-1}$ to denote the unit sphere over $\R^d$ and use $\B^{d-1}(\u,r)$ to denote the ball in $\R^d$ centered at $\u$, with radius $r$.
Let $\mathbf{A} \in \R^{d \times d}$, we define $\mathcal{E}(\u,\mathbf{A}):=\{\x \mid  \norm{\mathbf{A} (\x-\u)}_2 \le 1 \}$ to be the ellipsoid centered at $\u$, with shape $\mathbf{A}$.
For an $\ell$-order tensor $\mathbf{T}$, we use $\norm{\mathbf{T}}_F = \sqrt{\sum_\zeta \mathbf{T}_\zeta^2}$ to denote its Frobenius norm.
For two tensors $\mathbf{T}, \mathbf{T}'$, we will use $\mathbf{T} \cdot \mathbf{T}'$ to denote $\sum _\zeta \mathbf{T}_\zeta\mathbf{T}'_\zeta$.
For $\alpha = (\alpha_1,\dots,\alpha_d)$, a multi-index, we use $\x^\alpha = \prod_{i=1}^d \x_i^{\alpha_i}$, the monomial with multi-index $\alpha$.
For an event $A$ over $\R^d$, we will use $\Ind\{A\}$ to denote the indicator of $A$: 
it equals $1$ on $A$ and $0$ otherwise. 
For $D$ a distribution over $\R^d$ and $f: \R^d \to \R$, we use 
$\norm{f}_{D,p}:= (\E_{\x \sim D}[\abs{f(\x)}^p])^{1/p}$ to denote the $L_p$ norm of $f$ under $D$.
For a distribution $D$ over $\R^d \times \mathcal{Y}$, 
we use $D(\x,y)$ to denote the density of $D$ at 
$(\x,y)$ and use $D(y \mid \x)$ for the conditional distribution of $y$ given $\x$ under $D$. 
Finally, we use $D_\x$ to denote the marginal of $D$ on $\x$.


\paragraph{Basics on Logconcave Distributions}

The family of distributions $D$ that we study in this paper are assumed to have logconcave marginal $D_{\x}$ on the feature space, defined as follows.

\begin{definition}[Logconcave Distribution]\label{def logconcave distribution}
 A distribution $D_\x$ over $\R^d$ with density function $f(\x)$ 
 is logconcave if the function $\log f(\x)$ is concave.
    We say that a logconcave distribution $D_\x$,
    is isotropic logconcave if $\bm{\mu}:=\E_{\x \sim D_\x}[\x] = 0$ 
    and $\E_{\x \sim D_\x}[\x\x^\top] = \mathbf{I}$, 
    where $\mathbf{I}$ is the identity matrix.
    For $c \in (0,1)$, we say $D_\x$ is $c$-isotropic logconcave, if $\norm{\E_{\x \sim D_\x}[\x]}_2 \le c $ and $(1-c)\mathbf{I} \le \E_{\x \sim D_\x}[(\x-\mu)(\x-\mu)^\top] \le (1+c) \mathbf{I}$.
\end{definition}

As shown in~\cite{lovasz2007geometry}, any logconcave distribution 
can be efficiently transformed into a nearly isotropic logconcave distribution.

\begin{fact}[see, e.g.,~Theorem 2.6 in \cite{lovasz2007geometry}]\label{fact transform}
     There is an algorithm with the following guarantee: 
     given sample access to any logconcave distribution $D_\x$ on $\R^d$ and parameters $c,\delta \in (0,1)$, 
     the algorithm draws $M = \tilde O(dc^{-2}\log(1/\delta))$ samples from $D_\x$, 
     runs in sample polynomial time, and outputs an affine transformation 
     $L: \R^d \to \R^d$ such that when $\x \sim D_\x, \ L(\x)$ follows a $c$-isotropic logconcave distribution with probability at least $1-\delta$.
\end{fact}

Thanks to Fact~\ref{fact transform}, most of the technical lemmas in this work 
will focus on $c$-isotropic logconcave distributions, for suitably small $c$.

\paragraph{Basic Definitions for Partial Learning}

One of our key innovations is that of trying to produce weak learners for \emph{partial} classifiers. 
Such a classifier can be thought of as a function $f:X \rightarrow \{\pm 1,0\}$ 
(or more generally given as a distribution over $X\times \{\pm 1,0\}$), 
where the output $f(x)=0$ corresponds to $f$ not classifying $x$ 
and an output of $\pm 1$ corresponding to that Boolean classification. 
Formally, we set out some of our basic terminology as follows:

\begin{definition}[Agnostic Partial Learning]\label{def partial learning}
Let $\mathcal{H}$ be a class of Boolean (i.e., $\{\pm 1\}$-valued)
functions over $\R^d$, and let $D$ be a distribution over $\R^d \times \{\pm 1,0\}$.  
Denote by $D_f$ (called the \emph{effective distribution}) the conditional distribution of $D$ over the event 
$\mathcal{E} \eqdef \{(\x,y)\mid y \neq 0\}$. The error of a hypothesis $h: \R^d \to \{\pm 1\}$ 
is defined as $\err_{D_f}(h) = \Pr_{(\x,y)\sim D_f}[h(x) \neq y]$ and define
$\opt_{D_f}(\H) \eqdef \inf_{h \in \H} \err_{D_f}(h)$.

A (semi-)agnostic partial learner is one that given access to samples from such a distribution $D$ attempts to compute a hypothesis $h$ whose error is competitive with $\opt_{D_f}(\H)$, and in particular is bounded by $F(\opt_{D_f}(\H)) + \eps$ for some non-decreasing function $F:[0,1] \rightarrow [0,1]$ and $\eps > 0$.

We say that this problem is agnostic partial learning over some distribution $D^*$ or class of distributions $\mathcal{C}$ on $\R^d$ if one is additionally guaranteed that the marginal of $D$ over $\R^d$ is equal to $D^*$ or lies in $\mathcal{C}$. In particular, our paper will be focused on the problem of agnostic partial learning of certain well behaved function classes $\H$ over the class of logconcave distributions.
\end{definition}

We note that one important parameter of a partial classifier is the probability with which it assigns a non-zero value, a quantity we call the \emph{weight} of the classifier. This is particularly important as classifiers with low weight will be difficult to learn at the very least because many samples will be required in order to see anything of interest. Formally, we define:

\begin{definition}[Weight of a Partial Classifier]
For a distribution $D$ on $\R^d \times \{\pm 1, 0\}$ we define its \emph{weight}, denoted $W(D)$, to be $\Pr_{(\x,y)\sim D}[y\neq 0].$
\end{definition}

Finally, while the notion of error for a hypothesis given above is perhaps the most natural for measuring the error of a partial classifier, it is sometimes useful to consider the probability over the distribution $D$ as a whole that our hypothesis returns the wrong value. We call this the ``global error'' as defined below:

\begin{definition}[Global Classification Error] \label{def:global-error}
    Let $D$ be a distribution over $\R^d \times \{\pm 1,0\}$.
    We define the global error of a hypothesis $h: \R^d \to \{\pm 1\}$ with respect to $D$ 
    as $\err_D(h) \eqdef \Pr_{(\x,y)\sim D}[h(x) \neq y, y \neq 0]$. The optimal global error of $\H$ with respect to $D$ is 
    $\opt_D(\H) \eqdef \inf_{h \in \H} \err_D(h)$.
\end{definition}

\begin{remark}
{\em We remark that for every Boolean function $h$, 
$\err_D(h) = W(D) \, \err_{D_f}(h)$.
}
\end{remark}


\paragraph{Our Main Result}

Our main algorithmic result is applicable to a broad class of hypothesis classes that admit a weak partial learner so long as certain moment conditions on the effective marginal, $D_f$, are met. In practice, this weak learner will always be obtained by finding a low degree polynomial $p$ which has non-trivial inner product with $y$, but in principle any (efficient) weak learning algorithm should suffice. 

Note that for a partial classifier, one should consider applying the definition below to $D_f$.


\begin{definition}[$(k,\Delta,\beta,C)$-weakly learnable class]\label{def moment-based learner}
    Let $\H$ be a class of Boolean functions over $\R^d$. 
    For $k \in \N, \Delta>0, \beta \in (0,1),C>1$,
    we say that $\H$ is {\em $(k,\Delta,\beta,C)$-weakly learnable}, 
    if there is an algorithm 
    $\A$ with the following performance guarantee. 
    Suppose that for an input distribution $D$ over $\R^d \times \{\pm 1\}$, 
    the following conditions are satisfied: 
    \begin{enumerate}[leftmargin=*]
        \item There is some $1/2$-isotropic logconcave distribution $D'_\x$ over $\R^d$ and a given affine transformation $L: \R^d \to \R^d$ 
        such that $\tilde D_\x$, the distribution of $D_\x$ under the affine transformation $L$
        satisfies for every $\w \in \s^{d-1}$ and all $i \in [k]$, $\abs{\E_{\x \sim \tilde D_\x}[(\w\cdot\x)^i]-\E_{\x \sim D'_\x}[(\w\cdot\x)^i]}<\Delta$, 
        \item $\opt_{D}(\H)\le 1/C$. 
    \end{enumerate}
    Then $\A$, on input i.i.d.\ samples from $D$, outputs a hypothesis $\hat{h}$ satisfying $\err_{D}(\hat{h}) \leq 1/2-\beta$ with probability at least $99\%$. 
    We will call $\beta$ the ``advantage'' of the (weak hypothesis) $\hat{h}$.
    The algorithm $\A$ will be called a {\em $(k,\Delta,\beta,C)$-weak learner for $\H$}.
\end{definition}

A number of prior works require that the function class either exhibit 
a weak learner or some kind of moment-based identifiability. However, previous results 
typically require that this identifiability either holds for all marginal distributions or 
only for the specific marginals being learned. To make our argument work for partial 
learners, we will need to be able to work with the effective marginal $(D_f)_\x$, which is 
merely close to logconcave. Fortunately, if we merely assume that we nearly match central 
moments with a logconcave distribution (the requirement for $1/2$-isotropic can be fixed by 
applying an affine transformation), this is a condition that we can both reduce to (as we 
show in Sections \ref{sec decision tree} and \ref{sec:moment-discr}) and a condition that is 
sufficient in many instances (as we show in Section \ref{sec application}).




We are now ready to state our main algorithmic result. In particular, we show that any weakly learnable class of functions can be agnostically learned with respect to any logconcave distribution.

\begin{theorem}[Weak Learning Implies Agnostic Partial Learning]\label{thm: main-weak-to-strong}
    Suppose that $\H$ is a class of functions that is $(k,\Delta,\beta,C)$-weakly learnable 
    by an algorithm $\A$ with sample complexity $m$ and time complexity $T$. 
    Then there is an agnostic partial learning algorithm for $\H$
    with the following guarantee: 
    Given any $\eps\in (0,1)$ and sample access to any distribution $D$ 
    over $\R^d \times \{0,\pm 1\}$ of weight $W(D) \geq \alpha$     
    with marginal $D_\x$ logconcave, the algorithm 
    draws $(1/\beta^{2}) \, d^{O(k)}\poly_{k,\Delta}(1/(\alpha\eps)) \, m$ samples 
    from $D$, 
    runs in time $(1/\beta^{2}) \, d^{O(k)} \, 2^{\poly_{k,\Delta}(\log(1/(\alpha\eps)))}\, T$, 
    and outputs a hypothesis $\hat{h}$ such that with probability at least $99\%$, $\err_{D_f}(\hat{h})\le (C/2)\opt_{D_f}(\H) + \eps$. 
\end{theorem}


In other words, so long as $\H$ is weakly-learnable by an efficient algorithm $\A$ we have an agnostic partial learner for $\H$ that works over any logconcave marginal. This learner will have fully polynomial sample complexity (though the degree of that polynomial may depend on the parameters of $\H$'s weak learnability), and time complexity that is only $\mathrm{quasipoly}(1/\eps)$ worse. This can be used to learn partial classifiers even of small weight, though as we will see this algorithm substantially improves on the state of the art even for many classes of total functions. Additionally, this learner can be made testable in the sense of Theorem \ref{th testable}.

\section{Boosting a Weak Partial Hypothesis}\label{sec boosting}

The first key ingredient of our approach is a boosting algorithm establishing that a weak {\em partial} learner (Definition~\ref{def:partial-weak}) 
can always be efficiently converted into a (strong) agnostic partial learner (Theorem~\ref{th condition boosting}). 

This can be viewed as a natural extension of the boosting techniques widely used in distribution-free learning 
to the distribution-specific setting. Traditionally boosting algorithms have typically required that the weak learner is distribution-free. 
However, if a smooth boosting algorithm, see, e.g.~\cite{Feldman:10ab,KalaiK09}
is used, we show that it suffices 
to have a \emph{distribution-specific} weak \emph{partial} learner. 
Technically, this result follows from a careful use of the boosting algorithm of \cite{Feldman:10ab}, but conceptually we think this is a significant 
advance, as it allows us to consider only effective marginal distributions which are not too far from the original marginal---a condition that we will later show how to exploit.


Formally, we define a weak partial learner as follows.

\begin{definition}[$(\beta,C)$-weak partial learner] \label{def:partial-weak}
    For $0< \alpha,\beta <1$ and $C>1$, we say that 
    an algorithm $\A$ is a $(\beta,C)$-weak partial learner for a class $\H$ 
    with respect to the distribution $D^{\ast}$
    if it satisfies the following guarantee: 
    For every $\alpha \in (0,1)$ and for every input distribution $D$ over $\R^d \times \{\pm 1, 0\}$ with $\bx$-marginal $D_{\bx} =D^{\ast}$, $W(D)>\alpha$ and $\opt_{D_f}(\H) \le 1/C$, then $\A$ draws samples from $D$ and outputs a hypothesis $\hat{h}: \R^d \to \{\pm 1\}$, such that with probability at least $99\%$, $\err_{D_f}(\hat{h}) \le 1/2-\beta$.

    Furthermore, we say that $\A$ has time and sample complexity $T(\alpha)$ and $m(\alpha)$ if the appropriate complexities are bounded by the given functions.
\end{definition}

The main result of this section is the following theorem.

\begin{theorem}[Weak Partial Learning Implies Agnostic Partial Learning]\label{th condition boosting}
 Let $D^*$ be a distribution over $\R^d$, $\H$ be a hypothesis class over $\R^d$,  
 and $\A$ be a $(\beta,C)$-weak partial learner for $\H$ with respect to $D^*$ 
 with time and sample complexity $T(\alpha), m(\alpha)$.
There is an algorithm with the following performance guarantee: 
for any $\eps \in (0,1)$, any input distribution $D$ over $\R^d \times \{0,\pm 1\}$ 
with marginal $D_\x = D^*$ and $W(D) \ge \alpha$, the algorithm draws $N = \tilde O(\beta^{-2}/\eps)m(\alpha\eps)$ samples from $D$, 
runs in time polynomial in $N$ and $T(\alpha\eps)$,
and outputs a hypothesis $\hat{h}$ such that with probability at least $99\%$,
$\err_{D_f}(\hat{h}) \leq (C/2)\opt + \eps$.
\end{theorem}

Essentially, Theorem \ref{th condition boosting} is saying that in order to run a boosting algorithm instead of needing to run our weak learner on \emph{arbitrary} input distributions, it suffices to run it only on distributions that can be obtained from our original distribution $D_f$ (itself a conditioning of the base distribution $D$) by a rejection sampling procedure that accepts with probability at least $\eps$. This in turn means that a distribution-specific weak partial learner is sufficient instead of requiring a distribution-independent weak learner. This can be shown by analyzing the boosting algorithm of \cite{Feldman:10ab}.

\begin{fact}[Corollary of Theorem 3.5 in \cite{Feldman:10ab}]\label{th distribution specific boosting}
    Let $\H$ be a hypothesis class over $\R^d$ and let $\A$ be an oracle with the following behavior: for every distribution $D$ over $\R^d \times \{\pm 1\}$ 
    such that $\min_{h \in \H}\err_D(h) \leq 1/2 - a$, 
   given $m$ examples from $D$, it outputs a hypothesis $\hat{h}$ such that $\err_D(\hat h) \leq 1/2 -\beta$,
   with probability at least $99\%$.
    Then there is an algorithm $\mathcal{B}$ with the following performance guarantee: 
    for every distribution $D$ over $\R^d \times \{\pm 1\}$ and for every $\eps>0$, 
    $\mathcal{B}$ makes $\tilde O(\beta^{-2}/\eps)$ many calls to $\A$, it runs in time $\poly(T,\beta^{-1},\eps^{-1})$, and outputs a hypothesis $\hat h$ 
    such that $\err_D(\hat{h}) \leq \opt_D(H)/(1-2a)+ \eps$, with probability $99\%$. 
    Furthermore, each call to $\A$ is defined over some distribution $D'$ 
    so that $D'$ is given by $D$ conditioned on some explicit event $E$ of probability at least $\eps$.
\end{fact}

\begin{remark}
    The statement of Theorem 3.5 in \cite{Feldman:10ab} requires 
    access to a distribution-free weak learner. However, 
    the proof requires a weaker guarantee. Specifically, 
    given a fixed distribution $D$, to boost a weak learner over $D$, 
    it suffices that the weak learner $\A$ succeeds over distributions 
    whose marginals can be rejection sampled from $D_\x$ with an acceptance rate of 
    at least $\eps$.
    Since this is not explicitly stated in \cite{Feldman:10ab}, 
    we provide the details of the proof in Appendix~\ref{app boosting} for completeness. 
\end{remark}

Given Fact \ref{th distribution specific boosting}, the proof of Theorem \ref{th condition boosting} is straightforward.
\begin{proof}
Our algorithm will simulate running $\mathcal{B}$ from Fact \ref{th distribution specific boosting} to learn the distribution $D_f$. In particular, when the oracle $\A$ is called on some distribution $D'$, we note that $D'$ is given by $D_f$ conditioned on an event $E$ of probability at least $\eps$. This in turn is $D$ conditioned on an explicit event of probability at least $\alpha\eps$. Letting $D^0$ be obtained from the distribution $D$ by setting $y=0$ whenever $E$ does not hold, we note that $W(D^0)\geq \alpha\eps$ and $D^0_f = D'$. Therefore, the $(\beta,C)$-weak partial learner $A$ will satisfy the hypotheses of the oracle from Fact \ref{th distribution specific boosting} (with $1/2-a=1/C$). 
Thus, $\mathcal{B}$ can be simulated giving a final error of 
$$
\err_{D_f}(h) \leq \opt_{D_f}(\H)/(1-2a) + \eps = (1/2)\opt_{D_f}(\H)/(1/2-a) + \eps = (C/2)\opt_{D_f}(\H) + \eps.
$$
This completes our proof.
\end{proof}

\section{Efficient Weak Partial Learning of Weakly Learnable Classes}\label{sec decision tree}

By Theorem~\ref{th condition boosting}, to agnostically partially learn a hypothesis class $\H$ 
under logconcave distributions, it suffices to design an efficient weak partial learner.
Recall that our focus is on concept classes $\H$ that are efficiently 
$(k,\Delta,\beta,C)$-weakly learnable by an algorithm $\A$. 
In particular, this means that so long as the effective distribution $D_f$ has marginals whose low degree central moments approximately agree with those of some logconcave distribution (and in particular, if they agree with $D_\x$'s central moments), then $\A$ will learn a hypothesis with $\beta$-advantage over $D_f$.

It is not hard to see that simply running $\A$ over the distribution $D_f$ is not a good strategy, as the moment information of $(D_f)_\x$ could significantly deviate from that of the requirement of $\A$, even if $90\%$ of the examples have non-zero labels. 
We will need to find a way to reduce to the case where these central moments do approximately match. Our basic strategy will be to find a direction in which our moments fail to match if one exists, and use this to partition space into two halfspaces on which the masses of $D_\x$ and $(D_f)_\x$ substantially disagree. Doing this iteratively, we produce finer and finer partitions of $\R^d$ on which $D_f$ becomes more and more concentrated. We show that eventually this procedure must terminate in a partition where almost all of the mass of $D$ lies in regions where our moment matching assumption does apply. We can then use our $(k,\Delta,\beta,C)$-weak learner on each region to get the desired result.

Unfortunately, to even get started on this project, we need to be able to either certify that $D_\x$ and $(D_f)_\x$ have nearly matching central moments or to find a direction that exhibits this moment mismatch. This is itself a complicated algorithmic problem that we will put off solving until Section \ref{sec:moment-discr}. For the moment, we will be satisfied by formally defining the tool that we need:



\begin{definition}[Moment Discrepancy Finder] \label{def direction finding oracle}
    For $k \in \Z_+, \Delta>0, c, \alpha \in (0,1)$,
    we say that an algorithm $\A$ is a $(k,\Delta,c,\alpha)$-discrepancy finder 
    if it satisfies the following performance guarantee: 
    given sample access to a $1/2$-isotropic logconcave distribution $D_\x$ over $\R^d$ 
    and a distribution $D'_\x$, a conditional distribution of $D_\x$ over 
    some event with probability mass at least $\alpha$, if there exists $\w \in \s^{d-1}$ 
    and $j \in [k]$ such that 
    $\abs{\E_{\x\sim D_\x}[(\w\cdot\x)^j]-\E_{\x\sim D'_\x}[(\w\cdot \x)^j]} \geq \Delta$, 
    then with probability at least $99\%$, the algorithm outputs $\v \in \s^{d-1}$ and $i \in [k]$ 
    such that $\abs{\E_{\x\sim D_\x}[(\v\cdot \x)^i]-\E_{\x\sim D'_\x}[(\v\cdot\x)^i]} \geq c\Delta$.
\end{definition}

Our main result in this section is the following.


\begin{theorem}[Weak Learning$+$Discrepancy Finder$\rightarrow$Weak Partial Learning]\label{th decision tree k moment}
Let  $\eps,\eta \in (0,1)$, and let $\A_w$ be a $(k,\Delta,\beta,C)$-weak learner for $\H$ with sample complexity $M_w$ 
and time complexity $T_w$. Let $\A_m$ be a $(k,\Delta,c, \alpha)$-discrepancy finder, 
for an appropriate value of $\alpha = \eta\eps^{O_{k, \Delta}(1)}$
with sample complexity $M_m$ and time complexity $T_m$. 
Then there is an algorithm $\A$ that satisfies the following: 
  For every input distribution $D$ over $\R^d \times \{\pm 1, 0\}$, with $D_\x$ logconcave,  
 $W(D)>\eps$ and $\opt_{D_f}(\H)<(1-\eta)/C$, 
  then $\mathcal{A}$ draws at most $N(M_w+M_m)$ samples from $D$, 
  runs in time at most $N(T_w+T_m)$,
   where 
$$
N =(c\Delta)^{-4k}(k\log k+\log(C/(c\Delta)))^{4k^2}(\eta \eps)^{-O((c\Delta)^{-2}(k\log k+\log(C/(c\Delta)))^{2k})}.
$$ 
Algorithm $\A$ then outputs a hypothesis $h$ such that $\err_{D_f}(h) \leq 1/2-\Omega(\eta \beta)$ with probability at least $99\%$.
\end{theorem}


As discussed, to prove Theorem~\ref{th decision tree k moment}, we will need to iteratively decompose $\R^{d}$ into smaller and smaller convex sets. In order to ensure that this procedure produces the desired results, we will need to assume that each split satisfies some conditions that force the distributions of $D$ and $D_f$ across the two halves to differ sufficiently. In particular, we will formalize the condition we need in the following definition:


\begin{definition}[Good Split]\label{def good split}
Let $D$ and $D'$ be two distributions over $\R$ such that $D' =D_A$ for some event $A$, with $\Pr_{D}[A]\ge \alpha$. 
We say that $t \in \R$ is a $(k,\tau,\gamma,\Delta)$-good split for $D,D'$, if 
    $\abs{t}<T:=O(k\log k + \log(1/\Delta\alpha))$,
    and
    for $R$ either $\{z \in \R \mid z>t\}$ or $\{z \in \R \mid z<t\}$,
    one of the following properties holds:
    \begin{enumerate}[leftmargin=*]
        \item $\abs{t} \le \tau$ and $\abs{\Pr_{D}[R]-\Pr_{D'}[R]}>\gamma$.
        \item $R = \{z \in \R \mid z>t\}, t>\tau$ or $R = \{z \in \R \mid z<t\}, t<-\tau$,   $\Pr_{D}[R] \le \left(\Pr_{D'}[R]\right)^4$ and 
        $\Pr_{D'}[R]>\abs{t}^{-2k}$.
    \end{enumerate}

\end{definition}


Not only does a good split allow one to subdivide the line into two sets on which the probabilities of $D$ and $D'$ differ substantially, we show that if, in the case we are interested in, we can find a direction in which the moments of $D$ and $D_f$ disagree substantially, that we can produce a good split in this direction.


\begin{figure}[htbp]
    \centering
    \includegraphics[width=0.7\linewidth]{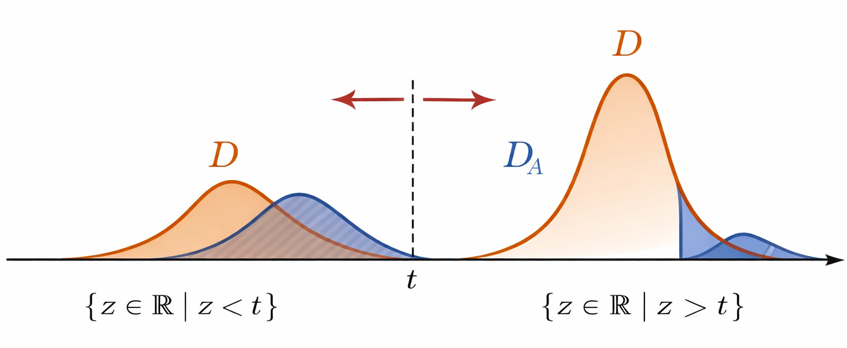}
    \caption{Illustration of a good split.
    } 
    \label{fig:split}
\end{figure}

\begin{lemma}\label{lm good split}
Let $D$ be a $1/2$-isotropic logconcave distribution over $\R$, 
$A$ an event with probability $\alpha >0$. Let $D'$ be the conditional distribution of $D$ over $A$.
 Suppose that $\abs{\E_{z \sim D} [z^k] - \E_{z' \sim D'} [(z')^k]}>\Delta$, 
 for some $\Delta>0$ and $k \in \N$. 
 Then there exists $t \in \R$ that is a $(k,\tau,\gamma,\Delta)$-good split for $D,D'$, with 
 $$\tau \eqdef \max\{(\Delta/20)^{-1/k}, 10(k\log k + \log(C/\Delta))\}, 
 \gamma \eqdef \tau^{-k}\Delta/100 \;,$$
 where $C$ is the constant from Fact~\ref{fact tail logconcave}.
 Furthermore, there is an algorithm that on input $k, \Delta, \alpha$, oracle access to $D$ and $D'$, and $\delta>0$,  
 draws $\poly(1/\gamma,1/\alpha, \log(1/\delta))$ samples from $D$, runs in sample-polynomial time, and outputs a $(k,\tau,\gamma,\Delta)$-good split 
 with probability at least $1-\delta$.
\end{lemma}

Lemma~\ref{lm good split} says that if the lower order moments of $D$ and $D'$ do not match, then a good split always exists, and we can efficiently find such a good split through a brute-force search. 
In particular, we want to remark that the existence of such a good split does not hold when $D$ is an arbitrary distribution. 
In our proof, we show the existence of such a good split by making use of the fact that a logconcave distribution has a sub-exponential tail probability.

The basic strategy of the proof is as follows. 
We know that $D$ and $D'$ have some low order moment that has substantial disagreement. 
We note that as these moments can be expressed as suitable integrals 
of the cumulative density functions that there must be some $t$ 
for which those cumulative densities differ. 
We have only to show that such $t$ satisfy the particular conditions of a good split. 
We first note that since $D$ has subexponential tails that the contribution 
to its moments from $t$ with $|t| \gg k\log(k) + \log(1/\Delta\alpha)$ 
is small enough that we can ignore these points. We then need to carefully 
analyze the integral over the remaining points to show that unless 
some $t$ satisfies either condition 1 or 2 that the moment discrepancy 
will be less than what is promised.

\begin{proof}[Proof of Lemma~\ref{lm good split}]

We first show the existence of such a good split. To do this, suppose that for every $\abs{t}<\tau$, we have $\abs{\Pr_{z \sim D}[z \in R]-\Pr_{z' \sim D'}[z' \in R]}<2\gamma$. We will show that we can find some $\abs{t}\le O(k\log k + \log(1/(\Delta\alpha)))$ that satisfies the second condition of a good split. We deal with the case where $k$ is odd, and the case where $k$ is even can be analyzed symmetrically.
Notice that 
\begin{align}\label{eq layer cake}
    \E_{z \sim D} [z^k] & = \int_{0}^\infty \Pr_{z \sim D}[z^k>s] ds - \int_{0}^\infty \Pr_{z \sim D}[z^k<-s] ds \notag \\ 
    & = \int_{0}^\infty ks^{k-1} \Pr_{z \sim D}[z>s] ds - \int_{0}^\infty ks^{k-1} \Pr_{z \sim D}[z<-s] ds.
\end{align}
We first show that those $\abs{z'}> T:= 40 (k\log k + \log(C/(\Delta\alpha)))$ only contribute little to the $k$th moment of $z'$. By Fact~\ref{fact tail logconcave}, we know that there is some universal constant $C>0$ such that for every $t>T$, 
\begin{align*}
    \Pr_{z' \sim D'}[\abs{z'}>t] = \Pr_{z \sim D}[\abs{z}>t, A]/\alpha \le \Pr_{z \sim D}[\abs{z}>t]/\alpha \le Ce^{-t/2}/\alpha.
\end{align*}
This implies that 
\begin{align*}
    \int_{T}^\infty ks^{k-1} \Pr_{z' \sim D'}[\abs{z'}>s] ds \le O\left(kT^{k-1}\exp(-T/2)/\alpha\right) \le \Delta/100.
\end{align*}
On the other hand, since for every $t, \abs{t}<\tau$, $\abs{\Pr_{z \sim D}[ \abs{z}>t ]-\Pr_{z' \sim D'}[\abs{z'} > t]}<\gamma$, we have 
\begin{align*}
    & \abs{\int_0^\tau ks^{k-1} \left(\Pr_{z \sim D}[ z>s ]-\Pr_{z' \sim D'}[z' > s]\right)ds} \le \tau^k\gamma\le \Delta/100, \\
    & \abs{\int_0^\tau ks^{k-1} \left(\Pr_{z \sim D}[ z<-s ]-\Pr_{z' \sim D'}[z' < -s]\right)ds} \le \tau^k\gamma\le \Delta/100.
\end{align*}
By \eqref{eq layer cake}, we know that 
\begin{equation}\label{mid range integral equation}
\int_{|s| \in [\tau,T]} ks^{k-1}\left| \Pr_{z \sim D}[z>s] - \Pr_{z' \sim D'}[z'>s] \right|ds \geq \Delta/2 .
\end{equation}
Assuming that for every $t \in [\tau,T]$, $\Pr_{z' \sim D'}[\abs{z'}>t]<t^{-2k}$, would imply that
\begin{align*}
    \int_\tau^T ks^{k-1}\Pr_{z' \sim D'}[\abs{z'}>s] ds \le \int_\tau^T ks^{-(k+1)} ds/10 = \tau^{-k} \le \Delta/10.
\end{align*}
On the other hand,
\begin{align*}
    \int_\tau^T ks^{k-1}\Pr_{z \sim D}[\abs{z}>s] ds \le C\int_\tau^T k s^{k-1} e^{-s/2} ds \le 2C\tau^{k-1}\exp(-\tau/2)\le \Delta/10.
\end{align*}
Combining these with the equivalent equations for negative $s$ imply that the left hand side of \eqref{mid range integral equation} is at most $2\Delta/5$, which is a contradiction.
This implies that there is some $t$ with $\abs{t} \in [\tau,T]$ such that $\Pr_{z' \sim D'}[\abs{z'}>t]>t^{-2k}$. 
In particular, since $t>\tau> 10(k\log k + \log(C))$, we know that $\Pr_{z \sim D}[\abs{z}>t] \le Ce^{-t/2} \le (\Pr_{z' \sim D'}[\abs{z'}>t])^4$. This shows that there is always a $(k,\tau,\gamma,\Delta)$-good split. In fact, reviewing these calculations, we see that there exists a $t$ that forms a good split so that either in case 1, we have the difference in probabilities at least $2\gamma$ or in case 2, we have that $\pr_{D'}[R] > (1.1)|t|^{-2k}$ and $\pr_D[R] < (\pr_{D'}[R]/1.1)^4$. This means that in order to find a good split, it suffices to estimate the tail probabilities $\pr_{z\sim D}[z>t]$ and $\pr_{z'\sim D'}[z'>t]$ to error $T^{-8k}/10$ for all $t$. However, a VC dimension argument says that this is accomplished by taking the empirical distribution over $O(T^{16k})$ samples.
\end{proof}


We next give an overview of how to make use of the discrepancy finding algorithm together 
with Lemma~\ref{lm good split}. Recall that the $(k,\Delta,\beta,C)$-weakly learning algorithms 
$\A_w$ work if the input distribution (after some suitable affine transformation) 
approximately matches each of its first $k$ moments with some approximately isotropic 
logconcave distribution. Thus, unless there is such a large moment discrepancy, 
we can simply run this algorithm. If there is a discrepancy, this can be detected 
by our discrepancy finder, which will also demonstrate a particular direction $v$ 
in which the moments fail to match. It will then be the case that 
$\langle v,D\rangle$ and $ \langle v,D_f\rangle$ will be real random variables 
with a moment mismatch to which we can apply Lemma \ref{lm good split} to find a good split. 
This will allow us to subdivide $\R^d$ into two halfspaces 
so that the restrictions of $D$ to each halfspace will have somewhat different weights.


From here we produce an iterative decomposition. We begin with a single subdivision consisting of all of $\R^d$. Then iteratively for each subdivision $R$ that $(D_f)_\x$ assigns enough mass we run our discrepancy finding algorithm on $D$ and $D_f$ conditioned on $R$. If it finds no discrepancy, we run our weak learner on that region. Otherwise, our discrepancy finder has found a moment discrepancy along some direction and we can run Lemma \ref{lm good split} to find a good split to subdivide $R$. This idea is summarized in Algorithm \ref{alg decision tree k moment}.

\begin{algorithm}[htbp]
		\caption{\textsc{k-MomentDecisionTree} (Constructs A decision tree as a weak hypothesis)}\label{alg decision tree k moment}
		\begin{algorithmic} [1]
\State \textbf{Input:}  Sample access from $D$, minimum weight parameter $\eta>0$,
accuracy parameter $\eps \in (0,1)$, moment discrepancy $\Delta>0$ and advantage $\beta>0$, access to $(k,\Delta,\beta,C)$-weakly learning algorithm $\A_w$ and $(k,\Delta,c,\alpha)$-discrepancy finder $\A_m$ with $\alpha= \eta\eps^{C_{k,\Delta}}$, where $C_{k,\Delta}$ is a large enough constant given $k$ and $\Delta$.
\State \textbf{Output:} $\hat{h}:$ weak hypothesis for $D_f$ and $\mathcal{U}$, a partition of $\R^d$. 

\medskip

\State Let $\tau, \gamma$ be parameters in the statement of Lemma \ref{lm good split} as a function of $c\Delta$ and $k$.
\State Let $\xi= (\eta \eps/2)^{C'/\gamma^2}$ for $C'$ a sufficiently large constant.
\State Define $T$ as in the statement of Definition~\ref{def good split}  as a function of $c\Delta,k$ and $\alpha=\xi$.
\State Let $\mathcal{U}_0=\{\emptyset\},\mathcal{U}_1=\{\R^d\}$, $t=1$.
\While{$\mathcal{U}_{t} \neq \mathcal{U}_{t-1}$}
\State $\mathcal{U}_{t+1} = \emptyset$.
\For{$R \in \mathcal{U}_t$}
\State Let $D_R$ be the conditional distribution of $D$ over $R$, normalizing to $1/2$-isotropic position.
\If{$(\Pr_{(\x,y)\sim D_f}[\x \in R]<\xi)$} add $R$ to $\mathcal{U}_{t+1}$
\Else 
 \State 
 run $\A_m$ over $D_R$ and $(D_R)_f$ to get a direction $\u \in \s^{d-1}$ such that $\abs{\E_{(\x,y)\sim D_R}(\u\cdot \x)^i-\E_{(\x,y)\sim (D_R)_f}(\u\cdot \x)^i}>c\Delta$, for an $i \in [k]$.
\If{$\u$ is not returned by $\A_m$} add $R$ to $\mathcal{U}_{t+1}$
\Else \State  find a $(i,\tau,\gamma,c\Delta)$-good split for $(\u\cdot\x),(\u\cdot\x')$, where $\x\sim (D_R)_\x$ and $\x'\sim ((D_R)_f)_\x$ and add $R_1,R_2$, the two regions split from $R$ to $\mathcal{U}_{t+1}$.
\EndIf
\EndIf
\EndFor
\State $t= t+1$.
\EndWhile
\For{$R \in \mathcal{U}_t$}
\State If $\Pr_{(\x,y) \sim D_f}[\x \in R] \geq \xi$, run $\A_w$ on $D_f$ conditioned on $\x \in R$. Estimate the error of this hypothesis to within $\beta/3$ and if the result is less than $1/2-2\beta/3$, let $h_R$ be the returned hypothesis. If $\Pr_{(\x,y) \sim D_f}[\x \in R] < \xi$ or the error of the returned hypothesis is too large, let $h_R$ be a random function.
\EndFor
\State Return the weak hypothesis $\hat{h}(\x) = h_R(\x)$ if $\x \in R$, for $R \in \mathcal{U}_{t},$  and the partition $\mathcal{U} = \mathcal{U}_t$.
\end{algorithmic}
\end{algorithm}

We give an overview of the analysis of Algorithm \ref{alg decision tree k moment}, before giving the full proof. The analysis will follow two parts.

Firstly, we show that the number of rounds of the algorithm is bounded. For this, we note that for each round that we did not terminate, we must have subdivided at least one region in the previous round. Thus, the total number of rounds is at most $|\mathcal{U}_t|$. On the other hand, whenever we use a good split for a region $R$, the resulting regions $R_1$ and $R_2$ will have $\pr_{D_\x}[R_i]$ not too much smaller than $\pr_{D_\x} [R]$. Since $\pr_{(\x,y)\sim D_f}[\x \in R]$ could not have been less than $\xi$ (or we would not have considered $R$ for subdivision), this means that the mass $D_\x$ assigns to each region must be bounded below. This in turn gives an upper bound on the number of regions and rounds.


In the second part, we need to show that the returned hypothesis has reasonably large advantage. We note 
that for any region $R$ in our final partition with $\Pr_{(\x,y)\sim D_f}[\x \in R] \geq \xi$, we must 
have small central moment discrepancies and so in any such region where $\opt < 1/C$, our weak learner 
will return a hypothesis with advantage at least $\beta$. For the remaining regions, our random hypothesis 
will have error at most  $1/2$. We therefore need to establish that the total mass of the low weight 
regions is small. In particular, we need the total mass of these regions to be noticeably smaller than 
$\eta$, so that as long as $\opt < (1-\eta)/C$, there will still be some heavy regions with $\opt < 1/C$. 
To do this, we instead bound  
$\sum_{R \in \mathcal{U}_t} \Pr_{(\x,y)\sim D_f}[\x \in R]^{1-2s}$ 
for some small, positive constant $s$. This in turn can be bounded by considering 
the potential function
$$
\Phi_t := \sum_{R \in \mathcal{U}_t} (\Pr_{(\x,y)\sim D_f}[\x \in R]\Pr_{(\x,y)\sim D}[\x \in R])^{1/2-s} \;.
$$
In particular, we show that this potential is non-increasing so long as $s$ is at most a small multiple of $\gamma^2$. As $\Phi_1=1$, this provides an appropriate bound on the small mass regions.

We summarize the above properties of the partition $\mathcal{U}$ output by Algorithm \ref{alg decision tree k moment} as Lemma \ref{lm partition} and use it to prove Theorem \ref{th decision tree k moment}.

\begin{lemma}
\label{lm partition}
For $k \in \Z_+$, $\eps,\eta, c \in (0,1)$ and $\Delta>0$, 
Let $T,\gamma,\tau,$ and $\xi$ be as in Algorithm \ref{alg decision tree k moment}.

If Algorithm \ref{alg decision tree k moment} is run on a distribution $D$ over $\R^d \times \{\pm 1,0\}$, with $D_\x$ logconcave with $W(D) > \eps$, then upon termination, 
    with probability at least $99\%$, the final partition $\mathcal{U}$ of $\R^d$ satisfies: 
    \begin{enumerate}
        \item \label{it size} $\card{\mathcal{U}} \le O(T^{2k}/(c\Delta \xi \eps))$.
        \item \label{it mismatch} For every $R \in \mathcal{U}$ 
        with $\Pr_{(\x,y)\sim D_f}[\x \in R]>\xi$, for every $\u \in \s^{d-1}$ and for every $i \in [k]$ 
        $\abs{\E_{(\x,y)\sim D_R}[(\u\cdot \x)^i]-\E_{(\x,y)\sim (D_R)_f}[(\u\cdot \x)^i]} \le \Delta.$
        \item\label{it prob} $\sum_{R \in \mathcal{B}}\Pr_{(\x,y)\sim D_f}[\x \in R] \le \eta/2$, where $\mathcal{B}=\{R \in \mathcal{U} \mid \Pr_{(\x,y)\sim D_f}[\x \in R] \le \xi\}$.
    \end{enumerate}
\end{lemma}

\begin{proof}[Proof of Lemma \ref{lm partition}]
    We start by showing the first item in the statement that $\card{\mathcal{U}} \le O(T^{2k}/(c\Delta\xi \eps))$. 
    For each time step $t$, let $R \in \mathcal{U}_t$.
    If $R \not \in \mathcal{U}_{t-1}$, it must be the case, where $R$ is split from some $R' \in \mathcal{U}_{t-1}$ using a $(i,\tau,\gamma,c\Delta)$ good split, with $i\in [k]$. 
This implies that $\Pr_{(\x,y) \sim D_f}[\x \in R']>\xi$ and $\Pr_{(\x,y) \sim D}[\x \in R']>\eps\xi$.

By the definition of an $(i,\tau,\gamma,c\Delta)$ good split, either $\Pr_{D_\x}[R\mid R']>\gamma$ or $\Pr_{D_\x}[R\mid R']>T^{-2k}$. Since $c\Delta T^{-2k}<\gamma$, in the first case, we have 
$\Pr_{D_\x}[R]>c\Delta T^{-2k}\Pr_{D_\x}[R']>c\Delta T^{-2k}\xi\eps$. In the second case, $\Pr_{D_\x}[R] \ge  \eps \Pr_{(D_f)_\x}[R]>\eps T^{-2k}\Pr_{(D_f)_\x}[R']>T^{-2k}\xi\eps.$
Since $\sum_{R \in \mathcal{U}_t}\Pr_{(\x,y)\sim D}[\x \in R] = 1$ for each time step $t$,
we know that $\card{\mathcal{U}} \le O(T^{2k}/(c\Delta\xi \eps))$.

Next, we show the second item in the statement. Let $R \in \mathcal{U}$ be a region with $\Pr_{(D_f)_\x}[R]>\xi$.
Notice that by the definition of $\A_m$, if there is a direction $\u \in \s^{d-1}$ and some $i \in [k]$  such that $\abs{\E_{(\x,y)\sim D_R}(\u\cdot \x)^i-\E_{(\x,y)\sim (D_R)_f}(\u\cdot \x)^i} > \Delta,$ 
then some direction $\w$ will be output and $R$ will be split into two regions, which contradicts the fact that Algorithm \ref{alg decision tree k moment} terminates.

Finally, we show the third item in the statement. To do this, we 
define a potential function
\begin{align*}
    \Phi_t: = \sum_{R \in \mathcal{U}_t}\left(\Pr_{(\x,y)\sim D}[\x \in R]\Pr_{(\x,y)\sim D_f}[\x \in R]\right)^{1/2-s}.
\end{align*}
We show that for the choice of $s = \gamma^2/10$, $\Phi_{t+1}\le \Phi_t$ for all $t$.

It suffices to show that for any $R \in \mathcal{U}_t$ that is split into regions $R_1,R_2$,   
\begin{align*}
 \left(\Pr_{(\x,y)\sim D}[\x \in R]\Pr_{(\x,y)\sim D_f}[\x \in R]\right)^{1/2-s} \ge \sum_{i=1,2} \left(\Pr_{(\x,y)\sim D}[\x \in R_i]\Pr_{(\x,y)\sim D_f}[\x \in R_i]\right)^{1/2-s}
\end{align*}
To simplify the notation, we denote by $a:=\Pr_{(\x,y)\sim D}[\x \in R_1 \mid \x \in R]$ and $b:=\Pr_{(\x,y)\sim D_f}[\x \in R_1 \mid \x \in R]$.
Notice that 
\begin{align*}
    &\sum_{i=1,2} \left(\Pr_{(\x,y)\sim D}[\x \in R_i]\Pr_{(\x,y)\sim D_f}[\x \in R_i]\right)^{\frac{1}{2}-s} \\
    =&  \left(\Pr_{(\x,y)\sim D}[\x \in R]\Pr_{(\x,y)\sim D_f}[\x \in R]\right)^{\frac{1}{2}-s} \left((ab)^{1/2-s}+\left((1-a)(1-b)\right)^{\frac{1}{2}-s}\right)
\end{align*}
We next show that under an $(i,\tau,\gamma,c\Delta)$-good split, $\left((ab)^{1/2-s}+\left((1-a)(1-b)\right)^{\frac{1}{2}-s}\right) \le 1$. We prove this by considering cases for a good split. In the first case in the definition of $(i,\tau,\gamma,\Delta)$-good split, we have $\abs{a-b}>\gamma$.
In this case, 
we have 
\begin{align*}
    \left((ab)^{1/2-s}+\left((1-a)(1-b)\right)^{\frac{1}{2}-s}\right) 
   & = \left(\left( (\frac{a+b}{2})^2-(\frac{a-b}{2})^2  \right)^{\frac{1}{2}-s}+\left( (1-\frac{a+b}{2})^2-(\frac{a-b}{2})^2  \right)^{\frac{1}{2}-s}\right) \\ 
   &\le 2 \left( \frac{1}{4}-(\frac{a-b}{2})^2  \right)^{\frac{1}{2}-s} \le 2 \left( \frac{1-\gamma^2}{4}  \right)^{\frac{1}{2}-s} \le 1.
\end{align*}
Here in the first inequality, we use the fact that, 
for fixed value of $(a-b)/2$, the summation is maximized when $(a+b)/2 = 1/2$; 
while in the second inequality, we use the condition that $\abs{a-b}>\gamma$. 
In the last inequality, we use the choice of $s \le \gamma^2/10$. 

In the second case, by the definition of $(i,\tau,\gamma,c\Delta)$-good split, 
we have $a \le b^4$. Notice that $b<1/10$, otherwise we are in the first case. 
We have 
\begin{align*}
    \left((ab)^{1/2-s}+\left((1-a)(1-b)\right)^{\frac{1}{2}-s}\right) 
   & = \left(\left( ab  \right)^{\frac{1}{2}-s}+\left( 1-a-b+ab  \right)^{\frac{1}{2}-s}\right) \\ 
   &\le b^{5(1/2-s)} + 1-(1/2-s)(a+b-ab)\\
   & \le 1- (1/2-s)b + 2b^{5(1/2-s)}<1 \;.
\end{align*}
Here, the first inequality follows that for $z \in (0,1)$, and $p<1$, $(1-z)^p \le 1-pz$, and the second inequality follows the bound $a \le b^4, b \le 1/10$.

So, we have $\Phi_{t+1} \le \Phi_t$ for all $t$.
Using this, we know that 
\begin{align*}
    \Phi_1 = 1 & \ge \sum_{R \in \mathcal{B}_t}\left(\Pr_{(\x,y)\sim D}[\x \in R]\Pr_{(\x,y)\sim D_f}[\x \in R]\right)^{1/2-s} \\ & \ge \left( \Pr_{(\x,y)\sim D}[y \neq 0] \right)^{1/2-s} \sum_{R \in \mathcal{B}_t}\left(\Pr_{(\x,y)\sim D_f}[\x \in R]\right)^{1-2s} \\
    & \ge \left( \Pr_{(\x,y)\sim D}[y \neq 0] \right)^{1/2-s} \sum_{R \in \mathcal{B}_t}\left(\Pr_{(\x,y)\sim D_f}[\x \in R]\right)\xi^{-2s},
\end{align*}
which gives us 
\begin{align*}
\sum_{R \in \mathcal{B}_t}\left(\Pr_{(\x,y)\sim D_f}[\x \in R]\right) \le \left( \Pr_{(\x,y)\sim D}[y \neq 0] \right)^{s-1/2}\xi^{2s} \le \eta/2,
\end{align*}
where the last inequality holds when $\xi=(\eta \eps/2)^{1/2s} = (\eta \eps/2)^{O(1/\gamma^2)}$.
\end{proof}

With Lemma \ref{lm partition} at hand, 
we can immediately prove Theorem \ref{th decision tree k moment}.


\begin{proof}[Proof of Theorem \ref{th decision tree k moment}]
We first show the correctness of Algorithm \ref{alg decision tree k moment} by showing that it outputs a hypothesis with $\Omega(\eta\beta)$ advantage with respect to $D_f$.
Let $\mathcal{B}=\{R \in \mathcal{U} \mid \Pr_{(\x,y)\sim D_f}[\x \in R] \le \xi\}$.
We denote by $\mathcal{G} = \mathcal{U} \setminus \mathcal{B}$ and denote by $\mathcal{M} \subseteq \mathcal{G}$ the set of $R$, where a random binary function $h_R$ is assigned. 
We show that $\sum_{R \in \mathcal{M}}\Pr_{(\x,y)\sim D_f}[\x \in R]$ is not too large. 
By Item \ref{it mismatch} in Lemma \ref{lm partition},
it holds for every $R \not \in \mathcal{B}$ and $\u \in \s^{d-1}$ and for every $i \in [k]$,
$\abs{\E_{(\x,y)\sim D_R}(\u\cdot \x)^i-\E_{(\x,y)\sim (D_R)_f}(\u\cdot \x)^i} \le \Delta.$
If $\A$ fails to output a weak hypothesis over $R$, it must be the case, where 
we have $\Pr_{(\x,y)\sim D_f}[h^*(\x) \neq y \mid \x \in R]>1/C$, where $h^*$ is the hypothesis in $\H$ that achieves $\opt_{\D_f}(\H)$. 
Now suppose that $\sum_{R \in \mathcal{M}}\Pr_{(\x,y)\sim D_f}[\x \in R]>(1-\eta)$, then we have
\begin{align*}
 \err_{D_f}(h^*)\ge \sum_{R \in \mathcal{M}}\Pr_{(\x,y)\sim D_f}[\x \in R]\Pr_{(\x,y)\sim D_f}[h^*(\x) \neq y \mid \x \in R]>(1-\eta)/C \;,
\end{align*}
which gives a contradiction to the fact that $\opt_{D_f}\le (1-\eta)/C$.
By Item~\ref{it prob} of Lemma~\ref{lm partition}, $\sum_{R \in \mathcal{B}}\Pr_{(\x,y)\sim D_f}[\x \in R] \le \eta/2$.
Thus, we have $\sum_{R \in \mathcal{G}\setminus \mathcal{M}}\Pr_{(\x,y)\sim D_f}[\x \in R]>\eta/2$. 
Since for each $R \in \mathcal{G}\setminus \mathcal{M}$, 
we have $\Pr_{(\x,y)\sim D_f}[h_R(\x) \neq y \mid R]<1/2-\beta/3$, (and the error on all other regions is $1/2$) we know that $\err_{D_f}(\hat{h}) \le 1/2-\Omega(\eta \beta)$.

Next, we bound the sample complexity and the time complexity of Algorithm \ref{alg decision tree k moment}. 
Observe that every time Algorithm \ref{alg decision tree k moment} makes a split, 
it adds one more region to the partition $\mathcal{U}_t$. 
This implies Algorithm \ref{alg decision tree k moment} will terminate 
after making $\card{\mathcal{U}}-1$ splits.
By Item \ref{it size} of Lemma \ref{lm partition}, 
it will run $\A_w,\A_m$ and the algorithm for finding a good split
at most $ O(T^{2k}/(c\Delta\xi\eps))$ times for regions $R$ 
created by Algorithm \ref{alg decision tree k moment}. 
Since $\Pr_{(\x,y)\sim D}[\x \in R]>\xi\eps$ and 
$\Pr_{(\x,y)\sim D_f}[\x \in R]>\xi$, an example drawn from $D_R, \ (D_R)_f$ 
can be done by drawing $(\eps\xi)^{-1} = (\eta\eps)^{-O(\gamma^{-2})}$ examples from $D$. 
So, each call to $\A_w$ draws $(\eps\xi)^{-1}M_w$ samples from $D$
and runs in $(\eps\xi)^{-1}T_w$ time. Similarly,
each call to $A_m$ draws $(\eps\xi)^{-1}M_m$ samples from $D$ and runs for $(\eps\xi)^{-1}T_m$ time.
Each call to the algorithm for finding a good split takes $\poly(\xi^{-1})$ sample and time. 
Thus, the sample complexity of Algorithm \ref{alg decision tree k moment} 
is $ O(T^{2k}/(c\Delta\xi^2\eps))(M_w+M_m)$ and the time complexity 
of Algorithm \ref{alg decision tree k moment} is $O(T^{2k}/(c\Delta\xi^2\eps))(T_w+T_m)$.

Recall that $\gamma = c\Delta (k\log k+\log(C/c\Delta))^{-k}$, 
$\xi = (\eta\eps)^{O(1/\gamma^2)}$ and $T = O(k\log k + \log(1/\Delta\xi))$. 
This gives $T^{2k}/(c\Delta\xi^2\eps^2)=N 
= (c\Delta)^{-4k}(k\log k+\log(C/(c\Delta)))^{4k^2}(\eta \eps)^{-O((c\Delta)^{-2}(k\log k+\log(C/(c\Delta)))^{2k})}$. 
\end{proof}

\section{Finding Moment Discrepancies under Logconcave Distributions} 
\label{sec:moment-discr}


By Theorem~\ref{th decision tree k moment}, an important step towards 
a weak learner for agnostic partial learning (Definition~\ref{def moment-based learner}), 
is to design 
a moment discrepancy finder (Definition~\ref{def direction finding oracle}). 
For the simplest case of $k=2$, one can leverage mean estimation and singular-value decomposition to achieve this goal. Unfortunately, for $k>3$, this task appears challenging in general.
For example, the straightforward way to achieve this goal requires approximating 
eigenvectors of higher-order moment tensors, which is NP-hard \cite{hillar2013most} in general. 

We instead show that if the distribution $D_\x$ 
is logconcave and $(D_f)_\x$ is the conditional distribution 
over an event $A$ with probability $\alpha$
there is a $(k,\Delta,2^{-O(k^2)}, \alpha)$-discrepancy finder 
that runs in 
$\poly_{k,\Delta}(d)2^{\polylog_{k,\Delta}(\alpha^{-1})}$ time. 

Specifically, the main algorithmic result of this section is the following.

\newpage


\begin{theorem}[Moment Discrepancy Finder for Logconcave Distributions]\label{th direction finding}
There is an algorithm with the following performance guarantee.  
Given sample access to a $1/2$-isotropic logconcave distribution $D_\x$ over $\R^d$, 
oracle access to an event $A$ with $\Pr_{D_\x}[A] = \alpha$, $k \in \Z_+$ and $\Delta \in (0,1)$, 
the algorithm draws $M= d^{O(k)}\alpha^{-2}(\log(1/\alpha)/\Delta)^{2^{O(k)}}$ 
examples from $D_\x$,
runs in time 
$$\poly(M) + 2^{(\log(1/\alpha)/\Delta)^{2^{O(k)}}}$$ 
and satisfies the following:
\begin{enumerate}
    \item With probability at least $99\%$, it either outputs $\perp$, or outputs a vector $\v \in \s^{d-1}$ such that 
$\abs{\E_{\x \sim D_\x}[(\v\cdot\x)^i]-\E_{\x \sim (D_f)_\x}[(\v\cdot\x)^i]} \geq 2^{-O(k^2)}\Delta$, for some $i \in [k]$.
\item If there exists $\w \in \s^{d-1}$ and $i \in [k]$ such that 
$\abs{\E_{\x \sim D_\x}[(\w\cdot\x)^i]-\E_{\x \sim (D_f)_\x}[(\w\cdot\x)^i]} \geq \Delta$,
where $(D_f)_\x$ is the conditional distribution of $D_\x$ over $A$, 
the algorithm outputs $\v \in \s^{d-1}$ with probability at least $99\%$.
\end{enumerate}

\end{theorem}


Before going to the technical analysis needed to prove Theorem~\ref{th direction finding}, 
we provide an overview of the algorithm. Our first step is to show that the discrepancy between the higher moment tensors of $D$ and $D_f$ is in some sense small. For this, we note by the
Carbery-Wright concentration inequality (Fact~\ref{fact polynomial growth}), one can show that for any degree-$k$ polynomial $p(\x)$, 
$\abs{\E_{\x \sim D_\x}[p(\x)]-\E_{\x \sim (D_f)_\x}[p(\x)]}$
is not too large. This implies (see Lemma~\ref{lm Q function learning}), that there must be a degree-$k$ 
polynomial $Q(\x)$, with $L_2$-norm at most $O(\log^k(1/\alpha))$ under $D_\x$, 
such that for every degree $k$ polynomial $p(\x)$, we have $\E_{\x \sim D_\x}[p(\x)]-\E_{\x \sim (D_f)_\x}[p(\x)] = \E_{\x \sim D_\x}[Q(\x)p(\x)]$. 
Given that the difference between the higher moments is in some sense ``small'', our plan will be to find a low dimensional subspace $V$ that captures any possible directions of large discrepancy. We can then conclude with an exhaustive search over $V$.

For the purpose of motivation, we begin by assuming that $D_\x$ is a standard Gaussian 
distribution. In this case, we can expand $Q(\x)$ in the Hermite basis, yielding a sequence 
of tensors $\mathbf{T}_0,\dots,\mathbf{T}_k$, with small Frobenius norm so that for every 
$\v \in \s^{d-1}$, we can write $\E_{\x \sim D_\x}[Q(\x)(\v \cdot \x)^k]$ as a linear 
combination of the $\mathbf{T}_{\new{i}}\cdot \v^{\otimes i}$ for $i\in [k]$.
If we flatten each tensor $\mathbf{T}_i$ and view it as a bilinear form 
$\R^d \times \R^{d^{i-1}} \to \R$, we can then find singular vectors 
of each bilinear form that correspond to large singular values.
These singular vectors span a subspace $V$ with dimension $\poly_k(\log(1/\alpha)/\Delta)$ 
such that $\E_{\x \sim D_\x}[Q(\x)(\v \cdot \x)^k]$ can be approximated as a linear 
combination of $\E_{\x \sim D_\x}[Q(\x)(\v' \cdot \x)^i]$, where $\v'$ is the projection of 
the unit vector $\v$ onto $V$, plus an error of at most $\Delta/2$.
Then in time exponential in $\dim(V)$, we can find such a direction in $V$ near maximal 
moment discrepancy, which gives an algorithm with sample complexity $\poly(1/\alpha)$ and 
time complexity quasi-polynomial with respect to $1/\alpha$.

However, the above approach heavily makes use of properties of Gaussian distributions, such as moments in orthogonal directions being independent, which do not hold for a general logconcave distribution.
Our plan is still to make use of the above idea of dimension reduction, but with a relaxed goal. 
In particular, we would like to find a subspace $V$ such that $\sup_{\u^{1},\dots,\u^{(k)} \in \s^{d-1}} \E_{\x \sim D_\x}[Q(\x)\prod_{j=1}^k(\u^{(j)}\cdot\x)]]$ is at most some constant multiple of the maximum when $\u^{1},\dots,\u^{(k)}$ are taken from $V$, plus $\Delta/2$.
If such a subspace $V$ is found, then we can brute force search over $V$, to find a direction with moment discrepancy. We achieve this goal through two steps.

In the first step, Lemma~\ref{lm subspace finding v1}, we observe that given any polynomial $Q(\x)$, the linear function $\E_{\x \sim D_\x}[Q(\x)p(\x)]$ can be written as $\sum_{i=0}^k \mathbf{T}_i \cdot \E_{\x\sim D_\x}[\nabla^i p(\x)]$.
By making use of the recent progress on KLS conjecture \cite{klartag2025thin}, 
we are able to show in Lemma~\ref{lm gradient norm} that the Frobenius norms 
of these tensors $\mathbf{T}_1,\dots,\mathbf{T}_k$ are not large.
By flattening these tensors as a bilinear form and finding singular vectors 
of each bilinear form that correspond to large singular values,
we are able to find a subspace $V_1$, such that if $\u^{1},\dots,\u^{(k)}$ 
are all orthogonal to $V_1$, 
then $\E_{\x \sim D_\x}[Q(\x)\prod_{j=1}^k(\u^{(j)}\cdot\x)]$ must be very small.

The second step is to recursively build the target subspace using the subroutine in the first step. Given any $k$ unit vectors $\u^{1},\dots,\u^{(k)}$,
we can write each $\u^{(i)} = \u^{(i)}_{V_1}+\u^{(i)}_{V_1^\perp}$.
So, we can expand
$\E_{\x \sim D_\x}[Q(\x)\prod_{j=1}^k(\u^{(j)}\cdot\x)]$ into a linear combination of $2^k$ terms, each of which is the expectation of $Q(\x)\prod_{j=1}^k(\v^{(j)}\cdot \x)$, where $\v^{(j)}$ is either $\u^{(j)}_{V_1}$ or $\u^{(j)}_{V_1^\perp}$.
We know if $\v^{1},\dots,\v^{(k)}$ are all orthogonal to $V_1$, then $\E_{\x \sim D_\x}[Q(\x)\prod_{j=1}^k(\v^{(j)}\cdot\x)]$ must be very small.
This implies that during the analysis, we can without loss of generality assume that $\u^{(1)} \in V_1$.
Specifically, let $\g^{(1)},\dots,\g^{(N)}$ be a basis for the subspace $V_1$,
we can write $\u^{(1)} = \sum_{i=1}^N a_i \g^{(i)} $.
This implies that 
\begin{align*}
    \E_{\x \sim D_\x}[Q(\x)\prod_{j=1}^k(\u^{(j)}\cdot\x)] = \sum_{i=1}^N a_i\E_{\x \sim D_\x}[Q(\x)(\g^{(i)}\cdot\x)\prod_{j=2}^k(\u^{(j)}\cdot\x)].
\end{align*}

Notice that if we apply the subroutine in the first step to each of the $Q(\x)(\g^{(i)}\cdot\x)$ and denote by the span of the found subspaces as $V_2$, then if $\u^{(j)} \perp V_1 \oplus V_2$ for $j=2,\dots,k$. then $\E_{\x \sim D_\x}[Q(\x)\prod_{j=1}^k(\u^{(j)}\cdot\x)]$ must be very small.
As a result, if we recursively run this approach for $k$ times, it must be the case that the final subspace $V$ approximately maintains the maximum of $\E_{\x \sim D_\x}[Q(\x)\prod_{j=1}^k(\u^{(j)}\cdot\x)]$.

We now give a roadmap of this section. In Section~\ref{sec KLS}, we will provide technical lemmas that will be used in our analysis based on the recent progress on KLS conjecture. In Section~\ref{sec dimension reduction}, we will give formal analysis of the dimension reduction approach we discussed above. 
Finally, in Section~\ref{sec mismatch finding}, we will make use of the dimension reduction method to give our discrepancy finding algorithm and present the proof of Theorem~\ref{th direction finding}.

\subsection{Useful Technical Facts}\label{sec KLS}

 To begin with, we have the following Poincaré inequality for isotropic logconcave distributions, 
 which follows from recent progress on thin-shell bounds for isotropic logconcave distributions.

\begin{lemma}\label{lm gradient}
    Let $D_\x$ be a $1/2$-isotropic logconcave distribution over $\R^d$ and $p:\R ^d \to \R$ 
    be a degree-$k$ polynomial. Then it holds
    $
        \var_{\x \sim D_\x}[p(\x)] \le O((\log^2 d) \E_{\x \sim D_\x}[\norm{\nabla p(\x)}_2^2]).
    $
\end{lemma}

\begin{proof}[Proof of Lemma~\ref{lm gradient}]
We without loss of generality assume $D_\x$ is isotropic. To see this, let $\bm{\mu} = \E_{\x\sim D_\x}[\x]$ and $\bm \Sigma$ be the covariance matrix of $\x$. Then $\x' = \bm \Sigma^{-1/2}(\x-\bm \mu)$ is isotropic.
Denote by $D_\x'$, the distribution of $\x'$ and write $g(\x') = p(\Sigma^{1/2}\x'+\bm\mu)$.
We have $\var_{\x\sim D_\x}[p(\x)] = \var_{\x'\sim D_{\x'}}[g(\x')]$
and $\E_{\x'\sim D_{\x'}}[\norm{\nabla g(\x')}_2^2] = \E_{\x'\sim D_{\x'}}[\norm{\Sigma^{1/2}\nabla p(\Sigma^{1/2}(\x'+\bm\mu))}_2^2] \le 2\E_{\x \sim D_\x}[\norm{\nabla p(\x)}_2^2]$.

    By the Poincaré inequality \cite{lee2018kannan}, we know that there is a constant $C(D)>0$ such that for every polynomial $p(\x)$, we have 
    \begin{align*}
        \var_{\x \sim D_\x}[p(\x)] \le C(D) \E_{\x \sim D_\x}[\norm{\nabla p(\x)}_2^2].
    \end{align*}
By Theorem 1.5 in \cite{milman2009role} and Theorem 1.1 in \cite{eldan2013thin}, there is a universal constant $C>0$ such that
we have 
\begin{align}\label{eq small shell constant}
    C(D) \le C \log d \sum_{i=1}^d \frac{\sigma^2_i}{i}.
\end{align}
Here, $\sigma_i^2 = \sup_{D_\x} \E_{\x\sim D_\x} [(\norm{\x}_2-\sqrt{i})^2]$, where the supremum is taken over all isotropic logconcave distributions over $\R^i$. By Theorem 1.1 in \cite{klartag2025thin}, there is a universal constant $c>0$ such that $\sigma^2_i<c, \forall i$. Thus, by \eqref{eq small shell constant}, we get
\begin{align*}
        \var_{\x \sim D_\x}[p(\x)] \le O(\log^2 d) \E_{\x \sim D_\x}[\norm{\nabla p(\x)}_2^2]) \;.
\end{align*}
This completes the proof. 
\end{proof}

Using Lemma~\ref{lm gradient}, we establish the following result that bounds the $L_2$ norm of a polynomial $p$ in terms of the norm of its gradient under a logconcave distribution $D_\x$.

\begin{lemma}\label{lm gradient norm}
    Let $D_\x$ be a $1/2$-isotropic logconcave distribution over $\R^d$ 
    and let $p: \R^d \to \R$ be a degree-$k$ polynomial. 
    Denote by $F(p) \eqdef (\E_{\x\sim D_\x}[\nabla^i p(\x)])_{i=0}^k$, where 
    $\nabla^i p(\x)$ is the $i$-th order derivative tensor of $p$. Then 
    $
        \E_{\x \sim D_\x}[p^2(\x)] \le O(\log d)^{2k} \norm{F(p)}_F^2 \;.$
\end{lemma}

\begin{proof}[Proof of Lemma~\ref{lm gradient norm}]
We prove Lemma~\ref{lm gradient norm} using an induction on $k$.
If $k=0$, we have $\E_{\x \sim D_\x}[p^2(\x)] = \norm{F(p)}_F^2$. 
Suppose the statement holds for $k-1$, we show this for $k$. 
By Lemma~\ref{lm gradient}, for some universal constant $C>0$, we have
    \begin{align*}
        \E_{\x \sim D_\x}[p^2(\x)] & = \E_{\x \sim D_\x} [p(\x)]^2 + \var_{\x \sim D_\x}[p(\x)] \\
        & \le  \E_{\x \sim D_\x} [\nabla^0 p(\x)]^2 + C \log^2 d \E_{\x \sim D_\x}[\norm{\nabla p(\x)}_2^2] \\
        & = \E_{\x \sim D_\x} [\nabla^0 p(\x)]^2 + C \log^2 d \E_{\v\sim N(0,\mathbf{I})} \E_{\x \sim D_\x}[(\v\cdot \nabla p(\x))^2] \\
        & \le \E_{\x \sim D_\x} [\nabla^0 p(\x)]^2 + C^{2k-1}\log^{2k} d \E_{\v\sim N(0,\mathbf{I})}[ \norm{F(\v\cdot \nabla p(\x))}_F^2 ]\\
        & \le O (\log d)^{2k} \norm{F(p)}_F^2 \;.
    \end{align*}
Here the second inequality follows from the induction on the degree $k-1$ 
polynomial $\v \cdot \nabla p(\x)$.
\end{proof}

\subsection{Novel Dimension-Reduction Approach for Logconcave Distributions}\label{sec dimension reduction}

In this subsection, we design and analyze a new dimension-reduction approach 
to obtain our discrepancy finding algorithm. 
The key module of our dimension-reduction approach is the following result. 



\begin{algorithm}[htbp]
		\caption{\textsc{EffectiveSubspaceIdentifying}}\label{alg subspace finding v1}
		\begin{algorithmic} [1]
\State \textbf{Input:}  Sample access to $D_\x$, accuracy parameter $\eps \in (0,1)$, confidence parameter $\delta \in (0,1)$, degree $m$-polynomial $Q(\x)$ with a small mean, degree parameter $k$
\State \textbf{Output:} $V$, a low-dimensional subspace

\medskip

\State Let $\eta=\eps2^{-O(k\log k)} \norm{Q(\x)}_{D_\x,2}$ and $\eta'= \norm{Q(\x)}_{D_\x,2} d^{-k}\eps/(2\sqrt{k!})$.
\State For a degree-$k$ polynomial $p$, define $F(p):= \left( \E_{\x\sim D_\x}[\nabla^i p(\x)]\right)_{i=0}^k.$
\State For each monomial $\x^\zeta$ of degree at most $k$, where $\zeta$ is the multi index, use empirical mean estimation to find $\mu_\zeta \in \R$ and a sequence of symmetric tensors $Z_\zeta$ 
such that $\abs{\mu_\zeta - \E_{\x \sim D_\x}[Q(\x)\x^{\zeta}]}<\eta'/4$ and $\norm{Z_\zeta-F(\x^\zeta)}_F \le O(\eta'/(C^k\log^k(d)\norm{Q}_{D_\x,2}) )$.
\State Find tensor $\hat{\mathbf{T}}_1,\dots,\hat{\mathbf{T}}_k$ such that $\norm{\hat{\mathbf{T}}_i}_F \le 2C^k \log^k(d) \norm{Q}_{D_\x,2}$, and $\abs{\mu_\zeta-\sum_{i=1}^k\hat{\mathbf{T}}_i \cdot Z^{(i)}_\zeta} \le \eta'/2$, for every $\zeta$, with $\abs{\zeta} \le k$, using convex programming.
\State For each $\hat{\mathbf{T}}_i$, treat it as a bilinear form over $\R^d \times \text{sym}(\R^{d^{i-1}})$ and compute the singular vectors $\v^{(i,1)},\dots,\v^{(i,l)} \in \R^d$ of $\hat{\mathbf{T}}_i$ with corresponding singular value at least $\eta$.
\State Return $V$, the subspace spanned by these singular vectors.
\end{algorithmic}
\end{algorithm}


\begin{lemma}[Effective Subspace Identification]\label{lm subspace finding v1}
    There is an algorithm that, given sample access to any $1/2$-isotropic logconcave distribution $D_\x$ over $\R^d$, has the following performance guarantee. 
    For any given parameters $m, k \in \N$, $\eps,\delta \in (0,1),$ and for any 
    degree $m$-polynomial $Q: \R^d \to \R$ with 
    $\abs{\E_{\x\sim D_\x}[Q(\x)]}<\eps(dk)^{-O(k)}\norm{Q(\x)}_{D_\x,2}$,
    the algorithm draws $M=\poly(d^{k},2^{O(k\log k+m)},1/\eps,\log(1/\delta))$ samples from $D_\x$, 
    runs in $\poly(M)$ time, and outputs a subspace $V$ of dimension 
    $2^{O(k\log k)}\log^{2k}(d)/\eps^2$ such that with probability at least $1-\delta$ the following holds: for every collection of unit vectors $\{\u^{(i)}\}_{i=1}^k$  and $j \in [k]$,
    that are perpendicular to $V$, we have 
    $\abs{\E_{\x \sim D_\x}\left[Q(\x)\prod_{i=1}^j(\u^{(i)}\cdot\x)\right]}<\eps \norm{Q(\x)}_{D_\x,2}$.
\end{lemma}

\begin{proof}[Proof of Lemma~\ref{lm subspace finding v1}]
We show that Algorithm~\ref{alg subspace finding v1} satisfies the statement.
To this end, we define the following linear map from $P_k$---the space of degree-$k$ polynomials 
over $\R^d$---to $\bigoplus_{i=0}^k \text{sym}(\R^{d^i})$, the space of vectors, whose $i$th entry is an $i$th order symmetric tensor,
\begin{align*}
    F(p) \eqdef \left( \E_{\x\sim D_\x}[\nabla^i p(\x)]\right)_{i=0}^k.
\end{align*}
We claim that $F(p)$ is an invertible linear map. To see this, consider any polynomial $p(\x) = \sum_{i=0}^k \mathbf{A}_i \cdot \x^{\otimes i}$ such that $F(p) = 0$. Notice that $\nabla^k p(\x) = k! \mathbf{A}_k$. Since $F(p) = 0$, we know that $\nabla^k p(\x) = 0$, which implies $\mathbf{A}_k=0$. By induction, we know that $\mathbf{A}_i=0, i \in [k]$ and thus $p(\x) = 0$. So, $F(p)$ is an invertible linear map.

We consider the function over $P_k$, 
$G(p):= \E_{\x\sim D_\x}[Q(\x)p(\x)] = \E_{\x\sim D_\x}[Q(\x)F^{-1}\circ F(p)]$ which 
is a linear functional over $\text{sym}(\R^{d^i})$. 
This implies that there exist tensors $\mathbf{T}_i \in \text{sym}(\R^{d^i})$, 
for $i=0,\dots,k$, such that 
\begin{align*}
    \E_{\x\sim D_\x}[Q(\x)p(\x)] = \sum_{i=0}^k \mathbf{T}_i \cdot \E_{\x\sim D_\x}[\nabla^i p(\x)] \;.
\end{align*}
In particular, $\abs{\mathbf{T}_0} < \eps(dk)^{-O(k)}\norm{Q(\x)}_{D_\x,2}$, because by 
setting $p(\x)=1$,   $\abs{\E_{\x\sim D_\x}[Q(\x)]}<\eps(dk)^{-O(k)} \norm{Q(\x)}_{D_\x,2}$. 
This implies that for every monomial $\x^\zeta$, with degree at most $k$, it holds 
$\abs{\E_{\x\sim D_\x}[\x^\zeta]\mathbf{T}_0}<\eta'$.


We claim that $\norm{\mathbf{T}_i}_F \le C^k \log^k(d) \norm{Q}_{D_\x,2}$ for $i \in [k]$. Suppose this is not correct, since $F(p)$ is an invertible linear map, we choose $p(\x)$ such that $F(p) =(\mathbf{T}_i)_{i=0}^k$. By Lemma~\ref{lm gradient norm}, we have
\begin{align*}
    \E_{\x\sim D_\x}[Q(\x)p(\x)] \le \norm{Q}_{D_\x,2}\norm{p}_{D_\x,2} \le C^k \log^{k} d \norm{F(p)}_F\norm{Q}_{D_\x,2} \;.
\end{align*}
On the other hand,
\begin{align*}
    \E_{\x\sim D_\x}[Q(\x)p(\x)] =  \sum_{i=0}^k\mathbf{T}_i \cdot \E_{\x\sim D_\x}[\nabla^i p(\x)]  = \sum_{i=0}^k \norm{\mathbf{T}_i}_F^2 \;.
\end{align*}
Thus, we have $\norm{F(p)}_F \le C^k \log^k d \norm{Q}_{D_\x,2}$, which contradicts the fact that $\norm{F(p)}_F > C^k \log^k d \norm{Q}_{D_\x,2}$.

Suppose that we have found $\hat{\mathbf{T}}_1,\dots,\hat{\mathbf{T}}_k$ such that for every $\hat{\mathbf{T}}_i$, $\norm{\hat{\mathbf{T}}_i}_F \le 2C^k \log^k(d) \norm{Q}_{D_\x,2}$ and for every monomial $\x^\zeta$ with degree at most degree $k$, 
where $\zeta$ is a multi-index, we have $\abs{\E_{\x \sim D_\x}[Q(\x)\x^\zeta]-\sum_{i=1}^k\hat{\mathbf{T}}_i \cdot \E_{\x\sim D_\x}[\nabla^i \x^\zeta]} \le \eta'$. 
Notice that such a set of tensors always exists due to the existence of $\mathbf{T}_0,\dots,\mathbf{T}_k$. 
To complete the proof, we first show that we can efficiently find the desired subspace $V$ through these tensors and then analyze the complexity of finding these tensors.

For each $\hat{\mathbf{T}}_i$, we treat it as a bilinear form over $\R^d \times \text{sym}(\R^{d^{i-1}})$ and compute the singular vectors $\v^{(i,1)},\dots,\v^{(i,l)} \in \R^d$ of $\hat{\mathbf{T}}_i$ with corresponding singular value at least $\eta$.
Since $\ell \le \norm{\hat{\mathbf{T}}_i}_F^2/\eta^2$, and $\norm{\hat{\mathbf{T}}_i}_F \le 2C^k \log^k(d) \norm{Q(\x)}_{D_\x,2}$,
we know that the dimension of $V$ is at most $2^{O(k\log k)}\log^{2k}(d)/\eps^2$.
We claim that $V$ satisfies the statement of Lemma~\ref{lm subspace finding v1}. 

Let $\u^{(1)},\dots,\u^{(k)} \perp V$ be unit vectors and let $p(\x) = \prod_{i=1}^j(\u^{(i)}\cdot\x)$. We have 
\begin{align*}
    \abs{\E_{\x\sim D_\x}[Q(\x)p(\x)]} & \le  \abs{ \sum_{i=1}^k \hat{\mathbf{T}}_i \cdot \E_{\x\sim D_\x}[\nabla^i p(\x)]} + \abs{\sum_{i=1}^k \hat{\mathbf{T}}_i \cdot \E_{\x\sim D_\x}[\nabla^i p(\x)]-\E_{\x\sim D_\x}[Q(\x)p(\x)]} \\
    & = \abs{ \sum_{i=1}^k \sum_{j = 1}^k \u^{(j)} \cdot \hat{\mathbf{T}}_i \cdot \E_{\x\sim D_\x}[\nabla^{i-1} p(\x)/(\u^{(j)}\cdot \x)] } + \abs{\sum_{i=1}^k \hat{\mathbf{T}}_i \cdot \E_{\x\sim D_\x}[\nabla^i p(\x)]-\E_{\x\sim D_\x}[Q(\x)p(\x)]} \\
    & \le \eta \sum_{i=1}^k \sum_{j = 1}^k \norm{\E_{\x\sim D_\x}[\nabla^{i-1} p(\x)/(\u^{(j)}\cdot \x)]}_F + \abs{\sum_{i=1}^k \hat{\mathbf{T}}_i \cdot \E_{\x\sim D_\x}[\nabla^i p(\x)]-\E_{\x\sim D_\x}[Q(\x)p(\x)]}.
\end{align*}
We next bound the term $\norm{\E_{\x\sim D_\x}[\nabla^{i-1} p(\x)/(\u^{(j)}\cdot \x)]}_F =\norm{\E_{\x\sim D_\x}[\nabla^{i-1} \prod_{\ell \neq j}(\u^{(\ell)}\cdot\x)]}_F $. We have
\begin{align*}
    \norm{\E_{\x\sim D_\x}[\nabla^{i-1} \prod_{\ell \neq j}(\u^{(\ell)}\cdot\x)]}_F & \le k! \norm{\E_{\x\sim D_\x}[\sum_{ \substack {S \subseteq [k]\setminus\{j\} \\ \card{S}=i-1}} \prod_{\ell \not\in S}(\u^{(\ell)}\cdot\x)\text{sym}\bigotimes_{\ell \in S}\u^{(\ell)}]}_F \\
    & \le k!\sum_{ \substack {S \subseteq [k]\setminus\{j\} \\ \card{S}=i-1}} \abs{\E_{\x\sim D_\x} [\prod_{\ell \not\in S}(\u^{(\ell)}\cdot\x)]} \le k!\binom{k-1}{i-1}\left(2(k-i)\right)^{k-i}.
\end{align*}
Here, the first inequality follows from the fact that $\norm{\text{sym}\bigotimes_{\ell \in S}\u^{(\ell)}}_F \le 1$ and the second inequality follows from Fact~\ref{fact moment logconcave}.
This implies that 
\begin{align*}
    \sum_{i=1}^k \sum_{j = 1}^k \norm{\E_{\x\sim D_\x}[\nabla^{i-1} p(\x)/(\u^{(j)}\cdot \x)]}_F 
    &\le k!k \sum_{i=1}^k\binom{k-1}{i-1}\left(2(k-i)\right)^{k-i} \\ 
    &= k!k \sum_{i=0}^{k-1}\binom{k-1}{i}\left(2i\right)^{i} \\ 
    &\le k! k^2(2k^2)^k=  2^{O(k\log(k))} \;.
\end{align*}
On the other hand, we notice that we can expand $p(\x)$ as a linear combination of monomials, $p(\x) = \sum_{\zeta} c_\zeta \x^\zeta$.
By linearity, we have 
\begin{align*}
    \abs{\sum_{i=1}^k \hat{\mathbf{T}}_i \cdot \E_{\x\sim D_\x}[\nabla^i p(\x)]-\E_{\x\sim D_\x}[Q(\x)p(\x)]} & = \abs{\sum_\zeta c_\zeta \left(\sum_{i=1}^k \hat{\mathbf{T}}_i \cdot \E_{\x\sim D_\x}[\nabla^i \x^\zeta]-\E_{\x\sim D_\x}[Q(\x)\x^\zeta]\right)} \\
    & \le \norm{c}_2 \sqrt{\sum_{\zeta} \left(\sum_{i=1}^k \hat{\mathbf{T}}_i \cdot \E_{\x\sim D_\x}[\nabla^i \x^\zeta]-\E_{\x\sim D_\x}[Q(\x)\x^\zeta]\right)^2} \\
    & \le \sqrt{k!}d^{k}\eta'.
\end{align*}
Here, in the second inequality, we use the fact that the coefficient vector 
$c = (c_\zeta)$ comes from  expanding $\prod_{i=1}^j(\u^{(i)}\cdot\x)$ 
and has $\ell_2$-norm at most $\sqrt{k!}$. 
By the choice of $\eta$ and $\eta'$, we have 
$\abs{\E_{\x\sim D_\x}[Q(\x)p(\x)]} \le \eps\norm{Q}_{D_\x,2}$.

To conclude the proof, it remains to upper bound the sample complexity of finding the tensors $\hat{\mathbf{T}}_1,\dots,\hat{\mathbf{T}}_k$. We know that for each monomial $\x^\zeta$, by Fact~\ref{fact moment logconcave} and Fact~\ref{fact polynomial growth}, 
there is some constant $C>0$ such that 
\begin{align*}
    \E_{\x \sim D_\x}[Q^2(\x)(\x^{\zeta})^2] \le \sqrt{\E_{\x \sim D_\x}[Q^4(\x)]\E_{\x \sim D_\x}[(\x^{\zeta})^4]} \le C^{2m}(8k)^{4k} \norm{Q}_{D_\x,2}^2.
\end{align*}
By Chebyshev's inequality, we can estimate some $\mu_\zeta \in \R$ such that $\abs{\mu_\zeta - \E_{\x \sim D_\x}[Q(\x)\x^{\zeta}]}<\eta'/4$ with $d^{O(k)}2^{O(k\log k+m)}/\eps^2$ examples from $D_\x$. Furthermore, denote by $S_\zeta$ the set of indices that are involved in $\x^\zeta$,
for every $i$, 
we have 
\begin{align*}
    \E_{\x\sim D_\x}[\| \nabla^i \x^\zeta\|_F^2] 
    &\le (k!)^2\E_{\x\sim D_\x}\left[\Big\| \sum_{ \substack {S \subseteq S_\zeta \\ \card{S}=i}}\prod_{\ell \not\in S}(\mathbf{e}_\ell\cdot\x)\text{sym}\bigotimes_{\ell  \in S}\mathbf{e}_\ell \Big\|_F^2 \right]  \\
    &\le (k!)^2 \sum_{ \substack {S \subseteq S_\zeta \\ \card{S}=i}} \E_{\x\sim D_\x}[\prod_{\ell \not\in S}(\mathbf{e}_\ell\cdot\x)^2] \\
    &\le 2^{k\log k}\binom{k}{i}(2C(k-i))^{2(k-i)}.
\end{align*}
By Chebyshev's inequality, we can estimate a tensor $Z^{(i)}_\zeta$ such that $\norm{Z^{(i)}_\zeta-\E_{\x\sim D_\x}[\nabla^i \x^\zeta]}_F<\eta'/(C^k \norm{Q}_{D_\x,2} 4\sqrt{k})$ using $O\left((Ck)^k\log^{2k}(d) \E_{\x\sim D_\x}[\norm{\nabla^i \x^\zeta}_F^2]/(\eta')^2\right)$ examples from $D_\x$. This implies that $Z=(Z^{(i)}_\zeta)_{i=0}^k$ is $\eta'/(C^k \norm{Q}_{D_\x,2} 4\log^k(d))$ close to $F(\x^\zeta)$ in Frobenius norm.

By the triangle inequality, any tensors $\hat{\mathbf{T}}_1,\dots,\hat{\mathbf{T}}_k$ such that 
$\norm{\hat{\mathbf{T}}_i}_F \le 2C \log^k(d)\norm{Q}_{D_\x,2}$ for all $i \in [k]$, 
and $\abs{\mu_\zeta-\sum_{i=0}^k\hat{\mathbf{T}}_i \cdot Z^{(i)}_\zeta} \le \eta'/2$ must satisfy
$\abs{\E_{\x \sim D_\x}[Q(\x)\x^\zeta]-\sum_{i=1}^k\hat{\mathbf{T}}_i \cdot \E_{\x\sim D_\x}[\nabla^i \x^\zeta]} \le \eta'$, 
for every monomial $\x^\zeta$. Since $\E_{\x\sim D_\x}[Q(\x)p(\x)] =  \sum_{i=1}^k\mathbf{T}_i \cdot \E_{\x\sim D_\x}[\nabla^i p(\x)]$, 
we know that such a set of tensors exists due to $\mathbf{T}_1,\dots,\mathbf{T}_k$ and thus we can find $\hat{\mathbf{T}}_1,\dots,\hat{\mathbf{T}}_k$ through convex programming.
Furthermore, the total number of examples used by Algorithm~\ref{alg subspace finding v1} is $d^{O(k)}2^{O(k\log k+m)}/\eps^2$.
\end{proof}


Using Lemma~\ref{lm subspace finding v1} as a subroutine, 
we give an algorithm that recursively calls Lemma~\ref{lm subspace finding v1} 
to build a low-dimensional subspace $V$ that approximately preserves 
the maximum of the quantity 
$\E_{\x \sim D_\x}[Q(\x)\prod_{j=1}^k(\u^{(j)}\cdot\x)]$, 
when $\u^{(1)},\dots,\u^{(k)}$ are taken from $V$.
Specifically, we establish the following.

\begin{lemma}\label{lm subspace v3}
There is an algorithm with the following performance guarantee. 
Given sample access to a $1/2$-isotropic logconcave distribution $D_\x$,
parameters $m,k \in \N$, $\eps,\delta \in (0,1),$ and any 
degree $m$-polynomial $Q: \R^d \to \R$, it draws 
$M = \poly(d^k, C_kN_{k,m}(\eps),\log(1/\delta))$ examples from 
$D_\x$ over $\R^d$, 
runs in $\poly(M)$ time, and outputs a subspace $V$ of dimension $N_{k,m}(\eps)$ 
such that with probability at least $1-\delta$,
    \begin{align*}
       S_{k,\R^d}(Q) \le C_{k}S_{k,V}(Q) + \eps \norm{Q}_{D_\x,2},
    \end{align*}
where for $k \in \N$, subspace $V \subseteq \R^d$ and a function $Q$, we denote 
\begin{align*}
 S_{k,V}(Q) \eqdef \sup_{0 \le i \le k}\sup_{\u^{(1)},\dots,\u^{(i)} \in V \cap \s^{d-1}} \abs{\E_{\x \sim D_\x}[Q(\x)\prod_{j=1}^i(\u^{(j)}\cdot\x)]}.   
\end{align*}
For $m,k \in \N$ and for $\eps \in (0,1)$, $C_k = 2^{O(k^2)}$ and $N_{k,m}(\eps) = (2^{O(k+m)}\log^k(d)/\eps)^{2^{O(k)}}$.
\end{lemma}

\begin{algorithm}[htbp]
		\caption{\textsc{LowSubspaceFinding} (Find a low dimensional subspace that preserves discrepancy)}\label{alg subspace finding v3}
		\begin{algorithmic} [1]
\State \textbf{Input:}  Sample access from $D_\x$, accuracy parameter $\eps \in (0,1)$, confidence parameter $\delta \in (0,1)$, degree $m$-polynomial $Q(\x)$, $k$-degree parameter
\State \textbf{Output:} $V$, a low-dimensional subspace

\medskip

\State If $k=0$, Return $\emptyset$
\State  Let $U^{(0)} = \emptyset$ and $\eps_0 = \eps$
\For{$i=1,\dots,k$}
\State Let $G_j=\{\g^{(j1)},\dots,\g^{(jN_j)}\}$ be a standard basis for $U^{(j)}$, for $j \le i-1$
\For{$\g=(\g_1,\dots,\g_{i-1})$, where $\g_j \in G_j$}
\State Define $F_g(\x) = Q(\x)\prod_{j=1}^{i-1}(\g_j\cdot\x)$
\State Estimate $Q_g \in \R$ such that $\abs{Q_g-\E_{\x\sim D_\x}[F_g(\x)]}< (dk)^{-O(k)}\eps_{i-1}  \sqrt{\var[F_g(\x)]}$.
\State Apply Algorithm~\ref{alg subspace finding v1} to $F_g(\x) - Q_g$ with dimension $k-i+1$ and error parameter $\eps_{i-1}/2$ to get a subspace $U_g$
\EndFor
\State Let $U^{(i)}$ be the subspace with dimension $N_i$ spanned by all the $U_g$ found in the $i$th iteration
\State Let $\eps_i = (2^kC^m N_{i})^{-1}\eps_{i-1}/2$.
\EndFor
\State Return $V$ the subspace spanned by $U^{(i)}$, $i \in [k]$.
\end{algorithmic}
\end{algorithm}

\begin{proof}[Proof of Lemma~\ref{lm subspace v3}]
For the purpose of analysis, we will extend a degree $m$ polynomial to a tensor-valued polynomial. Specifically, such a function $P$ maps $\R^d$ to $\R^{d_{1}\times \dots \times d_{\ell}}$ and each entry $P_\zeta$ of $P$ is a degree-$m$ polynomial.
For an $\ell$-order tensor $\mathbf{T} \in \R^{d_1\times \cdots \times d_\ell}$, 
define
$\norm{\mathbf{T}}_* := \sup_{\v^{(1)} \in S^{d_1-1},\dots,\v^{(\ell)} \in S^{d_\ell-1}}\abs{\mathbf{T}\cdot \otimes_{j=1}^\ell \v^{(j)} }$. For $k \in \N$ and subspace $V \subseteq \R^d$, and a $\ell$-order tensor valued function $P$, define
\begin{align*}
 S_{k,V}(P) = \sup_{0 \le i \le k}\sup_{\u^{(1)},\dots,\u^{(k)} \in V \cap \s^{d-1}} \norm{\E_{\x \sim D_\x}[P(\x)\prod_{j=1}^i(\u^{(j)}\cdot\x)]}_*.   
\end{align*}

Notice that in the $i$th iteration of the algorithm, if we write $Q^{(i-1)} = (F_\g)$, for $g = (\g_1,\dots,\g_{i-1}), \g_j \in G_j$ as a $i-1$-order tensor, 
then, we can write $Q^{(i)}$ as a $i$-order tensor, for each $g = (\g_1,\dots,\g_{i-1})$, $Q^{(i)}_g = (F_g(\x)(g^{(i1)}\cdot\x),\dots,F_g(\x)(g^{(iN_i)}\cdot\x))$, where $g^{i1},\dots,g^{iN_i}$ is the standard basis of $U^{(i)}$.
In particular, under this notation, we have $Q(\x) = Q^{(0)}$.
We will show by induction that for $j=0,\dots,k$, the subspace $V^{(j)}$ spanned by $U^{(k)},\dots,U^{(k-j+1)}$ satisfies
\begin{align*}
    S_{j,\R^d}(Q^{(k-j)}) \le C_{j}S_{j,V^{(j)}}(Q^{(k-j)}) + \eps_{k-j} \norm{ \norm{Q^{(k-j)}}_F }_{D_\x,2}.
\end{align*}
Notice that if this is true, then for $j = k$, we have
\begin{align*}
    S_{k,\R^d}(Q) \le C_{k}S_{k,V}(Q) + \eps \norm{Q}_{D_\x,2}.
\end{align*}
The base case $j = 0$ holds trivially because for every subspace $V'$, we have $S_{0,V'}(Q^k) = \norm{\E_{\x\sim D_\x}[Q^{(k)}]}_*$.
Suppose this holds for $j = k-1$, show this for $j=k$.

Let $\u^{(1)},\dots,\u^{(i)}$ be $i$ unit vectors in $\R^d$, where $i \le k$.
We write $\u^{(j)} = \a^{(j)} + \b^{(j)}$, where $\a^{(j)} = \proj_{U^{(1)}} \u^{(j)}$ and $\b^{(j)} =\proj_{(U^{(1)})\perp}\u^{(j)}$.
Then, we have
\begin{align*}
    \abs{\E_{\x \sim D_\x}[Q(\x)\prod_{j=1}^i(\u^{(j)}\cdot\x)]} \le \sum_{S \subseteq [i]} \abs{\E_{\x \sim D_\x}[Q(\x)\prod_{j\not\in S}(\b^{(j)}\cdot\x)\prod_{j\in S}(\a^{(j)}\cdot\x)]}.
\end{align*}
We upper bound each term in the above summation. Recall that $\abs{Q_\emptyset-\E_{\x\sim D_\x}[Q(\x)]}< 2^{-O(k\log k)}\eps_1 \norm{Q}_{D_\x,2}$.
If $S = \emptyset$, 
then we have 
\begin{align*}
    \abs{\E_{\x \sim D_\x}[Q(\x)\prod_{j\not\in S}(\b^{(j)}\cdot\x)\prod_{j\in S}(\a^{(j)}\cdot\x)]} & 
    = \abs{\E_{\x \sim D_\x}[Q(\x)\prod_{j=1}^i(\b^{(j)}\cdot\x)]}\\
    & = \abs{\E_{\x \sim D_\x}[Q_\emptyset\prod_{j=1}^i(\b^{(j)}\cdot\x) + (Q(\x)-Q_\emptyset)\prod_{j=1}^i(\b^{(j)}\cdot\x)]} \\
    & \le \abs{Q_\emptyset}\abs{\E_{\x \sim D_\x}[\prod_{j=1}^i(\b^{(j)}\cdot\x)]} + \abs{\E_{\x \sim D_\x}[(Q(\x)-Q_\emptyset)\prod_{j=1}^i(\b^{(j)}\cdot\x)]} \\
    & \le (2i)^i\abs{Q_\emptyset} + \eps_0 \norm{Q}_{D_\x,2}/2 \;,
\end{align*}
where the last inequality follows by Fact~\ref{fact moment logconcave} 
and the definition of $U^{(1)}$. 

In the rest of the proof, we focus on the case where $S \neq \emptyset$, 
which implies there is at least one $a^{(j)}$ that 
appears in the product $\prod_{j\not\in S}(\b^{(j)}\cdot\x)\prod_{j\in S}(\a^{(j)}\cdot\x)$. 
Without loss of generality, we assume it is $a^{(1)}$ 
and denote by $M_S(\x) = (\a^{(1)}\cdot\x)^{-1} \prod_{j\not\in S}(\b^{(j)}\cdot\x)\prod_{j\in S}(\a^{(j)}\cdot\x)$ the remaining part of the product. 
We write $\a^{(1)} = \sum_{\g \in G_1}^{N_1} \mathbf{c}_\g \g$ under the basis $G_1$.
Then, we have
\begin{align*}
 \abs{\E_{\x \sim D_\x}[Q(\x)\prod_{j\not\in S}(\b^{(j)}\cdot\x)\prod_{j\in S}(\a^{(j)}\cdot\x)]}  & = \abs{\sum_{\g \in G_1} \mathbf{c}_\g \E_{\x \sim D_\x}[Q(\x)(\g\cdot\x)M_S(\x)]} \\
 & = \abs{\sum_{\g \in G_1} \E_{\x\sim D_\x} [\mathbf{c}\cdot Q^{(1)}(\x)M_S(\x)]} \;,
\end{align*}
where $\mathbf{c} = (\mathbf{c}_\g)_{\g \in G_1}$.
We notice that for every polynomial $p(\x)$, we have 
\begin{align}\label{eq star norm}
   \sup_{\v \in U^{(1)}\cap \s^{d-1}} \abs{Q(\x)(\v\cdot\x) p(\x)} = \sup_{\v \in U^{(1)}\cap \s^{d-1}} \abs{\v \cdot Q^{(1)}(\x) p(\x)} = \norm{Q^{(1)}(\x) p(\x)}_*  \;.
\end{align}
By induction, we have 
\begin{align*}
    \abs{\E_{\x \sim D_\x}[Q(\x)(\a^{(1)}\cdot\x)M_S(\x)]} & \le \norm{\E_{\x \sim D_\x}[Q^{(1)}(\x)M_S(\x)]}_* \le S_{k-1,R^d}(Q^{(1)}) \\
    &\le C_{k-1}S_{k-1,V^{(k-1)}}(Q^{(1)}) + \eps_1 \sqrt{\E_{\x \sim D_\x} [\norm{Q^{(1)}(\x)}_F^2]} \\
    & \le C_{k-1}S_{k,V}(Q) + \eps_1 \sqrt{\E_{\x \sim D_\x} [\norm{Q^{(1)}(\x)}_F^2]} \;.
\end{align*}
Here, the last inequality follows from \eqref{eq star norm}.
By Fact~\ref{fact polynomial growth}, for some universal constant $C>0$, 
\begin{align*}
   \E_{\x \sim D_\x}\left[\norm{Q^{(1)}(\x)}_F^2 \right] 
   &= \sum_{\g \in G_1}\E_{\x \sim D_\x}[(Q(\x)(\g^{(i)}\cdot\x))^2] = \E_{\x \sim D_\x}[ Q(\x)^2\sum_{\g \in G_1} (\g\cdot\x)^2] \\
   &\le \sqrt{\E_{\x \sim D_\x}Q(\x)^4}\sqrt{\E_{\x \sim D_\x}[(\sum_{i=1}^N (\g^{(i)}\cdot\x)^2])^2} \le N_1^2 C^{2m} \norm{Q}_{D_\x,2}^2 \;.
\end{align*}
Thus, we get 
\begin{align*}
   \abs{\E_{\x \sim D_\x}[Q(\x)(\a^{(1)}\cdot\x)M_S(\x)]}  \le C_{k-1}S_{k,V}(Q) + \eps_1 N_1 C^{m} \norm{Q}_{D_\x,2} \;,
\end{align*}
which implies 
\begin{align*}
    \abs{\E_{\x \sim D_\x}[Q(\x)\prod_{j=1}^i(\u^{(j)}\cdot\x)]} & \le (2^kC_{k-1}+2^{O(k
    \log k)})S_{k,V}(Q) + (2^kC^mN_1\eps_1+\eps_0/2)\norm{Q(\x)}_{D_\x,2} \\
    & \le C_k S_{k,V}(Q) + \eps \norm{Q(\x)}_{D_\x,2} \;,
\end{align*}
where we use our choice that
$C_k = 2^{k}C_{k-1}+2^{O(k\log k)} = 2^{O(k^2)}$ and $\eps_1 \le (N_1C^m)^{-1}\eps_0/2$.

To complete the proof, we bound the dimension of the subspace $V$ and the 
complexity of the algorithm.
Notice that in the $i$th iteration of the algorithm, 
we run Algorithm~\ref{alg subspace finding v1} $T_i = \prod_{j=1}^{i-1} \dim(U^{(j)})$ 
times with parameter $\eps_{i-1}/2$.

By Lemma~\ref{lm subspace finding v1}, it follows that 
in the $i$th iteration, each time we run Algorithm~\ref{alg subspace finding v1}, 
it draws $ M_i=d^{O(k)}2^{O(k\log k+m)}/\eps_{i-1}^2$ examples from $D$, 
runs in $\poly(M_i)$ time, and outputs a subspace $U^{(g)}$ of 
dimension $2^{O(k\log k)\log^{2k}(d)}/\eps_{i-1}^2$. So, in the $i$th iteration, 
the total number of examples we draw from $D$ is $T_iM_i$, 
the running time is $T_i \poly(M_i)$ and 
$\dim (U^{(i)}) = T_i 2^{O(k\log k)\log^{2k}(d)}/\eps_{i-1}^2$.

Since $\eps_0 = \eps$ and $\eps_i = (2^kC^m N_{i})^{-1}\eps_{i-1}/2$ for $i \in [k]$, we get 
the dimension of $V$ is at most $N_{k,m}(\eps) = (2^{O(k+m)}\log^k(d)/\eps)^{2^{O(k)}}$ and the 
sample/time complexity of the algorithm is $\poly(d^k, C_kN_{k,m}(\eps),\log(1/\delta))$.
\end{proof}

\begin{remark}
     We remark that the dimension of $V$ found by Lemma~\ref{lm subspace v3} has a $\polylog(d)$ factor dependence on the dimension $d$. Such a dependence appears because Lemma~\ref{lm gradient norm} has a $\polylog(d)$ factor dependence. If the KLS conjecture holds, then we can get an improved version of Lemma~\ref{lm gradient norm} without the dependence on $d$, which implies we can remove the dependence of $\polylog(d)$ in Lemma~\ref{lm subspace v3}. However, for our application, we will see later that we can recursively apply Lemma~\ref{lm subspace v3} to the discovered subspace $V$ to reduce the $\polylog(d)$ dependence to $\log\log(d)$ dependence.
\end{remark}


\subsection{From Dimension Reduction to Moment Discrepancy Finding}\label{sec mismatch finding}

In this section, we will apply Lemma~\ref{lm subspace v3} to find a direction $\w$ along which $D_\x$ and $(D_f)_\x$ have mismatched moments. To begin with, we will specify the polynomial function $Q(\x)$ 
we want to use.



\begin{lemma}\label{lm Q function learning}
There is an algorithm such that, given sample access to a logconcave distribution $D_\x$ over $\R^d$, oracle access to an event $A$ with $\Pr_{D_\x}[A] = \alpha$, and parameters $k \in \N, \eps,\delta \in (0,1)$,  
it draws $\poly(d^k,1/\eps^k,\log(1/\delta),1/\alpha)$ samples from $D_\x$, runs in sample-polynomial time, 
and outputs a degree-$k$ polynomial $Q: \R^d \to \R$ with 
    $\norm{Q}_{D_\x,2} \le O(\log^{k}(1/\alpha))+\eps$, 
    such that with probability at least $1-\delta$, 
    for every degree at most $k$ polynomial $p: \R^d \to \R$, it holds 
    $\abs{\E_{\x\sim D_\x}[Q(\x)p(\x)]-\E_{\x\sim (D_f)_\x}[p(\x)]} \le \eps \norm{p(\x)}_{D_\x,2}$, where $(D_f)_\x$ is the conditional distribution of $D_\x$ over event $A$.
\end{lemma}

\begin{proof}[Proof of Lemma~\ref{lm Q function learning}]
    We first show the existence of such a polynomial. 
    Define an inner product over the space of polynomial functions over $\R^{d}$ 
    as $\langle p,q\rangle_{D_\x}:=\E_{\x\sim D_\x}[pq]$ 
    as well as $\norm{\cdot}_{D_\x,2}$, the norm induced by the inner product, 
    and denote by $P_k$ the subspace of all polynomials with degree at most $k$ over $\R^d$. 
    Under the inner product $\langle p,q\rangle_{D_\x}$, 
    define $Q^*(\x): = \proj_{P_k}\Ind\{A\}/\alpha$. 
    Notice that for every polynomial $p(\x)$ with at most degree $k$, 
    \begin{align*}
         \E_{\x\sim (D_f)_\x}[p(\x)] = \E_{\x\sim D_\x}[p(\x)\Ind\{A\}/\alpha] = \E_{\x\sim D_\x}[Q^*(\x)p(\x)] \;,
    \end{align*}
where the last equation uses the fact that $\langle Q^*(\x) - \Ind\{A\}/\alpha,p(\x) \rangle_{D_\x} = 0$.
Furthermore, we have 
\begin{align*}
    \norm{Q^*(\x)}_{D_\x,2} = \sup_{p \in P_k, \norm{p}_{D_\x,2}=1}\langle p(\x),Q^*(\x)\rangle_{D_\x} = \sup_{p \in P_k, \norm{p}_{D_\x,2}=1} \E_{\x\sim D_\x}[p(\x)\Ind\{A\}/\alpha] \;.
\end{align*}
For $r \ge 2$, let $r' = (1-r^{-1})^{-1}$.
By H\"{o}lder's inequality, we have 
\begin{align*}
    \abs{\alpha^{-1}\E_{\x\sim D_\x}[p(\x)\Ind\{A\}]} & \le \alpha^{-1} \E_{\x\sim D_\x}[p^r(\x)]^{1/r}\E_{\x\sim D_\x}[\Ind\{A\}^{r'}]^{1/r'} \\
    &= \alpha^{-1/r}\E_{\x\sim D_\x}[p^r(\x)]^{1/r} \le \alpha^{-1/r}(Cr)^k \;.
\end{align*}
Here, in the last inequality, we use Fact~\ref{fact polynomial growth}. By choosing $r = O(\log(1/\alpha))$, we get $\norm{Q^*(\x)}_{D_\x,2} \le O(\log^k(1/\alpha))$.

It suffices to find a polynomial $Q(\x) \in P_k$ such that $\norm{Q-Q^*}_{D_\x,2}<\eps$. To do this, solve the polynomial regression problem that finds a polynomial $Q \in P_k$ such that 
\begin{align*}
    \E_{\x\sim D_\x}[(Q(\x)-\Ind\{A\}/\alpha)^2] 
    &\le \min_{p(\x) \in P_k}\E_{\x\sim D_\x}[(p(\x)-\Ind\{A\}/\alpha)^2] + \eps^2 \\
    &= \E_{\x\sim D_\x}[(Q^*(\x)-\Ind\{A\}/\alpha)^2]+\eps^2 \;,
\end{align*}
which needs $\poly(d^k/(\alpha\eps))$ sample from $D_\x$, using standard sample complexity analysis of linear regression (see, e.g., \cite{bartlett2002rademacher}.)
Notice that 
\begin{align*}
    \E_{\x\sim D_\x}[(Q(\x)-Q^*(\x)+Q^*(\x)-\Ind\{A\}/\alpha)^2] = \E_{\x\sim D_\x}[(Q(\x)-Q^*(\x))^2] + \E_{\x\sim D_\x}[(Q^*(\x)-\Ind\{A\}/\alpha)^2] \;.
\end{align*}
This means the solution of the polynomial regression gives a polynomial $Q(\x) \in P_k$ such that $\norm{Q(\x)-Q^*(\x)}_{D_\x.2} \le \eps$ and $\norm{Q}_{D_\x,2} \le O(\log^k(1/\alpha))+\eps$. 

Furthermore, for every polynomial $p(\x) \in P_k$, we have 
\begin{align*}
    \abs{\E_{\x\sim D_\x}[Q(\x)p(\x)]-\E_{\x\sim (D_f)_\x}[p(\x)]} = \abs{\E_{\x\sim D_\x}[(Q(\x)-Q^*(\x))p(\x)]}\le \eps\norm{p(\x)}_{D_\x,2}.
\end{align*}
\end{proof}



With Lemma~\ref{lm subspace v3} and Lemma~\ref{lm Q function learning} at hand, 
we are able to directly prove Theorem~\ref{th direction finding}.

\begin{algorithm}[htbp]
		\caption{\textsc{MismatchFinding} (Finding a direction with a moment discrepancy)}\label{alg mismatch finding}
		\begin{algorithmic} [1]
\State \textbf{Input:}  Sample access from $D$, event $A$, degree $k$, parameter $\Delta$.
\State \textbf{Output:} $\v \in \s^{d-1}$ or $\perp$, indicating $D$ and $D_f$ have nearly matched moments. 

\medskip 

\State Apply Lemma~\ref{lm Q function learning} to $D_\x$ and $A$ with parameter $\eps_1 = 2^{-O(k^2)}\Delta$ and degree $k$ to get a degree $k$-polynomial $Q(\x)$.
\State Run Algorithm~\ref{alg subspace finding v3} with input $Q(\x)-1$, degree $k$ and parameter $\eps = O(\Delta \log^{-k}(1/\alpha))$ to get a subspace $U$
\State Let $D^{(U)}_\x$ be the distribution of $\proj_U(\x)$
\State Apply Lemma~\ref{lm Q function learning} to $D^{(U)}_\x$ and $A$ with parameter $\eps = 2^{-O(k^2)}\log^{-k}(1/\alpha)\Delta$ and degree $k$ to get a degree $k$-polynomial $Q'(\x)$.
\State Run Algorithm~\ref{alg subspace finding v3} with input polynomial $Q'(\x)-1$, degree $k$ and parameter $\eps = 2^{-O(k^2)}\log^{-k}(1/\alpha)\Delta$ to get a subspace $V$.
\State Run a brute-force search over $V$ to find a vector $\v \in \s^{d-1}$ such that $\abs{\E_{\x \sim D_\x}[(\v\cdot\x)^i]-\E_{\x \sim D_f}[(\v\cdot\x)^i]}>2^{-O(k^2)}\Delta$, for some $i \in [k]$. 
\State If such a vector is not found, then return $\perp$, otherwise return $\v$.
\end{algorithmic}
\end{algorithm}

\begin{proof}[Proof of Theorem~\ref{th direction finding}]

    

By Fact~\ref{fact moment logconcave}, we know that for every $\w \in \s^{d-1}$, $\norm{(\w\cdot\x)^i}_{D_\x,2} \le (Ck)^k$, for $i \in [k]$
So, by applying Lemma~\ref{lm Q function learning} with $\eps_1 =2^{-O(k^2)}\Delta$ over $D_\x$ and $A$
we can find a degree $k$-polynomial $Q(\x')$ with $\norm{Q(\x)}_{D_\x,2}
\le O(\log^{k}(1/\alpha))$ such that for every unit vector $\w \in \s^{d-1}$, and $i \in [k]$
    \begin{align*}
        \abs{\E_{\x\sim D_\x}[Q(\x)(\w\cdot\x)^i]-\E_{\x\sim (D_f)_\x}[(\w\cdot\x')^i]} \le 2^{-O(k^2)}\Delta.
    \end{align*}

Consider the $k$-symmetric linear form $f(\v^{(1)},\dots,v^{(k)}) = \E_{\x \sim D_\x}[(Q(\x)-1)\prod_{i=1}^k(\v^{(i)}\cdot\x)]$. By \cite{bochnak1971polynomials}, for every subspace $V$, $\sup_{\v^{(1)},\dots,\v^{(k)} \in V \cap \s^{d-1}}f(\v^{(1)},\dots,v^{(k)}) = \sup_{\v \in V \cap \s^{d-1}}f(\v,\dots,\v)$.
By applying Lemma~\ref{lm subspace v3} on $Q(\x)-1$ with distribution $D_\x$, we find a subspace $U$ with dimension $(2^{O(k)}\log^{k}(1/\alpha)\log^k(d)/\Delta)^{2^{O(k)}}$,
such that
\begin{align*}
    \sup_{\w \in \s^{d-1}} \abs{\E_{\x\sim D_\x}[(Q(\x)-1)(\w\cdot\x)^k)]} \le 2^{O(k^2)} \sup_{i \in [k]}\sup_{\v \in U \cap \s^{d-1}} \abs{\E_{\x\sim D_\x}[(Q(\x)-1)(\v\cdot\x)^i)]} + \Delta/100.
\end{align*}
If there is some $\w \in \s^{d-1}$ such that $\abs{\E_{\x \sim D_\x}[(\w\cdot\x)^k]-\E_{\x \sim (D_f)_\x}[(\w\cdot\x)^k]}>\Delta$, 
there exists a unit vector $\u \in U$ and $i \in [k]$ such that
$ \abs{\E_{\x\sim D_\x}[(Q(\x)-1)(\u\cdot\x)^i)]} \ge 2^{-O(k^2)}\Delta$ and thus
\begin{align*}
\abs{\E_{\x \sim D_\x}[(\u\cdot\x)^i]-\E_{\x \sim (D_f)_\x}[(\u\cdot\x)^i]} \ge 2^{-O(k^2)}\Delta.
\end{align*}

Denote by $D^{(U)}_\x$ the distribution of $\a=\proj_{U} \x$ and $D^{(U)}_f$ the conditional distribution of $D^{(U)}_\x$ over $A$. For every $\u \in U$ and $i \in [k]$, we have 
\begin{align*}
    \E_{\x \sim D_\x}[(\u\cdot\x)^i]-\E_{\x \sim (D_f)_\x}[(\u\cdot\x)^i] = \E_{\a \sim D^{(U)}_\x}[(\u\cdot\a)^i]-\E_{\a \sim (D^{(U)}_f)_\x}[(\u\cdot\a)^i].
\end{align*}
Since $D^{(U)}_\x$ is a $1/2$-isotropic logconcave distribution, by applying Lemma~\ref{lm Q function learning} to $D^{(U)}_\x$ with $\eps_2 = 2^{-O(k^2)}\Delta$, 
we can find a degree $k$-polynomial $Q'(\a)$ with $\norm{Q'}_{D^{(U)}_\x,2}
\le O(\log^{k}(1/\alpha))$ such that for every unit vector $\w \in U$ and $i \in [k]$,
    \begin{align*}
        \abs{\E_{\a\sim D^{(U)}_\x}[Q'(\a)(\w\cdot\a)^i)]-\E_{\a\sim D^{(U)}_f}[(\w\cdot\a)^i]} \le 2^{-O(k^2)}\Delta.
    \end{align*}
By applying Lemma~\ref{lm subspace v3} on $Q'(\a)-1$ and distribution $D^{(U)}_\x$
we find a subspace $V$ with dimension $(2^{O(k^2)}\log^k(1/\alpha)\log^k\log(d)/\Delta)^{2^{O(k)}}$
such that
\begin{align*}
    \sup_{\w \in U, \norm{\w}_2=1 } \abs{\E_{\a\sim D^{(U)}_\x}[(Q'(\a)-1)(\w\cdot\a)^k)]} \le 2^{O(k^2)} \sup_{i \in [k]}\sup_{\v \in V, \norm{\v}_2=1} \abs{\E_{\a\sim D^{(U)}_\x}[(Q'(\a)-1)(\v\cdot\a)^i)]} + 2^{-O(k^2)}\Delta.
\end{align*}
Since
there exists a unit vector $\u \in U$ such that
$ \abs{\E_{\x\sim D_\x}[(Q(\x)-1)(\u\cdot\x)^k)]} \ge 2^{-O(k^2)}\Delta$, we know that
there exists a unit vector $\v \in V$ such that
$ \abs{\E_{\x\sim D_\x}[(Q'(\x)-1)(\v\cdot\x)^i)]} \ge 2^{-O(k^2)}\Delta$ and thus
\begin{align*}
\abs{\E_{\x \sim D_\x}[(\v\cdot\x)^i]-\E_{\x \sim D_f}[(\v\cdot\x)^i]}>2^{-O(k^2)}\Delta \;.
\end{align*}
By running a brute force search over $V$, we find such a vector $\v$. In particular, if no such vector $\v$ is output, then it must be the case that for every $\w \in \s^{d-1}$, $\abs{\E_{\x \sim D_\x}[(\w\cdot\x)^k]-\E_{\x \sim D_f}[(\w\cdot\x)^k]}<\Delta$.

Finally, we upper bound the sample and time complexity of the algorithm.
In Algorithm~\ref{alg mismatch finding}, we run Algorithm~\ref{alg subspace finding v3} and Lemma~\ref{lm Q function learning} twice. Each run of Lemma~\ref{lm Q function learning} takes $\poly(2^{(k^2)}d^k/(\alpha\Delta))$ sample and time.
Each run of Algorithm~\ref{alg subspace finding v3} takes at most $d^{O(k)}(2^{O(k)}\log^k(1/\alpha)\log^k(d)/\Delta)^{2^{O(k)}}$ sample and time. 
Finally, since the final found subspace $V$ has dimension $(2^{O(k^2)}\log^k(1/\alpha)\log^k\log(d)/\Delta)^{2^{O(k)}}$, the brute force search takes time
$2^{O(k^2\log(1/\Delta)\dim(V))}= 2^{(\log(1/\alpha)\log\log(d)/\Delta)^{2^{O(k)}}}$. 
Thus, the algorithm draws $d^{O(k)}(\log(1/\alpha)\log(d)/\Delta)^{2^{O(k)}}$ sample and runs in $d^{O(k)}(\log(1/\alpha)\log(d)/\Delta)^{2^{O(k)}}+2^{(\log(1/\alpha)\log\log(d)/\Delta)^{2^{O(k)}}}$ time.  
\end{proof}


\section{ Proof of Theorem~\ref{thm: main-weak-to-strong} and Testable Learning Extensions}\label{sec proof main}

In this section, we put together results established in the previous sections to complete the proof of Theorem~\ref{thm: main-weak-to-strong} and also provide our testable learning extensions.


\begin{proof}[Proof of Theorem~\ref{thm: main-weak-to-strong}]

    By Theorem~\ref{th direction finding}, for any $\alpha,\eps \in (0,1)$ and for any fixed constant $c(k,\Delta)\ge 1$, 
    there is a $(k,\Delta,2^{-O(k^2)},(\alpha\eps)^{c(k,\Delta)})$-discrepancy finder $\A_m$ for $D_\x$ and $(D_f)_\x$, with sample complexity 
    $d^{O(k)}\poly_{k,\Delta}((\alpha\eps)^{-1})$ and time complexity 
    $d^{O(k)} 2^{\polylog_{k,\Delta}((\alpha\eps)^{-1})}$, 
whenever $W(D)$---the weight of $D$ over $\R^d \times \{0,\pm 1\}$---is 
at least $(\alpha\eps)^{c(k,\Delta)}$.


Let $\eps' = 9\eps/10$.
Let $D$ be any distribution $\R^d \times \{0,\pm 1\}$ with weight $(\alpha\eps')$ and $D_\x$ a logconcave distribution. 
By Theorem~\ref{th decision tree k moment}, applied with $\eta = \eps/20$ and $\A_m$ as the discrepancy finder,
it follows that whenever $\opt_{D_f}(\H) \le (1-\eta)/C$, 
Algorithm~\ref{alg decision tree k moment} draws 
$d^{O(k)}\poly_{k,\Delta}((\alpha\eps)^{-1})$ samples from $D$, runs in $\poly(d^k)2^{\polylog_{k,\Delta}((\alpha\eps)^{-1})}$ time
and with probability at least $99\%$ outputs a hypothesis $\hat{h}$ with $\err_{D_f}(\hat{h}) \le 1/2-\Omega(\beta\eta)$.
In other words, there is a $(\Omega(\eps\beta),C/(1-\eta))$-weak partial learner for $\H$ under $D_\x$ with sample complexity $d^{O(k)}\poly_{k,\Delta}((\alpha\eps)^{-1})$ and time complexity $\poly(d^k)2^{\polylog_{k,\Delta}((\alpha\eps)^{-1})}$.

By Theorem~\ref{th condition boosting}, it follows that 
by calling the weak partial learner 
$\tilde O((\eps\beta)^{-2}/\eps')$ times, 
we obtain a hypothesis $\hat{h}$ with 
$\err_{D_f}(\hat{h}) \le C\opt_{D_f}(\H)/(2(1-\eta))+\eps'$.
Without loss of generality, we assume that $(C/2)\opt_{D_f}(\H)\le 1/2$; otherwise, the guarantee is vacuous.
This implies that 
\begin{align*}
 \err_{D_f}(\hat{h}) \le C\opt_{D_f}(\H)/(2(1-\eta))+\eps' \le (C/2)\opt_{D_f}(\H) + 2\eta + \eps' \le (C/2)\opt_{D_f}(\H)+\eps \;,   
\end{align*}
where the last inequality follows by our choice of $\eps'$. This concludes the proof of Theorem~\ref{thm: main-weak-to-strong}.    
\end{proof}


\paragraph{Testable Learning}
The algorithm establishing Theorem~\ref{thm: main-weak-to-strong} automatically gives an algorithm with the same complexity that succeeds 
in the testable learning framework~\cite{rubinfeld2023testing}, 
for any {\em fixed} logconcave distribution. 

Before we state our relevant theorem, we start with the definition of testable learning. 


\begin{definition}[Testable Agnostic Partial Learning under Logconcave Distributions] \label{def:testable}
Let $\mathcal{H}$ be a class of Boolean (i.e., $\{\pm 1\}$-valued)
functions over $\R^d$.
Let $D$ be a distribution over $\R^d \times \{\pm 1,0\}$ and 
$D_f$ be the conditional distribution of $D$ over the event 
$\mathcal{E} \eqdef \{(\x,y)\mid y \neq 0\}$. The error of $h: \R^d \to \{\pm 1\}$ 
is defined as $\err_{D_f}(h) = \Pr_{(\x,y)\sim D_f}[h(x) \neq y]$ and let 
$\opt_{D_f}(\H) \eqdef \inf_{h \in \H} \err_{D_f}(h)$.
Given a set of \iid examples $S$ drawn from $D$ 
and parameters $\eps,\delta \in (0,1)$, a testable learner with respect to a \emph{given} base distribution $D_B$ over $\R^d$ can either reject $S$
or accept and output a hypothesis $\hat{h}: \R^d \to \{\pm 1\}$.
Further, the following conditions must be met:
\begin{enumerate}
    \item(Soundness.) With probability at most $\delta$,
    the learner accepts but outputs a hypothesis $\hat{h}$, such that 
    $\err_{D_f}(\hat{h}) > F(\opt_{D_f}(\H))+\epsilon$, 
for some function $F:(0,1) \to (0,1)$. 
\item(Completeness.) Whenever $D_\x = D_B$, for some known logconcave distribution $D_B$ over $\R^d$, the learner will accept with probability at least $1-\delta$. 
\end{enumerate}
 For short, we say such a learner testably learns $\H$ with respect to $D_B$ to error $F(\opt_{D_f}(\H))+\epsilon$ with probability $1-\delta$.
\end{definition}

Our extension of Theorem~\ref{thm: main-weak-to-strong} to 
the testable learning setting is as follows.

\begin{theorem}[Testable Agnostic Partial Learning]\label{th testable}
Consider the task of testable agnostic partial learning $\H$ under logconcave distributions. 
    Suppose that $\H$ is $(k,\Delta,\beta,C)$-weakly learnable 
    by an algorithm $\A$ with sample complexity $m$ and time complexity $T$. 
    Then there is a testable agnostic partial learning algorithm for $\H$ with respect to a logconcave distribution $D_B$
    with the following guarantee: 
    given any $\eps\in (0,1)$ and sample access to any distribution $D$ 
    over $\R^d \times \{0,\pm 1\}$ of weight $W(D) \geq \alpha$ and sample access to $D_B$, the algorithm 
    draws $(1/\beta^{2}) \, d^{O(k)}\poly_{k,\Delta}(1/(\alpha\eps)) \, m$ samples 
    from $D$ and $D_B$, 
    runs in time $(1/\beta^{2}) \, d^{O(k)} \, 2^{\poly_{k,\Delta}(\log(1/(\alpha\eps))}\, T$, 
    and testably learns $\H$ with respect to $D_B$ to error $(C/2)\opt_{D_f}(\H) + \eps$ with probability at least $99\%$.
\end{theorem}


Our algorithm for Theorem~\ref{th testable} will be essentially the same as the algorithm for Theorem~\ref{thm: main-weak-to-strong} with the following modification. In the process of running Algorithm~\ref{alg decision tree k moment}, we compute a partition of space so that (if our distributional hypotheses are satisfied) almost all of the mass of $(D_f)_\x$ lies in regions where the central moments of $(D_f)_\x$ and $D_B$ approximately match (as is certified by our discrepancy finder). For the testable learning setting, as a high-level idea, we will essentially run Algorithm~\ref{alg decision tree k moment} identically, except every time we draw a sample from $D_\x$, we draw it from $D_B$, and we stop the algorithm after using an appropriate amount of time and samples. If a partition is generated successfully before we stop the algorithm, then it is likely that we can get a weak partial hypothesis using this partition; otherwise, we just reject. In particular, if $D_\x = D_B$, we will accept with high probability, as the input is the same as the one in the standard learning setting.

To formalize the above intuition, we will define a testable discrepancy finder as follows.
\begin{definition}[$(k,\Delta,c)$-Testable Discrepancy Finder] \label{def:discr-gen}
    For $k \in \Z_+, \Delta>0, c, \in (0,1)$,
    we say that an algorithm $\A$ is a $(k,\Delta,c)$-testable discrepancy finder for $D_\x$,
    if it satisfies the following performance guarantee: 
    given sample access to a $1/2$-isotropic logconcave distribution $D_\x$ over $\R^d$ 
    and a distribution $D'_\x$ over $\R^d$, and accuracy parameter $\eps \in (0,1/2)$ and weight parameter $\alpha \in (0,1)$,
    it either rejects or accepts.  With probability at least $99\%$ over the sample access, if it accepts, it will output a distribution $D_m$ over $\R^d$, with $\dtv(D_m,D'_\x)<\eps$, and a direction $\v \in \s^{d-1}$ and $i \in [k]$ 
    such that $\abs{\E_{\x\sim D_\x}[(\v\cdot \x)^i]-\E_{\x\sim D_m}[(\v\cdot\x)^i]} \geq c\Delta$,  if there exists $\w \in \s^{d-1}$ 
    and $j \in [k]$ such that 
    $\abs{\E_{\x\sim D_\x}[(\w\cdot\x)^j]-\E_{\x\sim D_m}[(\w\cdot \x)^j]} \geq \Delta$.
    Furthermore, if $D'_\x$ is a conditional distribution of $D_\x$ over some event with mass at least $\alpha$, with probability at least $99\%$, the algorithm will accept.
\end{definition}

\begin{lemma}\label{lm testable discepancy finder}
For every $1/2$-isotropic logconcave distribution $D_\x$ over $\R^d$, for every $k \in \Z_+, \Delta \in (0,1)$, there is a $(k,\Delta,2^{-O(k^2)})$ testable discrepancy finder for $D_\x$ with sample complexity $d^{O(k)}\poly_{k,\Delta}(1/(\alpha\eps))$ and time complexity $\poly(d^k/\eps)2^{\polylog_{k,\Delta}(1/\alpha)}$.
\end{lemma}

\begin{proof}[Proof of Lemma~\ref{lm testable discepancy finder}]
    By Lemma~\ref{lm subspace v3}, Lemma~\ref{lm Q function learning} and Theorem~\ref{th direction finding},
    suppose there is a degree $k$ polynomial $Q(\x)$ with $\norm{Q}_{D_B,2} \le O(\log^k(1/\alpha))$, such that  $\abs{\E_{\x\sim D_B}[Q(\x)p(\x)]-\E_{\x\sim D_m}[p(\x)]} \le 2^{-O(k^2)}\Delta $ for every degree $k$ polynomial $p(\x)$ of the form $p(\x)= \prod_{i=1}^k(\w^{(i)}\cdot\x)$, where $\w^{(i)} \in \s^{d-1}$,
    then Algorithm~\ref{alg mismatch finding} with such a function $Q$, will output in time $d^{O(k)}2^{\polylog_{k,\Delta}(1/\alpha)}$, a direction $\v \in \s^{d-1}$ and $i \in [k]$ 
    such that $\abs{\E_{\x\sim D_B}[(\v\cdot \x)^i]-\E_{\x\sim D_m}[(\v\cdot\x)^i]} \geq c\Delta$,  if there exists $\w \in \s^{d-1}$ 
    and $j \in [k]$ such that 
    $\abs{\E_{\x\sim D_B}[(\w\cdot\x)^j]-\E_{\x\sim D_m}[(\w\cdot \x)^j]} \geq \Delta$.

Given sample access to $D'_\x$, we will first use $\poly(1/(\alpha\eps))$ example to check 
if $\Pr_{\x\sim D'_\x}[\norm{\x}_2>100C\log(1/\alpha\eps)\sqrt{d}]<\alpha\eps/100$. If not, 
we will simply reject.
Since $D_B$ is a $1/2$-isotropic logconcave distribution, 
it follows from Fact~\ref{fact norm} that for $t>1$, 
we have $\Pr_{\x \sim D_B}[\norm{\x}_2>Ct\sqrt{d}]<\exp(-t\sqrt{d})$, 
which implies that if $D'_\x$ is a conditional distribution of $D_B$ 
over some event with mass $\alpha$, then 
$\Pr_{\x\sim D'_\x}[\norm{\x}_2>100C\log(1/\alpha\eps)\sqrt{d}]<\alpha\eps/100$ 
and we will not reject with high probability at this step.

Notice that the distribution $D_m$ defined as $D'_\x$ conditioned over examples with norm $\norm{\x}_2<r:=100C\log(1/\alpha\eps)\sqrt{d}$ satisfies $\dtv(D_m,D'_\x)<\eps$. To complete the proof, it suffices to find a $Q$ function discussed above.
By Lemma~\ref{lm Q function learning}, if $D'_\x$ is a conditional distribution of $D_B$ over an event with mass $\alpha$, such a function $Q$ exists. So, it suffices to find such a polynomial $Q$ using only sample access to $D_m$ and $D_B$. Let $B_{d,k} = d^{O(k)}$.
    We will use linear program to find the coefficient $(q_\zeta)_\zeta$ of polynomial $Q(\x) = \sum_\zeta q_\zeta \x^\zeta$, with $\norm{Q}_{D_B}\le O(\log^k(1/\alpha))$
    such that for each monomial $\x^\zeta$,
    \begin{align*}
        \abs{\E_{\x\sim D_B}[Q(\x) \x^\zeta]-\E_{\x\sim D_m} [\x^\zeta]} \le B_{d,k}^{-1}2^{-O(k^2)}\Delta .
    \end{align*}
    Notice that suppose this holds, then for every polynomial $p(\x) = \prod_{i=1}^k(\w^{(i)}\cdot\x)= \sum_{\zeta}a_\zeta \x^\zeta$, we have 
    \begin{align*}
    \abs{\E_{\x\sim D_B}[Q(\x)p(\x)]-\E_{\x\sim D_m}[p(\x)]} \le B_{d,k}^{-1}2^{-O(k^2)}\Delta \sum_\zeta \abs{a_\zeta} \le 2^{-O(k^2)}\Delta. 
    \end{align*}
Here, we use the fact that $\sum_{\zeta}\abs{a_\zeta}\le d^{O(k)}$.
The above linear program can be built by using samples from $D_B$ and $D_m$ to estimate $\E_{\x \sim D_B}[\x^\zeta]$ and $\E_{\x \sim D_m}[\x^\zeta]$ to error $B_{d,k}^{-1}2^{-O(k^2)}\Delta/4$.
We remark that if such a function $Q(\x)$ is not found, then we can just reject.
Since $D_B$ is logconcave, for each monomial, we can take the empirical mean of $\x^\zeta$ using $2^{-O(k^2)}\poly(d^k/\Delta)$ examples to do this estimation.
Since $\x \sim D_m$ has bounded norm $r$, we know that we can take $r^{O(k)}2^{O(k^2)}\poly(d^k/\Delta)$ examples from $D_m$, we can estimate the expectation of $\x^\zeta$ to the desired accuracy. This completes the proof.
\end{proof}

\begin{proof}[Proof of Theorem~\ref{th testable}]
We run the algorithm in Theorem~\ref{thm: main-weak-to-strong} with a modified version of Algorithm~\ref{alg decision tree k moment}. Let $T,\gamma,\tau,$ and $\xi$ be as in Algorithm~\ref{alg decision tree k moment}.
Every time Algorithm~\ref{alg decision tree k moment} draws an example from $D_\x$, instead, we draw it from $D_B$. Every time Algorithm~\ref{alg decision tree k moment} calls a discrepancy finder for $(D_R)_\x$ and $((D_R)_f)_\x$ over some region $R$, we will call a 
$(k,\Delta,c)$-testable discrepancy finder for $(D_B)\mid_R$ with input distribution $((D_R)_f)_\x$ with parameter weight parameter $\xi\eps$ and accuracy parameter $\eps'=(c\Delta\xi^2 \eps^2)/T^{2k}$.
 We will terminate the algorithm mutually and reject if its running time exceeds $d^{O(k)}2^{\poly_{k,\Delta}(1/\eps)}$ or draws more than $d^{O(k)}\poly_{k,\Delta}(1/\eps)$ samples or it stops because the testable discrepancy finder rejects or a good split is not found.

If a partition $\mathcal{U}$ is output, we will check if $\sum_{R \in \mathcal{B}}\Pr_{(\x,y)\sim D_f}[\x \in R] \le \eta/2$, where $\mathcal{B}=\{R \in \mathcal{U} \mid \Pr_{(\x,y)\sim D_f}[\x \in R] \le \xi\}$. We will reject if it is not the case, and otherwise we will run the weak learner over each region $R$ with $\Pr_{(\x,y)\sim D_f}[\x \in R] > \xi$ over the distribution output by the testable discrepancy finder. Since the distribution $(D_R)_m$ output by the testable discrepancy finder has TV distance at most $\eps'$, over all such regions $R$, we only throw away at most $\xi$ fraction of the mass. Furthermore, by the definition of the testable discrepancy finder, for every  $R \in \mathcal{U}$, with $\Pr_{(\x,y)\sim D_f}[\x \in R]>\xi$, it holds for every $\u \in \s^{d-1}$ and for every $i \in [k]$,
$\abs{\E_{\x\sim (D_B)_R}(\u\cdot \x)^i-\E_{(\x,y)\sim (D_R)_m}(\u\cdot \x)^i} \le \Delta,$ which implies that the modified version of Algorithm~\ref{alg decision tree k moment} will output a hypothesis with $\Omega(\eta\beta)$ advantage over $D_f$.
Thus, the algorithm used for Theorem~\ref{thm: main-weak-to-strong}, will succeed. Furthermore, since the testable discrepancy finder will accept with a high probability when $D_\x = D_B$, we know that the whole testable learning algorithm will accept when $D_\x = D_B$ with a high probability.
\end{proof}

\section{Algorithmic Applications to Concrete Hypothesis Classes}\label{sec application}

In this section, we  give 
algorithmic applications of Theorem~\ref{thm: main-weak-to-strong} 
to several fundamental families of multi-index models.


\paragraph{Organization of This Section}
In Section~\ref{sec tools}, we will give basic technical lemmas that 
will be used for designing algorithms in this section. 
In Section~\ref{sec generic class}, we introduce a generic 
family of concept classes that can be partially learned using Theorem~\ref{thm: main-weak-to-strong}. Specifically, we show that a wide range of  
``smooth'' $K$-MIMs $\H$ have an efficient 
$(\ell,\Delta,\beta,C)$-weak learner, 
where the parameters $\ell,\Delta,\beta,C$ only depend 
on geometric properties of the underlying class.
This implies that any such family of smooth MIMs 
admits an agnostic partial learner under logconcave distributions
with sample complexity $\poly_{K,\Delta}(d,1/\eps)$ 
and computational complexity 
$\poly_{K,\Delta}(d)2^{\polylog_{K,\Delta}(1/\eps)}$.
We remark that the parameters $(\ell,\Delta,\beta,C)$ 
for the weak learning algorithm above could be qualitatively quite large, 
if they are derived directly from the underlying geometric properties of the class. To rectify this, in Section~\ref{sec sandwhiching polynomial}, 
we show that MIMs that admit low-degree $L_1$-sandwiching 
polynomials can be learned to any constant factor approximation 
using our framework. As a corollary of this result, 
in Section~\ref{ssec:qptas-apps}, we obtain such a constant factor approximation for agnostically learning functions of halfspaces and low-dimensional low-degree PTFs. 
In Section~\ref{sec partial linear}, we use a refinement of our framework 
to obtain improved learning algorithms for agnostic partial learning of 
halfspaces under logconcave distributions. Finally, 
in Section~\ref{sec multiclass}, we leverage these algorithms 
to obtain efficient and robust learners for multiclass linear classification.



\subsection{Basic Tools for Weak Learning Algorithms}\label{sec tools}

In this subsection, we give some basic tools that we will use in the following subsections.
By Theorem~\ref{thm: main-weak-to-strong}, given a hypothesis class $\H$, to design an agnostic partial learner under logconcave distributions, it suffices to design an efficient 
$(k,\Delta,\beta,C)$-weak learner for it.

Our weak learner makes use of the following simple framework. 
Given a distribution $D$ over $\R^d \times \{\pm 1\}$,
if we can find a real-valued function $g(\x)$ such that 
$\E_{(\x,y)\sim D}[yg(\x)]$ is non-trivial, 
then we can sample a random threshold $t \in [-1,1]$ 
such that $\err_D(\sign(g(\x)\new{-t)})$ has small error with non-trivial probability.

To formally formulate this framework, we give the following lemma.

\begin{lemma}\label{lm robust chow parameter}
    Let $D$ be a distribution over $\R^d \times \{\pm 1\}$, 
    and let $\H$ be a hypothesis class over $\R^d$. 
    Let $h^* \in \H$ be such that $\err_D(h^*) = \opt = \inf_{h \in \H}\err_D(h)$. 
    For a real-valued function $p: \R^d \to \R$ and $\beta>0$, the following holds.
    \begin{enumerate}[leftmargin=*]
        \item If $\E_{\x \sim D_\x}[h^*(\x)p(\x)]>\beta$ and $\opt<\beta^2(\E_{\x \sim D_\x}[p^2(\x)])^{-1}/16$, then $\E_{(\x,y)\sim D}[yp(\x)]>\beta/2$.
        \item If $\E_{(\x,y)\sim D}[yp(\x)]>\beta$, then the function $g(\x): \R^d \to [-1,1]$, defined by
        \begin{equation} \label{eqn:g}
        g(\x)= \proj_{[-1,1]} \left( \beta p(\x)(\E_{\x \sim D_\x}[p^2(\x)])^{-1}/4 \right) \;,
        \end{equation}
        satisfies $\E_{(\x,y)\sim D}[yg(\x)]>\Omega(\beta^2/(\E_{\x \sim D_\x}[p^2(\x)]))$.
    \end{enumerate}        
\end{lemma}

\begin{proof}[Proof of Lemma~\ref{lm robust chow parameter}]
The first item of the statement follows from the inequalities below
\begin{align*}
    \E_{(\x,y)\sim D}[yp(\x)] & = \E_{(\x,y)\sim D}[h^*(\x)p(\x)] + \E_{(\x,y)\sim D}[(y-h^*(\x))p(\x)]\\
    &\ge \E_{(\x,y)\sim D}[h^*(\x)p(\x)] - 2\E_{(\x,y)\sim D}[\Ind\{y \neq h^*(\x)\} \abs{p(\x)}] \\
    &  \ge \E_{\x\sim D_\x}[h^*(\x)p(\x)] - 2\sqrt{\opt\E_{(\x,y)\sim D}[ p(\x)^2]} > \beta -\beta/4=3\beta/4. 
\end{align*}
Here, in the first inequality, we use the fact that $h^*$ takes values in $\{\pm 1\}$.
The second inequality follows from H\"{o}lder's inequality. 
The third inequality follows from the assumption that 
$\opt<\beta^2(\E_{\x \sim D_\x}[p^2(\x)])^{-1}/4$.

In the rest of the proof, we prove the second item of the statement.
Let $W:=4(\E_{\x \sim D_\x}[p^2(\x)])/\beta$. 
    We have 
    \begin{align*}
        \E_{(\x,y)\sim D}[yg(\x)] = \E_{(\x,y)\sim D}[yg(\x)\Ind\{\abs{p(\x)}<W\}]+\E_{(\x,y)\sim D}[yg(\x)\Ind\{\abs{p(\x)}>W\}].
    \end{align*}

By the choice of $W$, we have 
\begin{align*}
    \abs{\E_{(\x,y)\sim D}[yg(\x)\Ind\{\abs{p(\x)}>W\}]}  \le \Pr_{\x \sim D_\x}[\abs{p(\x)}>W] \le \E_{\x \sim D_\x}[p^2(\x)]/W^2 \le \beta^2/(16(\E_{\x \sim D_\x}[p^2(\x)])).
\end{align*}
Here the first inequality follows from the fact that $\abs{yg(\x)}\le 1$ 
and the second inequality follows from Markov's inequality.
Furthermore,
\begin{align*}
    \E_{(\x,y)\sim D}[yp(\x)\Ind\{\abs{p(\x)}<W\}] &= \E_{(\x,y)\sim D}[yp(\x)]-\E_{(\x,y)\sim D_\x}[yp(\x)\Ind\{\abs{p(\x)}>W\}] \\
    &\ge  \E_{(\x,y)\sim D}[yp(\x)]-\sqrt{\E_{\x\sim D_\x}[p^2(\x)]\Pr_{\x \sim D_\x}[\abs{p(\x)}>W]} \\
    & \ge \beta - \beta/2 = \beta/2 \;,
\end{align*}
    which implies by the definition of $g(\x)$ that
    \begin{align*}
        \E_{(\x,y)\sim D}[yg(\x)\Ind\{\abs{p(\x)}<W\}] \ge 3\beta^2/(16(\E_{\x \sim D_\x}[p^2(\x)])).
    \end{align*}
So, we conclude that $g(\x):\R^d \to [-1,1]$ satisfies $\E_{(\x,y)\sim D}[yg(\x)]>\Omega(\beta^2/(\E_{\x \sim D_\x}[p^2(\x)]))$.
\end{proof}

Lemma~\ref{lm robust chow parameter} shows the following: 
For an agnostic learning problem, suppose we have some real-valued function $p(\x)$ 
such that the correlation $\E_{\x\sim D_\x}[h^*(\x)p(\x)]$ is non-trivial, 
where $h^*(\x)$ is the hypothesis in the target class that achieves the optimal error. 
Then, as long as $\opt$ is reasonably small, we can find a function $g(\x)$ 
with values in $[-1,1]$ such that it has non-trivial correlation with respect 
to the noisy label $y$.
Once such a $[-1,1]$-valued function $g(\x)$ is found, we can easily round 
it to a Boolean function as a weak hypothesis over the noisy distribution $D$.

\begin{fact}\label{fact correlation}
Let $X$ be a domain of examples and let $D$ be a distribution over $X \times \{\pm 1\}$. Let $g(\x): X \to [-1,1]$ such that $\E_{\x\sim D}[yg(\x)]>\gamma$, then $\E_{t \sim [-1,1]}[\err_{D}(\sign(g(\x)-t))]<(1-\gamma)/2$.    
\end{fact}

\begin{proof}[Proof of Fact~\ref{fact correlation}]
    To simplify the notation, let $g_t(\x) = \sign(g(\x)-t), t \in [-1,1]$. Fix an example $(\x,y) \in X\times \{\pm 1\}$. If $y=1$, then we have $\Pr_{t\sim [-1,1]}[g_t(\x) \neq y] = (1-g(\x))/2$, while if $y=-1$, then we have $\Pr_{t\sim [-1,1]}[g_t(x) \neq y] = (1+g(\x))/2$. This implies that for every example $(\x,y)$, we have $\Pr_{t\sim [-1,1]}[g_t(\x) \neq y] = (1-yg(\x))/2$. Thus, we have 
    \begin{align*}
      \E_{t\sim [-1,1]}\left[\err_{D}(g_t)\right] = \E_{(\x,y)\sim D}\left[ \Pr_{t\sim [-1,1]}\left[g_t(\x) \neq y\right]\right] = \E_{(\x,y)\sim D} \frac{1-yg(\x)}{2} \le \frac{1-\gamma}{2} \;.
    \end{align*}
\end{proof}

\subsection{Generic Family of MIMs}\label{sec generic class}
By Lemma~\ref{lm robust chow parameter} and Fact~\ref{fact correlation}, to get an efficient weak learning algorithm for a class $\H$ under a distribution $D$,
it suffices to show that when the moment information of the marginal distribution is close to that of some logconcave distribution, 
then we can efficiently find a polynomial with non-trivial correlation 
with respect to the optimal hypothesis.

Based on this observation, we show that a broad family of MIMs 
that satisfies a natural ``smoothness'' condition 
can be efficiently weakly learned. 
Specifically, we introduce the following generic family of MIMs.



\begin{definition}[$(K,\eta)$-good Multi-Index Models] \label{def:nice-MIMs}
    We say that a family of Boolean functions $\H$ on $\R^d$ 
    is a $(K,\eta)$-good class of Multi-Index Models if 
    it satisfies the following properties: 
    \begin{enumerate}[leftmargin=*]
        \item 
        $\H$ is a family of $K$-Multi-index models ($K$-MIMs). That is, 
        for every $h \in \H$ there is a $K$-dimensional subspace $V \subset \R^d$  
        such that $h(\x) = h(\proj_V(\x))$. Equivalently, there is a class of Boolean functions $\mathcal{G}$ over $\R^{K}$ 
        such that for every $h \in \H$ there exists $g \in \mathcal{G}$ and a matrix $\mathbf{W} \in \R^{K \times d}$ 
        with $\mathbf{W}\mathbf{W}^\top = \mathbf{I}_{K \times K}$  
        so that $h(\x) = g(\mathbf{W} \x)$.
        \item  Every $g \in \mathcal{G}$, such that $h(\x) = g(\mathbf{W} \x)$ for some $h \in \H$,         
        satisfies the following condition:
        For every ellipsoid $E=\mathcal{E}(\u,A)$, in $\R^K$, 
        there is a ball $E'=\mathcal{E}(\v,A/\eta)$ such that $E' \subseteq E$ and $g(\x)$ 
        has a constant sign over $E'$.
    \end{enumerate}
\end{definition}

\begin{figure}[tbp]
    \centering
    \includegraphics[width=0.45\linewidth]{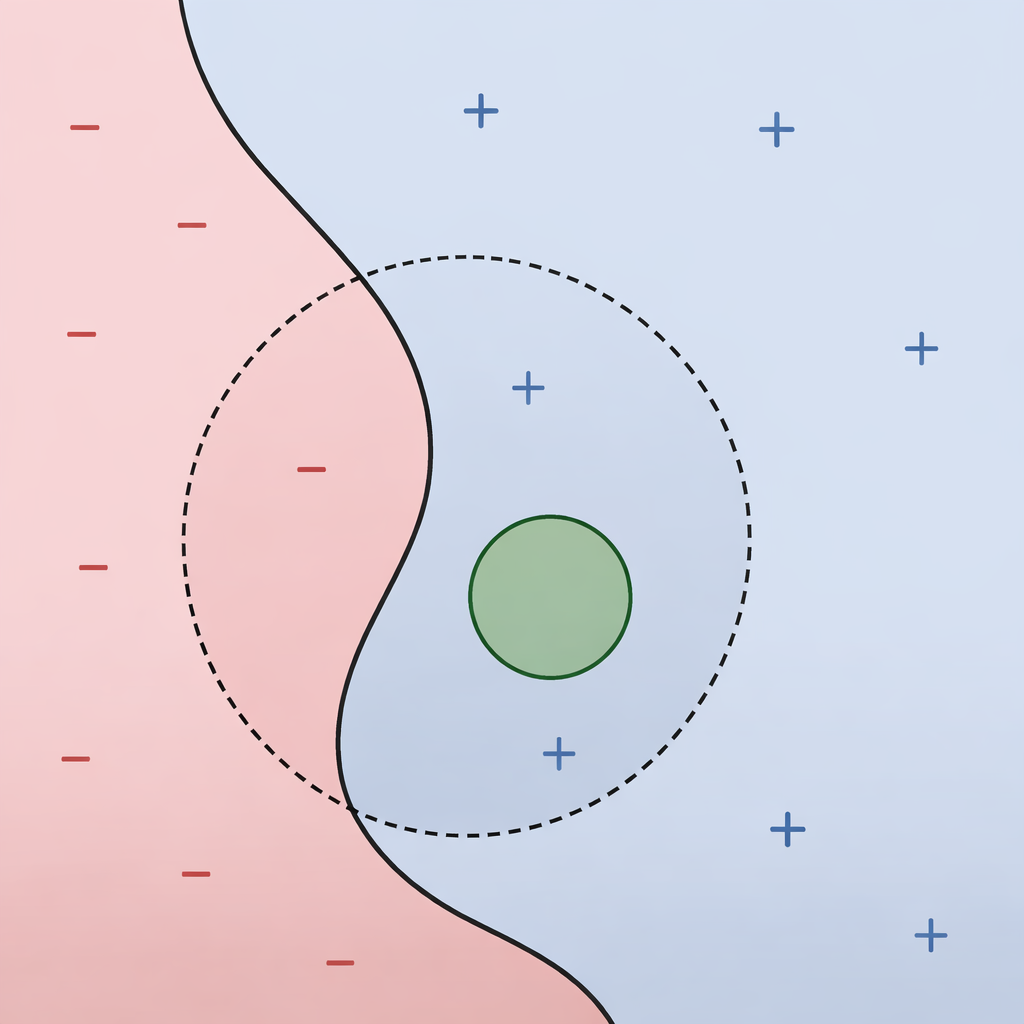}
    \caption{Generic Geometric Concept Class}
    \label{fig:placeholder}
\end{figure}

Intuitively, any function in a $(K,\eta)$-good class of MIMs 
has a reasonably smooth decision boundary, 
so that the classification of points will not change 
too frequently within a small ball.

As a concrete example, 
consider the class of halfspaces---a basic class of $1$-MIMs.
Given any interval $I \subseteq \R$ 
and any one-dimensional linear threshold function,
one can always find a subinterval $I'\subseteq I$, 
with half of the length, such that examples in $I'$ have the same sign.
This means the class of halfspaces is a $(1,1/2)$-good class of MIMs.


\medskip

Our main learning result for $(K,\eta)$-good classes of MIMs
is the following theorem.

\begin{theorem}[Agnostic Partial Learning of Good MIMs]\label{th generic family}
     Let $\H$ be a $(K,\eta)$-good class of MIMs over $\R^d$.  
     There is an algorithm such that for every distribution $D$ over $\R^d \times \{0,\pm 1\}$ 
     with $D_\x$ logconcave and $W(D) \geq \alpha$, given $\eps\in (0,1)$,
    it draws $(d\ell)^{O(\ell)}\poly_{K,\eta}(1/(\alpha\eps))$ samples from $D$,  
    runs in $(d\ell)^{O(\ell)}2^{\poly_{K,\eta}(\log(1/(\alpha\eps))}$ time, where $\ell = \exp((K/\eta)^{O(K)})$, 
    and outputs a hypothesis $\hat{h}$ such that with probability at least $99\%$, 
    $\err_{D_f}(\hat{h})\le \exp(O(\ell \log \ell)) \opt_{D_f}(\H) + \eps \;.$ 
\end{theorem}

By Theorem~\ref{thm: main-weak-to-strong}, to prove Theorem~\ref{th generic family}, 
it suffices to design an efficient weak learner 
for any $(K,\eta)$-good class of MIMs $\H$.
To achieve this, we establish the following technical lemma, 
that gives a basic condition  under which 
a general hypothesis class $\H$ can be efficiently weakly learned.



\begin{lemma}\label{lm natural condition}
    Let $\H$ be a hypothesis class over $\R^d$  
    and $D'_\x$ be an arbitrary $1/2$-isotropic logconcave distribution over $\R^d$. 
    Suppose that for every $h \in \H$, the following statement is satisfied:  
    there is a polynomial $p: \R^d \to \R$ with zero mean under $D'_\x$ 
    such that either  
    (i) if $h(\x) = 1$, then $p(\x)>\gamma$, or 
    (ii) if $h(\x) = -1$, then $p(\x)<-\gamma$. 
    If each polynomial $p$ has the form 
    $p(\x) = \sum_i c_i \prod_{j=1}^{\ell_i} (\w^{(ij)}\cdot\x)$ 
    with $\ell_{i} \le \ell$, $\w^{(ij)} \in \s^{d-1}$, and 
    $ \max\left\{ \sum_{i} \abs{c_i}, \sum_{i} \abs{c_i} \norm{\prod_{j=1}^{\ell_i} (\w^{(ij)}\cdot\x)}_{D_\x',2} \right\} \le B$,     
    then there is an algorithm with the following guarantee. 
    For any distribution $D$ over $\R^d \times \{\pm 1\}$ 
    such that for every $\w \in \s^{d-1}$ and $i \in [2\ell]$, 
    $\abs{\E_{\x \sim D_\x}[(\w\cdot\x)^i]-\E_{\x \sim D'_\x}[(\w\cdot\x)^i]} \le \Delta \le \gamma B^{-1}2^{-O(\ell)}$, 
    with $\opt_D(\H)< O((\gamma/B)^2)$, 
    the algorithm draws $\poly(d^\ell,B/\gamma)$ examples from $D$, 
    runs in time $\poly(d^\ell,B/\gamma)$, and outputs a hypothesis 
    $\hat{h}$ such that $\err_D(\hat{h})< 1/2-\Omega((\gamma/B)^2)$ with probability at least $99\%$.
\end{lemma}

Roughly speaking, the condition in Lemma~\ref{lm natural condition} 
assumes that each function $h \in \H$ 
has a one-sided approximation by a mean-zero low-degree polynomial $p$, 
which can be written as a sum of products of linear functions 
with small $L_2$-norms under a logconcave distribution $D'_\x$.
Specifically, this means that $\E_{\x \sim D'_\x}[p(\x)] = 0$ 
and whenever $h(\x) = 1$, it must be the case that $p(\x)>\gamma>0$.
The structure of $p(\x)$ implies that if $\x$ is drawn from $D_\x$, 
a distribution that approximately matches moments with $D'_\x$, 
the mean and the $L_2$ norm of $p(\x)$ will not change too much.
This implies if we take $h(\x)$ to be the optimal hypothesis, 
then, under $D_\x$, 
$h(\x)$ and $p(\x)$ must have non-trivial correlation---because
$\E_{\x\sim D_\x}[h(\x)p(\x)] \approx \E_{\x\sim D_\x}[(h(\x)+1)p(\x)]$.
By Lemma~\ref{lm robust chow parameter}, this implies that 
there is some low-degree polynomial that has non-trivial correlation with the label $y$, 
if $\opt_D(\H)$ is sufficiently small. Therefore, 
we can run a simple polynomial regression subroutine 
to find such a good polynomial and turn it into a weak hypothesis.

\begin{proof}[Proof of Lemma~\ref{lm natural condition}]
    Let $h^* \in \H$ be the hypothesis, such that $\err_D(h^*) = \opt$ and let $p(\x)$ be a degree $\ell$ polynomial in the statement for $h^*$.
    Without loss of generality, we assume that when $h^*(\x) = 1$, $p(\x)>\gamma$.
    Since $\E_{\x \sim D'_\x}[p(\x)] = 0$,
    we know that under $D_\x$, we have 
    \begin{align*}
        \abs{\E_{\x \sim D_\x}[p(\x)]} & = \abs{\sum_i c_i\E_{\x \sim D_\x}[\prod_{j=1}^{\ell_i} (\w^{(ij)}\cdot\x)]-\E_{\x \sim D'_\x}[\prod_{j=1}^{\ell_i} (\w^{(ij)}\cdot\x)]} \\
        & \le \sum_{i}\abs{c_i}\abs{\E_{\x \sim D_\x}[\prod_{j=1}^{\ell_i} (\w^{(ij)}\cdot\x)]-\E_{\x \sim D'_\x}[\prod_{j=1}^{\ell_i} (\w^{(ij)}\cdot\x)]} \\
        & \le (2e)^\ell\sum_{i}\abs{c_i} \Delta  \le O(\gamma).
    \end{align*}
Here, we use the fact that when  $\abs{\E_{\x \sim D_\x}[(\w\cdot\x)^i]-\E_{\x \sim D'_\x}[(\w\cdot\x)^i]} \le \Delta $ for every $\w \in \s^{d-1}$ and $i \in [2\ell]$,
$\abs{\E_{\x \sim D_\x}[\prod_{j=1}^{\ell_i} (\w^{(ij)}\cdot\x)]-\E_{\x \sim D'_\x}[\prod_{j=1}^{\ell_i} (\w^{(ij)}\cdot\x)]} \le (2e)^\ell \Delta$, when $\sum \ell_{ij} \le \ell$.

This implies that 
\begin{align*}
    \E_{\x \sim D_\x}[h^*(\x)p(\x)] & = \E_{\x \sim D_\x}[(h^*(\x)+1)p(\x)] - \E_{\x \sim D_\x}[p(\x)] \\
    & \ge \E_{\x \sim D_\x}[(h^*(\x)+1)p(\x)] - \gamma/20.
\end{align*}
Notice that if $\Pr_{\x \sim D_\x}[h^*(\x) = 1]<1/3$, then a constant polynomial $p^-(\x)\equiv-1$ satisfies the statement. 
So, we can assume that $\Pr_{\x \sim D_\x}[h^*(\x) = 1]>1/3$. This implies that 
\begin{align*}
    \E_{\x \sim D_\x}[h^*(\x)p(\x)] \ge \gamma/3 - \gamma/20 \ge \gamma/6.
\end{align*}
On the other hand, we have 
\begin{align*}
    \norm{p}_{D_\x,2}\le \sum_i \abs{c_i} \norm{\prod_{j=1}^{\ell_i} (\w^{(ij)}\cdot\x)}_{D_\x,2} \le O\left(\sum_i \abs{c_i} \norm{\prod_{j=1}^{\ell_i} (\w^{(ij)}\cdot\x)}_{D'_\x,2}\right) = O(B).
\end{align*}
By Lemma~\ref{lm robust chow parameter}, it follows 
that when $\opt<O((\gamma/B)^2)$, we have $\E_{(\x,y)\sim D}[yp(\x)]>\gamma/8$. 

Using polynomial regression \cite{KKMS:08}, we can draw $\poly(d^\ell/\gamma)$ examples from $D$, run in $\poly(d^\ell/\gamma)$ time, 
and compute a polynomial $\hat{p}$ such that $\E_{(\x,y) \sim D}[y\hat{p}(\x)]>\Omega(\gamma)$ and $\E_{\x \sim D_\x}[\hat{p}^2(\x)] \le O(B)$.
    By Lemma~\ref{lm robust chow parameter} and Fact~\ref{fact correlation}, we can round the computed polynomial $\hat p(\x)$ into a Boolean function $\hat{h}(\x)$ such that $\err_D(\hat{h}) \le 1/2 - \Omega((\gamma^2/B^2))$.
This completes the proof of Lemma~\ref{lm natural condition}.
\end{proof}

To prove Theorem~\ref{th generic family}, it suffices to show that any $(K,\eta)$-good class of MIMs $\H$ satisfies the assumption given by Lemma~\ref{lm natural condition}, which automatically gives a weak learning algorithm. Formally, we give Lemma~\ref{lm good class} below.

\begin{lemma}\label{lm good class}
    Let $\H$ be a $(K,\eta)$-good MIM family over $\R^d$. There is an algorithm $\A$ that draws $(d\ell)^{O(\ell)}$ examples, runs in $(d\ell)^{O(\ell)}$ time and $(2\ell,\exp(-\ell\log \ell),\exp(-O(\ell \log \ell)),\exp(O(\ell \log \ell)))$-weakly learns $\H$ with probability at least $99\%$, 
    where $\ell = \exp(O(K^8/\alpha^4))$ and $\alpha = \eta^K2^{-O(K\log K)}$.
\end{lemma}

Before we proceed with the proof of the above lemma, we note that 
Theorem~\ref{th generic family} is a direct corollary of Theorem~\ref{thm: main-weak-to-strong} 
and Lemma~\ref{lm good class}.

We give the intuition behind Lemma~\ref{lm good class}.
Since the value of a function $h \in \H$ at a point $\x$ only depends on the projection of $\x$ onto some $K$-dimensional subspace $V$, for a given logconcave distribution $D'_\x$, we will instead consider the distribution $\x'\sim D^V_\x$, the projection of $\x$ onto $V$, which is still logconcave.
By the anti-concentration property of logconcave distributions, 
points within a ball centered at the mean of $D^V_\x$ have density at least $2^{-O(K)}$.
By the goodness of $\H$, after a proper affine transformation, there is a ball of volume $\eta^{O(K)}$ in this region with negative sign and thus contains a box $S$ within the ball with $(\eta/K)^{O(K)}$ volume.
This implies that the box $S$ has non-trivial probability mass under $D^V_\x$.
We will construct a polynomial $p(\x')$ that always upper bounds the indicator $\Ind\{\x' \not\in S\}$ and is close to it in $L_1$ norm under $D^V_\x$.
This implies that the polynomial $p'(\x') = p(\x') - \E_{\x'\sim D^V_\x}[p(\x')]$ is slightly larger than $0$, whenever $\x'$ is positive.
To get such a polynomial, we will make use of the following result from~\cite{kane2013learning} 
which gives a low-degree construction of a pair of polynomials that ``sandwich'' any function of halfspaces. 

\begin{theorem}\label{th logconcave approximation}[Theorem 15 in \cite{kane2013learning}]
    Let $D_\x$ be any logconcave distribution in $\R^d$, and let $f(\x) = g(h_1(\x),\dots,h_m(\x)): \R^d \to \{0,1\}$ be a function of $m$ halfspaces, where $h_i(\x)=\sign(w^{(i)}\cdot\x+t_i)$. For every $\eps>0$, there is a pair of polynomials $p^u,p^l$ with degree $\exp(O(m^2/\eps)^4)$, such that for every $\x \in \R^d$, $p^l(\x) \le f(\x) \le p^u(\x)$, and $\E_{\x\sim D_\x}[\abs{p^u(\x)-p^l(\x)}] \le \eps.$ 
\end{theorem}

Finally, since the constructed polynomial $p'(\x')$ is a low-degree polynomial 
over a $K$-dimensional subspace $V$, we can expand each monomial of $p'$ 
in a form of $\prod_i(\w^{(i)}\cdot\x)$,
where $w^{(i)}$ is a standard basis vector of $V$.
This implies that we can control the size of $p'$ 
as a function that only depends on $K$.
Therefore, $\H$ satisfies the condition of Lemma~\ref{lm natural condition}, 
and we obtain a weak learning algorithm for $\H$.
We give the proof of Lemma~\ref{lm good class} below.


\begin{proof}[Proof of Lemma~\ref{lm good class}]
Let $L:\R^d \to \R^d$, be any invertible affine transformation. By the definition of a good MIM, we know that if $\H$ is $(K,\eta)$-good, then $\H_L:=\{h(L\x)  \mid h \in \H \}$ is also $(K,\eta)$-good. 
It suffices to give an algorithm such that for every input distribution $D$, if there is an $1/2$-isotropic logconcave distribution that has low-degree moments approximately matching with $D_\x$ and $\opt_D(\H)$ is small, then we can get a weak hypothesis over $D$.


    Let $D'_\x$ be any $1/2$-isotropic logconcave distribution. 
    Let $h(\x)\in \H$ such that $h(\x) = h(\proj_V(\x))$, where $V$ is a $K$-dimensional subspace. Write $h(\x) = g(\mathbf{W} \x)$, for $g \in \mathcal{G}$ and $\mathbf{W} \in \R^{K \times d}$ such that $\mathbf{W}\mathbf{W}^\top = \mathbf{I}_{K\times K}$.
    
    Denote by $D^{(V)}_\x$ the distribution of $\x' = \mathbf{W} \x$ for $\x \sim D'_\x$ and by $\bm \mu = \E_{\x' \sim D^{(V)}_\x}[\x']$. Notice that $D^{(V)}_\x$ is a $1/2$-isotropic logconcave distribution over $\R^K$.
    Without loss of generality, assume that $\bm \mu = 0$.

Let $B=\B^{K-1}(0,1/18)$.
    Since $\H$ is $(K,\eta)$-good, we know that there is some $B'=\B^{K-1}(\v,\eta/9)$ such that $B' \subseteq B$ and for every $\x' \in B'$, 
    we have $g(\x') \equiv 1$ or $g(\x') \equiv -1$.
    Without loss of generality, we assume $g(\x') \equiv -1$. Let $S \subseteq B'$ be the largest box contained in $B'$.
    We will next lower bound $\Pr_{\x' \sim D^{(V)}_\x}[\x' \in S]$.
     
    Notice that the largest box contained in $\B^{K-1}(0,1/18)$ has volume $2^{-O(K\log K)}$, which implies that $\vol(S) \ge \eta^K 2^{-O(K\log K)}$.
    By Fact~\ref{fact density}, we know that for every $\x' \in S$, the density  $D^V_\x(\x') \ge 2^{-8K}$.
    This implies that 
    \begin{align*}
        \Pr_{\x' \sim D^{(V)}_\x}[\x' \in S] = \int_{S}D^{(V)}_\x(\x') d\x' \ge 2^{-8K} \vol(S) \ge \eta^K 2^{-O(K\log K)}.
    \end{align*}
    Notice that the box $S$ is an intersection of $2K$ halfspaces. 
    By Theorem~\ref{th logconcave approximation}, there is a polynomial $p(\x')$ with degree $\ell = \exp(O(K^8/\alpha^4))$ such that
    \begin{align*}
        p(\x') \le \Ind\{\x' \in S\}, \  \E_{\x' \sim D^{(V)}_\x}[\abs{\Ind\{\x' \in S\} - p(\x')}] \le \alpha.
    \end{align*}
Since $\Pr_{\x' \sim D^{(V)}_\x}[\x' \in S] \ge \eta^K 2^{-O(K\log K)}$, by setting $\alpha = \Pr_{\x' \sim D^{(V)}_\x}[\x' \in S]/2$, the polynomial $p'(\x') = \E_{\x' \sim D^{(V)}_\x}[p(\x')]-p(\x') $ satisfies $p'(\x')> \alpha$ for every $x' \not \in S$. In particular, since $g(\x') = -1$ for every $\x' \in S$, we know that $p'(\x') >\alpha $, when $g(\x') = 1$. 

In particular,  we have $p'(\mathbf{W}\x) > \alpha$ whenever $h(\x) = 1$ and $\E_{\x \sim D'_\x}[p'(\mathbf{W}\x)] = 0$. 
Let $\w^{(1)},\dots,\w^{(K)}$ be the row vectors of  $\mathbf{W}$. 
Then we have 
\begin{align*}
    p'(\mathbf{W}\x) = \sum_i c_i \prod_{j=1}^{\ell_i}(\w^{(j)}\cdot\x) \;,
\end{align*}
with $ \ell_{i} \le \ell$, for all $i$.
Since $\E_{\x' \sim D^{(V)}_\x}[\abs{p'(\x')}]\le 2\alpha$,
it follows from Fact~\ref{fact log concave l1 growth} that 
\begin{align*}
    \sum_{i} \abs{c_i} \norm{\prod_{j=1}^{\ell_i}(\w^{(j)}\cdot\x)}_{D'_\x,2}\le \exp(O(\ell \log \ell))  \E_{\x' \sim D^{(V)}_\x}[\abs{p'(\x')}] = \exp(O(\ell \log \ell)).
\end{align*}

Let $B=\exp(O(\ell \log \ell))$. 
By Lemma~\ref{lm natural condition}, we have 
that when the input distribution $D$ has its marginal distribution $D_\x$ approximately matching the moments of $D'_\x$, 
we can weakly learn $\H$ in time $(d\ell)^{O(\ell)}$, 
when $\abs{\E_{\x \sim D_\x}[(\w\cdot\x)^i]-\E_{\x \sim D'_\x}[(\w\cdot\x)^i]} \le \Delta \le \alpha B^{-1}2^{-O(\ell)} = \exp(-O(\ell \log \ell))$ and $\opt \le (\alpha/B)^2 = \exp(-O(\ell \log \ell))$. This completes our proof. 
\end{proof}

    

\subsection{Weak Agnostic Learning via Sandwiching Polynomials}\label{sec sandwhiching polynomial}

Although Theorem~\ref{th generic family} gives agnostic partial learning algorithms for a broad class of MIMs under logconcave distributions,
the parameters of the $(k,\Delta,\beta,C)$-weak learner 
that we derive from the geometry of $\H$ are not good in general.
As a result, the final learning algorithm only gives a reasonable hypothesis when $\opt$ is very small.

In this section, we show that if a hypothesis class $\H$ admits 
low-degree $L_1$ 
sandwiching polynomial approximation,
then we are able to perform agnostic partial learning 
for such a hypothesis class
to error $(1+c)\opt$, for any desired constant $c>0$, 
with much better complexity. 

We formulate the class of functions we are interested in as follows.
\begin{definition}[MIM Classes with Low-Degree $L_1$-Sandwiching Polynomials] \label{def:class-sandwich}
    Let $\H$ be a hypothesis class over $\R^d$ and let $\ell: (0,1] \to \Z_+$ be an integer-valued function.
    We say that $\H$ is a $K$-MIM that admits degree-$\ell$ $L_1$ sandwiching polynomials, if the following conditions hold.
    \begin{enumerate}[leftmargin=*]
        \item
        $\H$ is a family of $K$-Multi-index models ($K$-MIMs). That is, 
        for every $h \in \H$ there is a $K$-dimensional subspace $V \subset \R^d$  
        such that $h(\x) = h(\proj_V(\x))$. Equivalently, there is a class of Boolean functions $\mathcal{G}$ over $\R^{K}$ 
        such that for every $h \in \H$ there exists $g \in \mathcal{G}$ and a matrix $\mathbf{W} \in \R^{K \times d}$ 
        with $\mathbf{W}\mathbf{W}^\top = \mathbf{I}_{K \times K}$  
        so that $h(\x) = g(\mathbf{W} \x)$.
        \item  For any $c \in (0,1)$, for every $g \in \mathcal{G}$, for every logconcave distribution $D_\x$ over $\R^K$, there is a pair of degree $\ell(c)$-polynomials $p^l,p^u$ such that for every $\x' \in \R^K$
    \begin{align*}
        p^l(\x') \le g(\x') \le p^u(\x'), \  \E_{\x' \sim D_\x}[\abs{p^u(\x')- p^l(\x')}] \le c. 
    \end{align*}
    \end{enumerate}    
\end{definition}


Our main result of this subsection is the following theorem.

\begin{theorem}[$L_1$-sandwiching Polynomial Approximation Implies QPTAS]\label{th sandwhiching}
Let $\H$ be a hypothesis class over $\R^d$ and let $\ell: (0,1] \to \Z_+$ be an integer-valued function.
    If $\H$ is a family of $K$-MIMs that admits degree-$\ell$ $L_1$ sandwiching polynomials, then
    there is an algorithm such that for every input distribution $D$ over $\R^d \times \{0,\pm 1\}$ of weight at least $\alpha$, with logconcave marginal, for every $c,\eps \in (0,1)$, it 
    draws $(d\ell)^{O(\ell)}\poly_{K,c}(1/(\alpha\eps))$ examples from $D$, runs in $(d\ell)^{O(\ell)}2^{\poly_{K,c}(\log(1/(\alpha\eps))}$ time, 
    where $\ell=\ell(c/6)$, 
    and
    outputs a hypothesis $\hat{h}$ such that with probability at least $99\%$,
    $\err_{D_f}(\hat{h})\le (1+c) \opt_{D_f}(\H) + \eps$.
\end{theorem}


By Theorem~\ref{thm: main-weak-to-strong}, to prove Theorem~\ref{th sandwhiching},
the central technical lemma we need is that 
any MIM that admits low-degree sandwiching polynomials is weakly learnable.


\begin{lemma}\label{lm weak learning L1}
    Let $\H$ be a hypothesis class over $\R^d$ and let $\ell: (0,1] \to \Z_+$ be an integer-valued function.
    If $\H$ is a family of $K$-MIMs that admits degree-$\ell$ $L_1$ 
    sandwiching polynomials, then
    there is an algorithm such that for every $c \in (0,1)$ and for every input distribution $D$ over $\R^d \times \{\pm 1\}$ such that there is some $1/2$-isotropic logconcave distribution $D'_\x$ over $\R^d$ that satisfies for every unit vector $\w$ and $i \in [2\ell]$, $\abs{\E_{\x \sim D_\x}[(\w\cdot\x)^i]-\E_{\x \sim D'_\x}[(\w\cdot\x)^i]}<\Delta<\exp(-O(K\ell \log \ell))$ with $\opt_D(\H) \le (1-c)/2$, where $\ell = \ell(c/3)$, it 
    draws $\poly((\ell d)^\ell/c^2)$ examples from $D$, runs in $\poly((\ell d)^\ell/c^2)$ time
    and
    outputs a hypothesis $\hat{h}$ with error $\err_D(\hat{h})<1/2- c^2\exp(-O(\ell \log \ell))$ with probability at least $99\%$.
\end{lemma}

We sketch the pseudo-code of the weak learning algorithm 
in Algorithm~\ref{alg weak-learning good class}.

\begin{algorithm}[htbp]
		\caption{\textsc{PolynomialWeakLearning} (Weak learning using a polynomial)}\label{alg weak-learning good class}
		\begin{algorithmic} [1]
\State \textbf{Input:}  Sample access to $D$, diameter parameter $B>0$, advantage parameter $\beta$, degree of polynomial $\ell$.
\State \textbf{Output:} $\hat{h}:$ weak hypothesis over $D$.

\medskip

\State $h_+(\x):= 1$, $h_-(\x):= -1$, $\mathcal{C}:=\{h_+,h_-\}$
\State Draw $S\sim D^m$, where
$m=\widetilde{O}(d^{O(\ell)}B^2/\beta^4)$.
\State Find a degree-$\ell$ polynomial $p(\x)$ maximizing
$
\sum_{(\x,y)\in S} yp(\x)
$
subject to
$
\frac{1}{m}\sum_{(\x,y)\in S}p(\x)^2
\le O(B),
$
using convex programming.
\State $g(\x)=\proj_{[-1,1]}
\left(\frac{\beta p(\x)}{CB}\right)$.
\For{$i \in [O(B^2/\beta^2)]$}
\State Sample $t'$ uniformly from $ [-1,1]$ and define $h_i(\x):=\sign(g(\x)-t')$
\State $\mathcal{C} \gets \mathcal{C} \cup \{h_i(\x)\}$
\EndFor
\State Draw $S'\sim D^{m'}$, where
$m'=\widetilde{O}(B^2/\beta^4)$.
\State \Return $\hat{h}:= \argmin_{h \in \mathcal{C}}\sum_{(\x,y)\in S'}\Ind\{h(\x)\neq y\}$.

\end{algorithmic}
\end{algorithm}

Intuitively, if $h^*$ is the optimal hypothesis in $\H$,
then $\E_{(\x,y)\sim D}[h^*(\x)y]>c$. 
Then, $\E_{(\x,y)\sim D}[p^u(\x)y]$ is at most $\E_{\x\sim D_\x}[p^u(\x)-p^l(\x)]$ far from $\E_{(\x,y)\sim D}[h^*(\x)y]$.
Since the sandwiching polynomials only depend on the projection of $\x$ onto some $k$-dimensional subspace $V$,
we can upper bound the sizes of these polynomials as a function of $k$.
As a result, as long as $D_\x$ approximately matches its low-degree moment with some logconcave distribution $D'_\x$, $\E_{\x\sim D_\x}[p^u(\x)-p^l(\x)]$ is close to $\E_{\x\sim D'_\x}[p^u(\x)-p^l(\x)]$. 
This implies that $\E_{(\x,y)\sim D}[p^u(\x)y]$ is at least $\Omega(c)$, which gives us a weak hypothesis, whenever $\opt<(1-c)/2$.

\begin{proof}[Proof of Lemma~\ref{lm weak learning L1}]
Let $D'_\x$ be a $1/2$-isotropic logconcave distribution.
Let $h^*(\x)$ be the optimal hypothesis in $\H$
such that $\err_D(h^*(\x)) = \opt = \opt_D(\H)$. 
Let $V$ be a $K$ dimensional subspace such that $h^*(\x) = h(\proj_V(\x))$. Write $h^*(\x) = g(\mathbf{W}\x)$, with $\mathbf{W}\mathbf{W}^\top = \mathbf{I}_{K\times K}$. We denote by $D^{(V)}_\x$ the distribution of $\x' = \mathbf{W}\x, \x \sim D'_\x$

We notice that $ \E_{(\x,y) \sim D}[g(\x')y] = \E_{(\x,y)\sim D}[h^*(\x)y] = 1-2\opt>c$.
Let $p^u(\x'), p^l(\x')$ be upper and lower polynomials for $g(\x')$
\begin{align*}
    \E_{\x' \sim D^{(V)}_\x}[ \abs{p^u(\x')-p^l(\x')}] \le c/3.
\end{align*}
We write 
$p^u(\x') = \sum_{\card{\alpha}\le \ell} a_\alpha (\x')^\alpha$ and $p^l(\x') = \sum_{\card{\alpha}\le \ell} b_\alpha (\x')^\alpha$.
Since $\E_{\x' \sim D^{(V)}_\x}[ \abs{p^u(\x')}] \le 2$ and $\E_{\x' \sim D^{(V)}_\x}[ \abs{p^l(\x')}] \le 2$, by Fact~\ref{fact log concave l1 growth}, we have 
\begin{align*}
    \sum_{\abs{\alpha} 
\le k} 
\left( \abs{a_\alpha}^2\E_{\x \sim D_\x}
[(\x')^{2\alpha}]\right)^{1/2}
\le
\exp(O(k\ell \log \ell)), 
\  \sum_{\abs{\alpha} 
\le k} 
\left( \abs{b_\alpha}^2\E_{\x \sim D_\x}
[(\x')^{2\alpha}]\right)^{1/2}
\le
\exp(O(K\ell \log \ell)).
\end{align*}

Since $\abs{\E_{\x \sim D_\x}[(\w\cdot\x)^i]-\E_{\x \sim D'_\x}[(\w\cdot\x)^i]}<\Delta$, for $\w \in \s^{d-1}$ and $i \in [2\ell]$, 
we have
\begin{align*}
    \abs{\E_{\x \sim D_\x}[p^u(\x')] - \E_{\x \sim D'_\x}[p^u(\x')]} 
    & = \abs{\sum_\alpha a_\alpha(\E_{\x \sim D_\x}[(\x')^\alpha]-\E_{\x \sim D'_\x}[(\x')^\alpha])} \\
        & \le \sum_{\alpha}\abs{a_\alpha}\abs{(\E_{\x \sim D_\x}[(\x')^\alpha]-\E_{\x \sim D'_\x}[(\x')^\alpha])} \\
        & \le (2e)^\ell\sum_{\alpha}\abs{a_\alpha} \Delta  \le c/6 \;.
\end{align*}
Similarly, we have 
\begin{align*}
    \abs{\E_{\x \sim D_\x}[p^l(\x')] - \E_{\x \sim D'_\x}[p^l(\x')]} 
    & = \abs{\sum_\alpha b_\alpha(\E_{\x \sim D_\x}[(\x')^\alpha]-\E_{\x \sim D'_\x}[(\x')^\alpha])} \\
        & \le \sum_{\alpha}\abs{b_\alpha}\abs{(\E_{\x \sim D_\x}[(\x')^\alpha]-\E_{\x \sim D'_\x}[(\x')^\alpha])} \\
        & \le (2e)^\ell\sum_{\alpha}\abs{b_\alpha} \Delta  \le c/6 \;.
\end{align*}
This implies that 
\begin{align*}
    \abs{\E_{(\x,y)\sim D}[h^*(\x)y] - \E_{(\x,y)\sim D}[p^u(\mathbf{W}\x)y]} &\le \E_{(\x,y)\sim D} [p^u(\mathbf{W}\x) -p^l(\mathbf{W}\x)] \\ 
    &\le 
    \E_{\x' \sim D^{(V)}_\x}[|p^u(\x')-p^l(\x')|] +c/3 \le 2c/3 \;,
\end{align*}
which gives us  $\E_{(\x,y)\sim D}[yp^u(\mathbf{W}\x)]>c/3.$ 
By Fact~\ref{fact polynomial growth}, we have $\E_{\x \sim D'_\x}[(p^u(\x))^2] \le \exp(\ell \log \ell)$.
 Using polynomial regression \cite{KKMS:08}, we can draw $\poly((\ell d)^\ell/c^2)$ examples from $D$ and run in $\poly((\ell d)^\ell)$ time to find a polynomial $\hat{p}$ such that $\E_{(\x,y) \sim D}[y\hat{p}(\x)]>c/6$ and $\E_{\x \sim D_\x}[\hat{p}^2(\x)] \le \exp(O(\ell \log \ell))$.
    By Lemma~\ref{lm robust chow parameter} and Fact~\ref{fact correlation}, we can round the computed polynomial $\hat p(\x)$ into a Boolean function $\hat{h}(\x)$ such that $\err_D(\hat{h}) \le 1/2 - c^2\exp(-O(\ell \log \ell))$.
\end{proof}

With Lemma~\ref{lm weak learning L1}, we can immediately obtain the proof of Theorem~\ref{th sandwhiching}. 

\begin{proof}[Proof of Theorem~\ref{th sandwhiching}]
Notice that for every invertible affine transformation $L:\R^d \to \R^d$, 
the class $\H_L:=\{h(L\x)\mid h \in \H\}$ is also a family of $K$-MIMs 
that admits degree-$\ell$ $L_1$ sandwiching polynomials. 
This implies that in order to show that the class $\H$ is 
$(\ell, \exp(-O(K\ell \log \ell)), c\exp(-O(\ell \log \ell)),C)$-weakly learnable, 
we only need to show that given distribution $D$ over $\R^d \times \{\pm 1\}$, 
such that $D_\x$ is close to some $1/2$-isotropic logconcave distribution $D'_\x$,
and $\opt_D(\H)$ is small, we can produce a weak hypothesis.
 
    For every $h \in \H$, such that $h(\x) = g(\mathbf{W}\x)$, with $g\in \mathcal{G}$,
    let $\x' = \mathbf{W}\x$.
    Then, we have $h(\x) = g(\x').$
    By definition of $\ell$, for every logconcave distribution $D'_\x$ over $\R^k$, 
    and for every $a \in (0,1)$, there is a pair of polynomials $p^u$ and $p^l$ 
    with degree $\ell(a)$ such that 
    \begin{align*}
        p^l(\x') \le g(\x')  \le p^u(\x'), \  \E_{\x \sim D'_\x}[|p^u(\x')- p^l(\x')|] \le a \;. 
    \end{align*}    
For every $c \in (0,1)$, by choosing $C=2(1+c)$, we have $C/2 = (1+c)$.
By Lemma~\ref{lm weak learning L1},
there is an algorithm $\A$ such that
given a distribution $D$ over $\R^d \times \{\pm 1\}$, with $\opt_D(\H)<1/C = (1-c/(1+c))/2$, it draws $\poly((\ell d)^\ell/c^2)$ examples from $D$, runs in $\poly((\ell d)^\ell/c^2)$ time
and outputs a hypothesis $\hat{h}$ with error $\err_D(\hat{h})<1/2- c^2\exp(-O(\ell \log \ell))$ with probability at least $99\%$, whenever 
there is some $1/2$-isotropic logconcave distribution $D'_\x$ over $\R^d$ that satisfies $\abs{\E_{\x \sim D_\x}[(\w\cdot\x)^i]-\E_{\x \sim D'_\x}[(\w\cdot\x)^i]}<\exp(-O(K\ell \log \ell))$. Here $\ell = \ell(c/(3(1+c))) \le \ell(c/6)$.

This means the class $\H$ is $(2\ell, \exp(-O(K\ell \log \ell)), c^2\exp(-O(\ell \log \ell)),C)$-weakly learnable.
By Theorem~\ref{thm: main-weak-to-strong},
There is an algorithm such that for every distribution $D$ over $\R^d \times \{0,\pm 1\}$, such that $D_\x$ is a logconcave distribution, and $W(D) \ge \alpha$, and
    for every $\eps,c\in (0,1)$,
    it draws $(d\ell)^{O(\ell)}\poly_{K,c}(1/(\alpha\eps))$ samples, 
    runs in $(d\ell)^{O(\ell)}2^{\poly_{K,c}(\log(1/(\alpha\eps))}$ time, where $\ell = \ell(c/6)$, 
    and outputs a hypothesis $\hat{h}$ such that with probability at least $99\%$, $\err_{D_f}(\hat{h})\le (1+c) \opt_{D_f}(\H) + \eps$.
\end{proof}

\subsection{Quasi-PTAS for Agnostic Partial Learning of Multi-index Models} \label{ssec:qptas-apps}

As an important application of Theorem~\ref{th sandwhiching}, we obtain 
a quasi-polynomial time approximation scheme (QPTAS) for a broad class of MIMs. 
The main tool we use is the polynomial approximation theory developed in \cite{kane2013learning}, which gives low-degree $L_1$ sandwiching polynomials for a broader class of MIMs.
We state this theorem as follows.
\begin{theorem}[Extension of Theorem 6 in \cite{kane2013learning}]\label{th smooth approximation}
    Let $D_\x$ be an isotropic logconcave distribution over $\R^d$ and let $f: \R^d \to \{0,1\}$. For every $\eps,\delta>0$, there is a pair of polynomials $p^l,p^u$ of degree $\ell = \exp(O(d^4/\delta^4+d^2\log^2(d/\eps)))$, such that 
    \begin{align*}
        p^l(\x) \le f(\x)  \le p^u(\x), \  \E_{\x \sim D_\x}[\abs{p^u(\x)- p^l(\x)}] \le \eps+ 2 \Pr_{\x\sim D_\x}[S_\delta], 
    \end{align*}
    where $S_\delta=\{\u \in \R^d \mid \exists \v \in \R^d: \norm{\u-\v}_2\le \delta, f(\u) \neq f(\v)\}$.
\end{theorem}
We remark that the original statement in \cite{kane2013learning}, only claims a slightly weaker guarantee---namely, that there is a degree-$\ell$ polynomial $p$ such that $\norm{p-f}_{D_\x,1}$ is bounded. 
However, obtaining $L_1$-sandwiching polynomials with qualitatively the same degree 
is implicitly implied from the proof. 
We refer to Appendix~\ref{app application} for the details.

Given Theorem~\ref{th smooth approximation}, the QPTAS can be obtained immediately as a corollary of Theorem~\ref{th sandwhiching}.
For example, by Theorem~\ref{th logconcave approximation}, the class of functions of $K$ halfspaces,
$$\H_{K,d}=\left\{ h(\x)=g(h_1(\x),\dots,h_K(\x))\mid g: \{\pm 1\}^K \to \{\pm1\}, h_i(\x) = \sign(w^{(i)}\cdot \x +t_i), i \in [K] \right\}$$ 
admits $\exp(O(K^8/c^4))$-degree $L_1$-sandwiching polynomials, which immediately implies Theorem~\ref{th functions of halfspaces}.

\begin{theorem}[QPTAS for Agnostically Learning Functions of Halfspaces] 
\label{th functions of halfspaces}
    There is an algorithm that for any distribution $D$ over $\R^d \times \{0,\pm 1\}$ 
    with $D_\x$ logconcave and $W(D) \ge \alpha$ has the following performance guarantee.  
    Given $K \in \Z_+$ and $\eps,c\in (0,1)$, the algorithm draws 
    $(d\ell)^{O(\ell)}\poly_{K,c}(1/(\alpha\eps))$ samples from $D$, it  
    runs in $(d\ell)^{O(\ell)}2^{\poly_{K,c}(\log(1/(\alpha\eps))}$ time, 
    where $\ell = \exp(O(K^2/c)^4))$, 
    and outputs a hypothesis $\hat{h}$ such that with probability at least $99\%$, 
    $\err_{D_f}(\hat{h})\le (1+c) \opt_{D_f}(\H_{K}) + \eps$.
\end{theorem}

A similar bound can also be obtained for the class of low-dimensional and 
low-degree Polynomial Threshold Functions (PTFs), defined as 
$$P_{K,m} = \left\{h(\x) = \sign(p(\mathbf{W}\x)) \mid \mathbf{W} \in \R^{K\times d}, \mathbf{W}\mathbf{W}^\top = \mathbf{I}_{K\times K}, p: \R^K \to \R, \deg(p) \le m \right\}.$$  

\begin{theorem}[QPTAS for Agnostically Learning Low-Dimensional PTFs] 
\label{thm:low-dim-PTFs}
    There is an algorithm that for any distribution $D$ over $\R^d \times \{0,\pm 1\}$ 
    with $D_\x$ logconcave and $W(D) \ge \alpha$ has the following performance guarantee.  
    Given $K,m \in \Z_+$ and $\eps,c\in (0,1)$, the algorithm draws 
    $(d\ell)^{O(\ell)}\poly_{K,m,c}(1/(\alpha\eps))$ samples from $D$, it  
    runs in $(d\ell)^{O(\ell)}2^{\poly_{K,m,c}(\log(1/(\alpha\eps))}$ time, 
    where $\ell = \exp(O(m^K2^{\tilde O(K)}\log^K(1/c)/c)^4))$, 
    and outputs a hypothesis $\hat{h}$ such that with probability at least $99\%$, 
    $\err_{D_f}(\hat{h})\le (1+c) \opt_{D_f}(P_{K,m}) + \eps$.
\end{theorem}
Theorem~\ref{thm:low-dim-PTFs} can be seen as a direct corollary of Theorem~\ref{th smooth approximation} and the result by \cite{lotz2015volume}, which gives an explicit bound on the volume of examples in $\R^d$ that are $\delta$-close to the zero set of a low-degree polynomial. We refer to Appendix~\ref{app application} for details.


\subsection{Agnostically Learning Partial Linear Classifiers}\label{sec partial linear}

In this section, we study the problem of agnostic partial learning halfspaces 
under logconcave distributions. Let $\H=\{h_{\w,t}(\x)=\sign(\w\cdot\x+t) \mid \w \in \s^{d-1}, t \in \R\}$ denote the class of halfspaces on $\R^d$.
By Theorem~\ref{th functions of halfspaces}, $\H$ already admits 
an agnostic partial learner that achieves $(1+c)\opt+\eps$ error, 
for any desired constant $c>0$, with complexity 
$(\ell d)^\ell2^{\poly_c(\log(1/\eps))}$, with $\ell=\exp(O(1/c^4))$.
The generic approach establishing this bound leads to a 
a polynomial dependence of fairly large degree in the sample and computational complexities, as it requires sufficiently accurate sandwiching polynomials.


In this section, we give agnostic partial learners for halfspaces 
that achieve error $\widetilde{O}(\opt_{D_f}(\H))+\eps$ and $O(\opt_{D_f}(\H))+\eps$ 
with quantitatively much better complexities.
Formally, we establish the following result.


\begin{theorem}\label{th halfspace quasi}
 There is an algorithm such that for every distribution $D$ over $\R^d \times \{0,\pm 1\}$ such that $D_\x$ is logconcave and $W(D) \ge \alpha$,
    for every $\eps\in (0,1)$,
    it draws $O(d^4) \, \poly(1/(\alpha\eps))$ samples, 
    runs in $\poly(d)2^{\poly(\log(1/(\alpha\eps))}$ time 
    and outputs a hypothesis $\hat{h}$ such that with probability at least $99\%$, $\err_{D_f}(\hat{h})\le O(\opt_{D_f}(\H)) + \eps$.
\end{theorem}

An implication of Theorem~\ref{th halfspace quasi} is that to achieve 
$O(\opt)+\eps$ error for halfspace learning, we can avoid using 
the framework of polynomial approximation, which requires the algorithm to compute higher moment tensor of the marginal distribution, but only make use of information from the first four moments of the marginal distribution.
Although this observation gives us reasonably good sample complexity, 
that is polynomial with respect to all input parameters,
since we need to invoke Theorem~\ref{th direction finding} for our learner,
this prevents us from getting a fully-polynomial time algorithm.

Our second main result in this section shows that by making 
use of the first two moments of the marginal distribution, 
we can learn halfspaces within error $\log^8(1/\alpha)) \tilde O(\opt_{D_f}(\H)) + \eps$ in fully-polynomial time. 
\begin{theorem}\label{th halfspace polytime}
    There is an algorithm such that for every distribution $D$ over $\R^d \times \{0,\pm 1\}$ such that $D_\x$ logconcave and $W(D) \ge \alpha$,
    for every $\eps\in (0,1)$,
    it draws $O(d^2) \, \poly(1/(\alpha\eps))$ samples, 
    runs in $\poly(d, 1/(\alpha\eps))$ time 
    and outputs a hypothesis $\hat{h}$ such that with probability at least $99\%$, $\err_{D_f}(\hat{h})\le O(\log^8 (1/(\alpha\eps))\opt_{D_f}(\H) + \eps$. 
\end{theorem}

\begin{remark}
    We remark that if $\opt_{D_f}(\H)\le \eps$, $\err_{D_f}(\hat{h})\le \log^8(1/\alpha) \tilde O(\opt_{D_f}(\H)) + \eps$. On the other hand, if $\opt_{D_f}(\H)>\eps$, we can choose to run the algorithm by with $\eps' =\Theta(\opt_{D_f}(\H))$, which gives $\err_{D_f}(\hat{h})\le \log^8(1/\alpha) \tilde O(\opt_{D_f}(\H)) + \eps$.
\end{remark}

To prove the above two theorems, we will prove the following structural result 
(Lemma~\ref{lm moment to anti-concentration}).
Roughly speaking,
if the first four moments of the marginal distribution $D_\x$ are close 
to those of some logconcave distribution, 
then for every halfspace, at least a constant fraction of the examples are 
a constant distance from its decision boundary. 
On the other hand, if $D_\x$ only has its first two moments close to some logconcave distribution $D'_\x$, but it can be seen as a conditional distribution of $D'_\x$ over some event with mass $\alpha$, then at least $\Omega(\log^{-4}(1/\alpha))$ fraction of the examples are far from its decision boundary.

\begin{lemma}\label{lm moment to anti-concentration}

Let $D_\x$ be a distribution over $\R^d$. Then the following holds.
\begin{enumerate}[leftmargin=*]
    \item \label{it 4moment} Suppose that there is some $1/2$-isotropic logconcave distribution $D'_\x$,
such that for a unit vector $\w \in \s^{d-1}$ and $i \in [4]$,
$\abs{\E_{\x \sim D_\x}[(\w\cdot\x)^i]-\E_{\x \sim D'_\x}[(\w\cdot\x)^i]}<1/64$, 
then for every halfspace of the form $h_{(\w,t)}(\x)=\sign(\w\cdot \x+t)$, $\Pr_{\x\sim D_\x}[\abs{\w\cdot \x+t}>1/4]>1/C$, for some universal constant $C>1$.
\item \label{it 2moment} Suppose that there is some $1/2$-isotropic logconcave distribution $D'_\x$,
such that for a unit vector $\w \in \s^{d-1}$, for $i \in [2]$, $\abs{\E_{\x \sim D_\x}[(\w\cdot\x)^i]-\E_{\x \sim D'_\x}[(\w\cdot\x)^i]}<1/64$, and $D_\x$ is a conditional distribution of $D'_\x$ over some event $A$ such that $\Pr_{\x \sim D'_\x}[A] = \alpha$, then for every halfspace of the form $h_{(\w,t)}(\x)=\sign(\w\cdot \x+t)$,
$\Pr_{\x\sim D_\x}[\abs{\w\cdot \x+t}>1/4]>\Omega(\log^{-4}(1/\alpha))$.
\end{enumerate}
\end{lemma}

\begin{proof}[Proof of Lemma~\ref{lm moment to anti-concentration}]
Fix a halfspace $h_{(\w,t)}(\x)=\sign(\w\cdot \x+t)$.
Since $D'_\x$ is a mean-zero, $1/2$ isotropic logconcave distribution, by Fact~\ref{fact moment logconcave},
we have $\E_{\x \sim D'_\x}[(\w\cdot\x)^2]\le 5/2$ and $\E_{\x \sim D'_\x}[(\w\cdot\x)^4]\le c$ for some constant $c>1$.
Let $\Delta = 1/64$.
We consider two cases for the threshold $t$. 

In the first case, we assume $\abs{t}>4$. 
By Chebyshev's inequality and the moment assumption on $D_\x$, we have $\Pr_{\x\sim D_\x}[\abs{\w\cdot\x}>3] \le (2.5+\Delta)/9 \le 1/3$. 
This implies that with probability at least $2/3$, we have $\abs{\w\cdot\x}<3$. 
Since $\abs{t}>4$, if $\abs{\w\cdot\x}<3$, it must be the case, where $\abs{\w\cdot\x+t}>1$. 
This implies that with probability at least $2/3$, $\abs{\w\cdot\x+t}>1$.

In the second case, we assume $\abs{t}<4$. By Paley–Zygmund's inequality, we have for every $\theta \in (0,1)$, 
    \begin{align}\label{eq pz}
        \Pr_{\x\sim \D_\x}\left[\abs{\w\cdot\x+t}> \sqrt{\theta \E_{\x\sim \D_\x} [(\w\cdot\x+t)^2]}\right] \ge (1-\theta)^2 \frac{(\E_{\x\sim \D_\x} [(\w\cdot\x+t)^2])^2}{\E_{\x\sim \D_\x} [(\w\cdot\x+t)^4]}.
    \end{align}
We will next lower bound $\E_{\x\sim \D_\x} [(\w\cdot\x+t)^2]$ and upper bound $\E_{\x\sim \D_\x} [(\w\cdot\x+t)^4]$ through the moment information of $\w\cdot \x$. 
Denote by $\mu_\w: = \E_{\x\sim D'_\x}[\w\cdot\x]$,
Since $\abs{t}<4$,
we have
\begin{align*}
    \E_{\x\sim \D_\x} [(\w\cdot\x+t)^2] = \E_{\x\sim \D_\x} [(\w\cdot\x - \mu_\w)^2+(t+\mu_\w)^2 +2(t+\mu_\w)(\w\cdot\x-\mu_\w)] \ge 1/8.
\end{align*}

Since $\E_{\x \sim D_\x}[(\w\cdot \x)^2] \in [1/2-\Delta,(5/2)+\Delta]$, by H\"{o}lder's inequality, we have 
\begin{align*}
&\abs{\E_{\x \sim D_\x}[(\w\cdot \x)^3]} \le \sqrt{\E_{\x \sim D_\x}[(\w\cdot \x)^2]\E_{\x \sim D_\x}[(\w\cdot \x)^4]}\le\sqrt{((5/2)+\Delta)\E_{\x \sim D_\x}[(\w\cdot \x)^4]}.
\end{align*}
If $\abs{\E_{\x \sim D_\x}[(\w\cdot\x)^4]-\E_{\x \sim D'_\x}[(\w\cdot\x)^4]}<1/64$, then we have 
\begin{align*}
    \E_{\x\sim \D_\x} [(\w\cdot\x+t)^4] = \E_{\x\sim \D_\x} [(\w\cdot\x)^4+t^4 + 6(\w\cdot\x)^2t^2+ 4t(\w\cdot\x)^3+4t^3(\w\cdot\x)] \le C',
\end{align*}
for some constant $C'>1$.
Combine with \eqref{eq pz}, with $\theta=1/2$, we have $\Pr_{\x\sim \D_\x}\left[\abs{\w\cdot\x+t}> 1/4\right] \ge 1/C$, for some $C>1$. This proves \cref{it 4moment}.

To prove \cref{it 2moment},
 by Lemma~\ref{lm Q function learning},we know that there is a degree-4 polynomial $Q(\x)$, with $\E_{\x\sim D'_\x}[Q^2(\x)] \le O(\log^8(1/\alpha))$ such that 
\begin{align*}
   \E_{\x \sim D_\x}[(\w\cdot\x)^4] =  \E_{\x \sim D'_\x}[Q(\x)(\w\cdot\x)^4] \le \sqrt{\E_{\x \sim D'_\x}[Q^2(\x)]\E_{\x \sim D'_\x}[(\w\cdot\x)^8]} \le O(\log^4(1/\alpha)), 
\end{align*}
which implies
\begin{align*}
    \E_{\x\sim \D_\x} [(\w\cdot\x+t)^4] = \E_{\x\sim \D_\x} [(\w\cdot\x)^4+t^4 + 6(\w\cdot\x)^2t^2 + 4t(\w\cdot\x)^3+4t^3(\w\cdot\x)] \le O(\log^4(1/\alpha)).
\end{align*}
Combine with \eqref{eq pz}, with $\theta=1/2$, we have $\Pr_{\x\sim \D_\x}\left[\abs{\w\cdot\x+t}> 1/4\right] \ge \Omega(\log^{-4}(1/\alpha))$.
\end{proof}

With Lemma~\ref{lm robust chow parameter} and Lemma~\ref{lm moment to anti-concentration}, we can immediately obtain weak learning algorithms for learning halfspaces. We summarize our weak learning algorithm as Algorithm~\ref{alg weak-learning anti}.

\begin{algorithm}[htbp]
		\caption{\textsc{HalfspaceWeakLearning} (Weak learning halfspaces under structured distributions)}\label{alg weak-learning anti}
		\begin{algorithmic} [1]
\State \textbf{Input:}  Sample access to $D$, diameter parameter $B>0$, advantage parameter $\beta$.
\State \textbf{Output:} $\hat{h}:$ weak hypothesis over $D$.

\medskip

\State $h_+(\x):= 1$, $h_-(\x):= -1$, $\mathcal{C}:=\{h_+,h_-\}$
\State Draw $S\sim D^m$ a set of $m=\tilde O(dB^4/\beta^4)$ \iid examples from $D$ 
\State Find $\w \in \B^{d-1}, t \in [-B,B]$ that maximize $\sum_{(\x,y) \in S}(y(\w\cdot\x+t))$ through linear regression and define $F(\x):=\w\cdot\x+t$.
\State $g(\x)=\proj_{[-1,1]}
\left(\frac{\beta F(\x)}{16B^2}\right): \R^d \to [-1,1]$
\For{$i \in [O(B^2/\beta^2)]$}
\State Sample $t'$ uniformly from $ [-1,1]$ and define $h_i(\x):=\sign(g(\x)-t')$
\State $\mathcal{C} \gets \mathcal{C} \cup \{h_i(\x)\}$
\EndFor
\State Draw $S'\sim D^m$ a set of $m=\tilde O(dB^4/\beta^4)$ \iid examples from $D$
\State \Return $\hat{h}:= \argmin_{h \in \mathcal{C}}\sum_{(\x,y)\in S'}\Ind\{h(\x)\neq y\}$.

\end{algorithmic}
\end{algorithm}

We first show that there is a polynomial time $(4,1/64,\Omega(1),O(1))$-weak learner for the class of halfspaces, which immediately implies Theorem~\ref{th halfspace quasi} using Theorem~\ref{thm: main-weak-to-strong}.

\begin{lemma}\label{lm halfspace 4moment}
    There is an algorithm $\A$ that draws $O(d)$ examples, runs in $\poly(d)$ time and $(4,1/64,\Omega(1),O(1))$-weakly learns the class of halfspaces with probability $99\%$.
\end{lemma}

\begin{proof}[Proof of Lemma~\ref{lm halfspace 4moment}]
Since the class of halfspaces is closed under affine transformations, we show that by setting the input parameter $B=O(1)$, Algorithm~\ref{alg weak-learning anti} outputs a weak hypothesis with $\Omega(1)$ advantage in $\poly(d)$ time.

Let $D$ be the input distribution over $\R^d \times \{\pm 1\}$. 
Let $D'_\x$, be a $1/2$-isotropic logconcave distribution, 
such that for any unit vector $\w \in \s^{d-1}$ and,
$\abs{\E_{\x \sim D_\x}[(\w\cdot\x)^i]-\E_{\x \sim D'_\x}[(\w\cdot\x)^i]}<1/64$, $i \in [4]$.
    Let $h_{\w^*,t^*} = \sign(\w^*\cdot\x + t^*)$ be the halfspace such that $\err_D(h_{\w^*,t^*}) = \opt< 1/C$, for some constant $C>1$.
    We first show that if $\abs{t^*}>4$, then a constant hypothesis $h_+(\x)=1$ or $h_-(\x)=-1$ must have a constant advantage.

Without loss of generality, we assume that $t^*>4$. By assumption and Chebyshev's inequality, we have 
\begin{align*}
    \Pr_{\x\sim D_\x}[h_{\w^*,t^*}(\x) = -1] = \Pr_{\x \sim D_\x}[\w^*\cdot\x < -4] \le 1/4.
\end{align*}
Since $\opt<1/C$, we know that $\Pr_{(\x,y)\sim D}[y=-1] \le \Pr_{\x\sim D_\x}[h_{\w^*,t^*}(\x) = -1] + \opt<2/5$. This implies that $h_+$ has $1/10$ advantage over $D$. So, in the rest of the proof, we will assume $\abs{t^*}<4$.

By Lemma~\ref{lm moment to anti-concentration}, we know that $\Pr_{\x\sim D_\x}[\abs{\w^*\cdot\x + t^*}>1/4]>1/C'$, for some universal constant $C'$. 
Let $F^*(\x) = \w^*\cdot\x + t^*$, then $\E_{\x \sim D_\x}[h_{w^*,t^*}(\x)F^*(\x)]>1/(4C')$. By the assumption on the moment information of $D_\x$, and $\abs{t^*}<4$, we know that $\E_{\x \sim D_\x}[(F^*(\x))^2]<O(1)$. 
By Lemma~\ref{lm robust chow parameter}, we know that when $\opt<1/C$, for some constant $C>1$, we have $\E_{(\x,y)\sim D}[yF^*(\x)]>2\beta$, for some $\beta>1/C$. 
Thus, using standard sample complexity analysis of linear regression (see, e.g., \cite{bartlett2002rademacher},) by running a linear regression over $m=\tilde O (d/\beta^2)$ examples drawn from $D$, we are able to find an affine function 
$F(\x) = \w\cdot \x + t$, $\w \in \B^{d-1}, \abs{t}\le 4$, such that $\E_{(\x,y)\sim D}[yF(\x)] \ge \E_{(\x,y)\sim D}[yF^*(\x)]-\beta \ge \beta$.
By Lemma~\ref{lm robust chow parameter} and Fact~\ref{fact correlation}, 
we know that using $F(\x)$ and a random threshold $t'$ sampled uniformly from $[-1,1]$, we can find a boolean function $h_i(\x)$ that has $\Omega(1)$ advantage with respect to $D$, with probability $\Omega(1)$.

Thus, by randomly sampling $t'$, $\Omega(1)$ times, Algorithm~\ref{alg weak-learning anti} finds such a weak hypothesis with high probability. Thus, by running a standard hypothesis testing approach, Algorithm~\ref{alg weak-learning anti} outputs a weak hypothesis as the statement.
\end{proof}

Combining Lemma~\ref{lm halfspace 4moment} and Theorem~\ref{thm: main-weak-to-strong}, we immediately get a quasi-polynomial time algorithm that agnostically learns a partial halfspace to $O(\opt)+\eps$ error under any logconcave distribution and prove Theorem~\ref{th halfspace quasi}.

We remark that the quasi-polynomial time of the algorithm is due to the use of a quasi-polynomial time moment discrepancy finder Algorithm~\ref{alg mismatch finding} for the $4$th moment.
If we have a faster algorithm for finding moment discrepancies, then the time complexity of the learning algorithm in Theorem~\ref{th halfspace quasi} can be improved.
On the other hand, since we can always efficiently find the moment discrepancies for the first two moments through mean estimation and singular value decomposition, 
we will give a polynomial time algorithm for agnostic partial learning halfspaces with $\tilde O(\opt)+\eps$ error.

To begin with, we give the weak learning algorithm for agnostic learning halfspaces under assumptions on their first two moments. 
\begin{lemma}\label{lm halfspace 2ndmoment}
    Let $D$ be a distribution over $\R^d \times \{\pm 1\}$ such that there is a $1/2$-isotropic logconcave distribution $D'_\x$, such that  
    $\abs{\E_{\x \sim D_\x}[(\w\cdot\x)^i]-\E_{\x \sim D'_\x}[(\w\cdot\x)^i]}<1/64$, for $i \in [2]$ and $D_\x$ is a conditional distribution of $D'_\x$ over some event $A$ such that $\Pr_{\x \sim D'_\x}[A] = \alpha$. Let $\opt = \inf_{h \in \H} \err_D(h)$, where $\H$ is the class of halfspaces. Suppose $\opt<c_0\log^{-8}(1/\alpha)$, for a sufficiently small constant $c_0>0$, then Algorithm~\ref{alg weak-learning anti} draws $\poly(d,\log(1/\alpha))$ examples from $D$, runs in $\poly(d,\log(1/\alpha))$ time and
    outputs a hypothesis $\hat{h}$ such that $\err_D(\hat{h})<1/2-\polylog^{-1}(1/\alpha)$ with probability at least $99\%$.
\end{lemma}

\begin{proof}[Proof of Lemma~\ref{lm halfspace 2ndmoment}]

Let $h_{\w^*,t^*}$ be the optimal halfspace under $D$.
We first show that if $\abs{t^*}>4$, then a constant hypothesis $h_+(\x)=1$ or $h_-(\x)=-1$ must have a constant advantage.

Without loss of generality, we assume that $t^*>4$. By assumption and Chebyshev's inequality, we have 
\begin{align*}
    \Pr_{\x\sim D_\x}[h_{\w^*,t^*}(\x) = -1] = \Pr_{\x \sim D_\x}[\w^*\cdot\x < -4] \le 1/4 \;.
\end{align*}
Since $\opt<\log^{-8}(1/\alpha)$, we know that $\Pr_{(\x,y)\sim D}[y=-1] \le \Pr_{\x\sim D_\x}[h_{\w^*,t^*}(\x) = -1] + \opt<2/5$. This implies that $h_+$ has $1/10$ advantage over $D$. So, in the rest of the proof, we will assume $\abs{t^*}<4$.

By Lemma~\ref{lm moment to anti-concentration}, we know that 
$\Pr_{\x\sim D_\x}[\abs{\w^*\cdot\x + t^*}>1/4]>1/\Omega(\log^4(1/\alpha))$. 
Letting $F^*(\x) = \w^*\cdot\x + t^*$, we have 
$\E_{\x \sim D_\x}[h_{w^*,t^*}(\x)F^*(\x)]>1/\Omega(\log^4(1/\alpha))$. 
By the assumption on the moment information of $D_\x$, and $\abs{t^*}<4$, 
we know that $\E_{\x \sim D_\x}[(F^*(\x))^2]<O(1)$. 
By Lemma~\ref{lm robust chow parameter}, since
$
\E_{\x\sim D_\x}[h^*(\x)F^*(\x)]
\ge \Omega(\log^{-4}(1/\alpha))
$
and
$
\E_{\x\sim D_\x}[(F^*(\x))^2]\le O(1),
$
we know that when
$
\opt\le c_0(1+\log(1/\alpha))^{-8}
$
for some sufficiently small universal constant $c_0>0$,
it holds
$
\E_{(\x,y)\sim D}[yF^*(\x)]\ge 2\beta,
$
where
$\beta=\Omega((1+\log(1/\alpha))^{-4})$.

Thus, by the same linear regression argument, we find
an affine function $F(\x)=\w\cdot\x+t$ satisfying
$
\E_{(\x,y)\sim D}[yF(\x)]\ge\beta.
$
By the corrected clipping step in Algorithm 6 and
Fact 7.2, we obtain a weak hypothesis with advantage
$\Omega(\beta^2)=
\Omega((1+\log(1/\alpha))^{-8})$.
\end{proof}

\begin{proof}[Proof of Theorem~\ref{th halfspace polytime}]
    We will use Algorithm~\ref{alg weak-learning anti} as the weak learning algorithm $\A_w$ and singular-value decomposition as a moment discrepancy finding algorithm $\A_m$ used in Algorithm~\ref{alg decision tree k moment}.
We notice that given any $1/2$-isotropic logconcave distribution $D_\x$ with weight $\alpha \in (0,1)$, $\A_m$ is a $(2,1/64,1/2)$ discrepancy finder for the distribution $D_\x$ and $(D_f)_\x$ with sample complexity $\poly(d/\alpha)$ and time complexity $\poly(d/\alpha)$.

For every $\alpha \in (0,1)$ and $\eps \in (0,1)$, let $D$ be any distribution over $\R^d \times \{\pm 1,0\}$, with $D_\x$ a logconcave distribution, $W(D) \ge \alpha\eps$ and $\opt_{D_f}(\H)<1/\Omega(\log^8(1/(\alpha\eps)))$.
Denote by $\mathcal{U}$, the partition generated by Algorithm~\ref{alg decision tree k moment}, after it terminates. By Lemma~\ref{lm partition}, we know that there is a constant $a>0$, such that the choice of $\xi = (\alpha\eps)^a$ satisfies the following property.

Let $\mathcal{B}=\{R \in \mathcal{U} \mid \Pr_{(\x,y)\sim D_f}[\x \in R] \le \xi \}$. Then, $\sum_{R \in \mathcal{B}}\Pr_{(\x,y)\sim D_f}[\x \in R] \le 1/10$.
We denote by $\mathcal{G} = \mathcal{U} \setminus \mathcal{B}$ and denote by $\mathcal{M} \subseteq \mathcal{G}$ the set of $R$, where a random binary function $h_R$ is assigned.
By definition, for each $R \in \mathcal{G}$, we have $\Pr_{(\x,y)\sim D_f}[\x \in R]>\xi$.
This implies that the weight of $D_R$ is at least $(\alpha\eps)^{1+a}$.

We show that $\sum_{R \in \mathcal{M}}\Pr_{(\x,y)\sim D_f}[\x \in R]$ is small.
By Lemma~\ref{lm halfspace 2ndmoment}, for each $R \in \mathcal{M}$, we have
\[
\Pr_{(\x,y)\sim D_f}[h^*(\x) \neq y \mid \x \in R]
>\Omega(\log^{-8}(1/(\alpha\eps))),
\]
where $h^*$ is the hypothesis in $\H$ that achieves $\opt_{D_f}$.
Now suppose that $\sum_{R \in \mathcal{M}}\Pr_{(\x,y)\sim D_f}[\x \in R]>1/2$, then we have
\begin{align*}
\err_{D_f}(h^*)
\ge \sum_{R \in \mathcal{M}}
\Pr_{(\x,y)\sim D_f}[\x \in R]
\Pr_{(\x,y)\sim D_f}[h^*(\x) \neq y \mid \x \in R]
>\Omega(\log^{-8}(1/(\alpha\eps))),
\end{align*}
which gives a contradiction to the fact that
$\opt_{D_f}<1/\Omega(\log^8(1/(\alpha\eps)))$,
for a sufficiently small implicit constant in the noise bound.
Thus, we have
$
\sum_{R \in \mathcal{G}\setminus \mathcal{M}}
\Pr_{(\x,y)\sim D_f}[\x \in R]>2/5.
$
Since for each $R \in \mathcal{G}\setminus \mathcal{M}$,
we have
$
\Pr_{(\x,y)\sim D_f}[h_R(\x) \neq y \mid R]
<
1/2-\Omega(\log^{-8}(1/(\alpha\eps))),
$
we know that
\[
\err_{D_f}(\hat{h})
\le 1/2-\Omega(\log^{-8}(1/(\alpha\eps))).
\]

As a summary, for every $\alpha \in (0,1)$ and $\eps \in (0,1)$, let $D$ be any distribution over $\R^d \times \{\pm 1\}$, with $D_\x$ a logconcave distribution, $W(D) \ge \alpha\eps$, and $\opt_{D_f}(\H)<1/\Omega(\log^8(1/(\alpha\eps)))$, Algorithm~\ref{alg decision tree k moment} runs in $\poly(d/(\alpha\eps))$ time and outputs a hypothesis $\hat{h}$ with error at most $1/2-\Omega(\log^{-8}(1/(\alpha\eps))).$

By Theorem~\ref{th condition boosting}, we know that every distribution $D$ over $\R^d \times \{0,\pm 1\}$, such that $D_\x$ is a logconcave distribution, and $\Pr_{(\x,y)\sim D}[y \neq 0] \ge \alpha$, for every $\eps\in (0,1)$, we can call Algorithm~\ref{alg decision tree k moment}
$
\widetilde{O}\left(
\frac{\log^{16}(1/(\alpha\eps))}{\eps}
\right)
$
times to get a hypothesis with error at most
$
O\left(
\log^8(1/(\alpha\eps))
\opt_{D_f}(\H)+\eps
\right).
$
\end{proof}

\begin{remark}[Polynomial-time testable learning]
The algorithm of Theorem~\ref{th halfspace polytime} extends to the testable setting by 
supplementing the first two moment checks with a fourth-moment upper bound test. This additional 
check is necessary because matching the first two moments alone does not certify the consequence of 
conditioning used in Lemma~\ref{lm halfspace 2ndmoment}. We follow the truncation and rejection 
rules of Section~\ref{sec proof main}. In each cell of the partition, let \(P\) denote the reference 
logconcave distribution, let \(Q\) denote the truncated effective marginal used for all moment 
checks, and let \(\rho\) lower-bound its conditioning weight. 
Set \(L=\log(e/\rho)\), \(\phi(\x)=\x\otimes \x\), and 
\(M_\nu=\E_{\x\sim\nu}[\phi(\x)\phi(\x)^\top]\), for a distribution $\nu$. 
We estimate \(M_Q\) and \(M_P\) to sufficiently small operator errors \(O(1)\) and \(O(L^{-4})\), 
respectively, and reject unless \(\widehat M_Q\preceq CL^4\widehat M_P+I\). Completeness follows 
from H\"{o}lder's inequality and Fact~\ref{fact polynomial growth}, 
which give \(\E_{\x\sim Q}[p^2]\leq CL^4 \E_{\x\sim P}[p^2]\) 
for every quadratic polynomial \(p\) when \(Q\) is obtained by conditioning 
or thinning \(P\) with weight at least \(\rho\). Conversely, acceptance certifies 
\(\E_{\x\sim Q}[(\w\cdot \x)^4]=O(L^4)\) for every unit vector \(\w\), 
by applying the matrix inequality to \(\w\otimes \w\). 
For the bounded-threshold optimal affine score \(F^* =\w^*\cdot\x+t^* \) considered 
in Lemma~\ref{lm halfspace 2ndmoment}, the first two moment conditions and 
this fourth-moment bound imply 
\(\|F^\ast\|_{L^1(Q)}\geq \|F^\ast\|_{L^2(Q)}^3/\|F^\ast\|_{L^4(Q)}^2=\Omega(L^{-2})\). 
Thus, Lemma~\ref{lm robust chow parameter} yields Boolean advantage \(\Omega(L^{-4})\) 
whenever the local optimal error is at most \(cL^{-4}\), using its correlation-dependent 
clipping \(g(x)=\operatorname{clip}_{[-1,1]}(c' L^{-2}F(x))\) 
in Algorithm~\ref{alg weak-learning anti}. The additional test involves matrices 
of dimension \(d^2\); boundedness of the truncated samples and the reference moment 
bounds permit their estimation and comparison in polynomial time. Consequently, 
the degree-two partitioning and boosting arguments give a testable learner 
with the same error guarantee and polynomial sample and time complexity.
\end{remark}

\subsection{Agnostic Multiclass Linear Classification via Partial Linear Classification}\label{sec multiclass}

In this section, we consider the problem of 
agnostic partial learning of multiclass linear classifiers 
under logconcave distributions.
In a multiclass classification problem, 
the label of an example is not from $\{\pm 1\}$, but from $[K]$ for some $K>2$.
A natural paradigm for solving a multiclass classification task 
is to reduce the problem to binary classification. 
Roughly speaking, for every pair of labels $(i,j)$, 
we want to learn a hypothesis $h_{ij}$ to tell whether 
an example is more likely to have label $i$ or label $j$. 
By combining these binary classifiers, we obtain 
a reasonably robust multiclass classifier.
We remark that even though $D_\x$ might be a well-behaved distribution, 
if we consider examples whose labels lie in $\{i,j \}$, their marginal distribution 
can be significantly different from the underlying distribution on the entire space; e.g., a Gaussian or logconcave , under which efficient learning algorithms were previously known in the literature.
Fortunately, our notion of agnostic {\em partial} learning 
provides a direct way to address this technical issue.
As concrete applications, for multiclass linear classification, 
we give a fully-polynomial time algorithm that achieves 
error $\tilde{O}(K \opt)+\eps$ over any logconcave distribution. 
Furthermore, we show that by allowing the runtime to become quasi-polynomial in $K/\eps$, we can get an algorithm with error $O(K) \opt+\eps$.

As an important application of agnostic partial linear classification, we give the first quasi-polynomial and polynomial time algorithm for agnostic multiclass linear classification under logconcave distribution. 
To begin with, we introduce the learning problem formally.



\begin{definition}[Agnostic Partial Learning Multiclass Linear Classifiers] \label{thm:mlc-formal}
    Let $K \in \Z_+, K \ge 2$.
    Define $\mathcal{M}_{K, d}=\{h_W(\x):\R^d \to [K] \mid h_W(\x) = \argmax_{i \in [K]}(W_i\cdot (\x,1)), W \in \R^{d+1\times K} \}$ be the class of $K$-class linear classifier($K$-MLC). 
    Let $D$ be a distribution over $\R^d \times [K] \cup \{0\}$ such that $D_\x$ is logconcave. Denote by $D_f$ the conditional distribution of $D$ over the event $\mathcal{E} \eqdef \{(\x,y)\mid y \neq 0\}$.  
    Let $h: \R^d \to [K]$ be a multiclass classifier and define $\err_{D_f}(h) = \Pr_{(\x,y)\sim D_f}[h(\x) \neq y]$. Denote by $\opt_{D_f}(\mathcal{M}_{K, d}) = \inf_{h \in \mathcal{M}_{K, d}} \err_{D_f}(h)$.
     Given a set of \iid sample access to $D$, and parameter $\eps,\delta \in (0,1)$, the goal of an algorithm is to output a hypothesis $\hat{h}$ such that with probability at least $1-\delta$, $\err_{D_f}(\hat{h}) \le F(\opt_{D_f}(\mathcal{M}_{K, d}))+\epsilon$ for some function $F:(0,1) \to (0,1)$. 
\end{definition}

The main result of this section is the following theorem.


\begin{theorem}[Agnostic Partial Learning $K$-MLC under Logconcave Distributions]\label{th multiclass}
    Consider the problem of agnostic partial learning multiclass linear classifiers under logconcave distributions. 
    \begin{enumerate}[leftmargin=*]
        \item  There is an algorithm $\A$ 
    such that for every input distribution $D$ over $\R^d \times \{ [K] \cup \{0\}\}$, with $W(D) \ge \alpha$
    such that $D_\x$ is a logconcave distribution, it draws $\poly(dK/(\alpha\eps))$ samples from $D$, 
    runs in $\poly(d)2^{\polylog(K/(\alpha\eps))}$ time 
    and outputs a hypothesis $\hat{h}: \R^d \to [K]$, such that with probability at least $99\%$, $\err_{D_f}(\hat{h})<O(K\opt_{D_f}(\mathcal{M}_{K, d}))+\eps$, where $\mathcal{M}_{K, d}$ is the family 
    of $K$-class linear classifiers.

    \item  There is an algorithm $\A$ 
    such that for every input distribution $D$ over $\R^d \times \{[K] \cup \{0\}\}$, with $W(D) \ge \alpha$
    such that $D_\x$ is logconcave, 
    it draws $\poly(dK/(\alpha\eps))$ samples from $D$, 
    runs in $\poly(dK/(\alpha\eps))$ time 
    and outputs a hypothesis $\hat{h}: \R^d \to [K]$, such that with probability at least $99\%$, $\err_{D_f}(\hat{h})< O(\log^8(K/(\alpha\eps))) (K\opt_{D_f}(\mathcal{M}_{K, d}))+\eps$, where $\mathcal{M}_{K, d}$ is the class of $K$-class linear classifier.
    \end{enumerate}  
\end{theorem}

\begin{algorithm}[htbp]
		\caption{\textsc{MulticlassReduction} (Multiclass classification using binary classification)}\label{alg multiclass reduction}
		\begin{algorithmic} [1]
\State \textbf{Input:}  Sample access from $D$, accuracy parameter $\eps>0$, halfspace learning algorithm $\A$.
\State \textbf{Output:} $\hat{h}: \R^d \to [K]$ hypothesis over $D$.

\medskip

\For{$i,j \in [K], i < j$}
\State Define $D_{ij}$ over $\R^d \times \{i,j,0\}$ sampled as follows. 
Draw $(\x,y)\sim D$, if $y \in \{i,j\}$, return $(\x,y)$, otherwise, return $(\x,0)$
\State Run $\A$ over input distribution $D_{ij}$ with accuracy parameter $\eps$ 
and get hypothesis $\hat{h}_{ij}:\R^d \to \{i,j\}$ and set $\hat{h}_{ji} = \hat{h}_{ij}$  (If the rate of seeing label $i,j$ is small, return an arbitrary classifier)
\EndFor
\State Return hypothesis $\hat{h}(\x) = \argmax_{i \in [K]} s(\x,i)$, 
where $s(\x,i) = \sum_{j \neq i} \Ind\{\hat h_{ij}(\x) = i\}$.
\end{algorithmic}
\end{algorithm}

Our algorithm for learning multiclass linear classifiers is obtained via a simple reduction to the binary case. As already noted, our reduction crucially leverages our partial learning framework.

\begin{lemma}[Reduction of MLC to Binary Linear Classification]\label{lm reduction multi-binary}
     Suppose there is an algorithm $\A$ for the problem of agnostic partial learning the class of halfspaces $\H$, 
     such that for every input distribution $D$ over 
     $\R^d \times \{\pm 1, 0\}$ and $\eps \in (0,1)$, 
     it draws $M(\eps)$ examples, runs in $T(\eps)$ time and outputs a hypothesis $\hat{h}_b: \R^d \to \{\pm 1\}$ with 
     $\err_D(\hat{h}_b)<C(\eps)(\opt_D(\H))+\eps$, 
     where $C(\eps)>0$ is some constant that only depends on $\eps$.
     Then there is an algorithm for the problem of agnostic partial learning of  multiclass linear classifiers  
     such that for every input distribution $D \times \{[K]\cup \{0\}\}$ and $\eps \in (0,1)$, 
     it draws $O(K^2M(\eps))$ examples, runs in $O(K^2 T(\eps))$ time and outputs a hypothesis $\hat{h}: \R^d \to [K]$ such that $\err_D(\hat{h})\le KC(\eps)(\opt_D(\mathcal{M}_{K, d}))+K^2\eps$.
\end{lemma}

\begin{proof}[Proof of Lemma~\ref{lm reduction multi-binary}]
We show that Algorithm~\ref{alg multiclass reduction} satisfies the statement 
of Lemma~\ref{lm reduction multi-binary}. The sample and time 
complexities of Algorithm~\ref{alg multiclass reduction} follow directly. In the 
rest of the proof, we establish the correctness 
of Algorithm~\ref{alg multiclass reduction}.

    Let $h_W$ be any multiclass linear classifier. For every $ i \neq j \in [K]$, let $h_{ij}(\x) = \sign((W_i-W_j)\cdot (\x,1))$ be a halfspace. 
    Notice that for every $\x$ and $i \in [K]$, $h_W(\x) = i$ if and only if $h_{ij}(\x) = 1$, 
    for all $j \neq i$.
    To simplify the notation, we will also treat the label of $h_{ij}$ in $\{i,j\}$, 
    where $h_{ij}(\x) = i$ if and only if the halfspace predicts $+1$.
    
    For every $i\neq j$, we define distribution $D_{ij}$ over $\R^d \times\{i,j, 0\}$ that is sampled in the following way.
    Draw $(\x,y) \sim D$, if $y \in \{i,j\}$, return $(\x,y)$, otherwise return $(\x,0)$.
    Let $h^*=h_{W^*}$ be the multiclass linear classifier that achieves optimal error $\opt$ over $D$ and let $h^*_{ij} = \sign((W^*_i-W^*_j)\cdot (\x,1))$. 
    For each pair of $i \neq j$, we have 
    \begin{align*}
        \err_{D_{ij}}(h^*_{ij}) & = \Pr_{(\x,y)\sim D}[h^* _{ij}(\x) \neq y, y \in \{i,j\} ] \\
        & =\sum_{r \in [K]}\Pr_{(\x,y)\sim D}[h^* _{ij}(\x) \neq y, h^*(\x)=r, y \in \{i,j\} ] \\
        & \le \Pr_{(\x,y)\sim D}[h^*(\x) = i, y = j] + \Pr_{(\x,y)\sim D}[h^*(\x) = j, y = i] + \sum_{r \not \in \{i,j\}}\Pr_{(\x,y)\sim D}[h^*(\x)=r, y \in \{i,j\} ]   \\
        & = \opt_{ij}+\opt_{ji}+ \sum_{r \not \in \{i,j\}} (\opt_{r,i}+\opt_{r,j}),
    \end{align*}
    where $\opt_{r,i}=\Pr_{(\x,y)\sim D}[h^*(\x)=r, y = i ]$
This implies by running $\A$ over $D_{ij}$, it draws $M(\eps)$ examples, runs in $T(\eps)$ time, and outputs a hypothesis $\hat{h}_{ij}$ such that $\err_{D_{ij}}(\hat{h}_{ij}) \le C(\eps)(\opt_{ij}+\opt_{ji}+\sum_{r \not \in \{i,j\}} (\opt_{r,i}+\opt_{r,j}))+\eps$.

Using the $O(K^2)$ learned $\hat{h}_{ij}$, we define the function $\hat{h}$ as follows. For each $\x$ and each $i \in [K]$, 
we define $s(\x,i) = \sum_{j \neq i} \Ind\{\hat h_{ij}(\x) = i\}$ and return $\hat{h}(\x) = \argmax_{i \in [K]} s(\x,i)$. 
We have 
\begin{align*}
    \err_D(\hat{h}) & = \sum_{j = 1}^K \Pr_{(\x,y)\sim D}[\hat{h}(\x) \neq y, y=j] \le \sum_{j = 1}^K \sum_{i \neq j} \Pr_{(\x,y)\sim D}[\hat{h}_{ji}(\x) \neq j, y=j] \\
    & = \frac{1}{2} \sum_{(i,j):i \neq j} \Pr_{(\x,y)\sim D}[\hat{h} _{ij}(\x) \neq y, y \in \{i,j\} ] \\
    &\le \frac{1}{2} \sum_{(i,j):i \neq j} \left(C(\eps)(\opt_{ij}+\opt_{ji}+\sum_{r \not \in \{i,j\}} (\opt_{r,i}+\opt_{r,j}))+\eps\right) \\
    & \le \frac{1}{2} C(\eps) \left( \sum_{i=1}^K\sum_{j=1}^K\sum_{r \neq i}\opt_{r,i} +\sum_{i=1}^K\sum_{j=1}^K\sum_{r \neq j}\opt_{r,j}\right)+ \frac{1}{2}K^2\eps \\
    &\le KC(\eps)(\opt_D(\mathcal{M}_{K, d}))+K^2\eps, 
\end{align*}
where the last inequality holds by the fact that $\sum_{(i,j):i\neq j}\opt_{ij} = \opt$. 
\end{proof}

We now make use of Lemma~\ref{lm reduction multi-binary} to prove Theorem~\ref{th multiclass}.

\begin{proof}[Proof of Theorem~\ref{th multiclass}]
Recall that for every hypothesis $h: \R^d \to [K]$ and for every $D$ over $\R^d \times [K] \cup \{0\}$, 
the global error satisfies that $\err_{D_f}(h) = \err_D(h)/W(D)$. 
Since $W(D)\ge \alpha$,
this implies that for some $\eps'$ that we will choose later, if we get a hypothesis such that $
\err_D(\hat{h})\le KC(\eps')(\opt_D(\mathcal{M}_{K, d}))+K^2\eps'$
\begin{align*}    
\err_{D_f}(\hat{h})\le KC(\eps')(\opt_{D_f}(\mathcal{M}_{K, d}))+K^2\eps'/\alpha.
\end{align*}
By Lemma~\ref{lm reduction multi-binary}, by setting $\eps' = \alpha\eps/K^2$, suppose we have an algorithm that under every distribution $D'$ over $\R^d \times \{0,\pm 1\}$, 
output a boolean hypothesis $h_b: \R^d \to \{\pm 1\}$, such that $\err_{D'}(h_b) \le C(\eps')(\opt_{D'}(\H)) + \eps'$,
which means $\err_{D'_f}(h_b) \le C(\eps')(\opt_{D'_f}(\H)) + (\eps/K^2)$.
then Algorithm~\ref{alg multiclass reduction}, gives a hypothesis $\hat{h}$ such that 
\begin{align*}    
\err_{D_f}(\hat{h})\le KC(\eps')(\opt_{D_f}(\mathcal{M}_{K, d}))+\eps.
\end{align*}
This implies that by running the algorithm in Theorem~\ref{th halfspace quasi} 
and Theorem~\ref{th halfspace polytime} $O(K^2)$ times, we can get a hypothesis $\hat{h}$,
with error 
$\err(\hat{h})= O(K\opt_{D_f}(\mathcal{M}_{K, d}))+\eps$
and 
$\err(\hat{h}) = O(\log^8(K/(\alpha\eps))) (K\opt_{D_f}(\mathcal{M}_{K, d}))+\eps$, 
respectively.
This concludes the proof of Theorem~\ref{th multiclass}.
\end{proof}

\paragraph{AI Discloure}
The authors proved all theorems and wrote all proofs in this work. 
The version submitted to FOCS 2026 in early April was prepared 
entirely without AI input. AI was used only during the subsequent 
preparation of this full version to audit the existing proofs and 
algorithms for minor errors and edge cases, and to identify 
typographical errors. The authors reviewed the feedback and 
incorporated minor corrections. AI was not used to derive any 
results or construct any proofs.

\bibliography{mydb}

\newcommand{\etalchar}[1]{$^{#1}$}
\begin{thebibliography}{DMRT25b}
\expandafter\ifx\csname url\endcsname\relax
  \def\url#1{\texttt{#1}}\fi
\expandafter\ifx\csname doi\endcsname\relax
  \def\doi#1{\burlalt{doi:#1}{http://dx.doi.org/#1}}\fi
\expandafter\ifx\csname urlprefix\endcsname\relax\def\urlprefix{URL }\fi
\expandafter\ifx\csname href\endcsname\relax
  \def\href#1#2{#2}\fi
\expandafter\ifx\csname burlalt\endcsname\relax
  \def\burlalt#1#2{\href{#2}{#1}}\fi

\bibitem[ABL17]{ABL17}
P.~Awasthi, M.~F. Balcan, and P.~M. Long.
\newblock The power of localization for efficiently learning linear separators with noise.
\newblock {\em J. {ACM}}, 63(6):50:1--50:27, 2017.

\bibitem[AK18]{arutyunyan2018deviation}
L.~M. Arutyunyan and E.~D. Kosov.
\newblock Deviation of polynomials from their expectations and isoperimetry.
\newblock 2018.

\bibitem[APS26]{alman2026learningfunctionshalfspaces}
J.~Alman, S.~Patel, and R.~A. Servedio.
\newblock Learning functions of halfspaces.
\newblock 2026, \burlalt{2603.08700}{http://arxiv.org/abs/2603.08700}.
\newblock \urlprefix\url{https://arxiv.org/abs/2603.08700}.

\bibitem[AV06]{arriaga2006algorithmic}
R.~Arriaga and S.~Vempala.
\newblock An algorithmic theory of learning: Robust concepts and random projection.
\newblock {\em Machine learning}, 63:161--182, 2006.

\bibitem[BK97]{blum1997learning}
A.~Blum and R.~Kannan.
\newblock Learning an intersection of a constant number of halfspaces over a uniform distribution.
\newblock {\em Journal of Computer and System Sciences}, 54(2):371--380, 1997.

\bibitem[BM02]{bartlett2002rademacher}
P.~L. Bartlett and S.~Mendelson.
\newblock Rademacher and gaussian complexities: Risk bounds and structural results.
\newblock {\em Journal of machine learning research}, 3(Nov):463--482, 2002.

\bibitem[BS71]{bochnak1971polynomials}
J.~Bochnak and J.~Siciak.
\newblock Polynomials and multilinear mappings in topological vector-spaces.
\newblock {\em Studia Mathematica}, 39(1):59--76, 1971.

\bibitem[CW01]{carbery2001distributional}
A.~Carbery and J.~Wright.
\newblock Distributional and l\^{} q norm inequalities for polynomials over convex bodies in r\^{} n.
\newblock {\em Mathematical research letters}, 8(3):233--248, 2001.

\bibitem[Dan16]{Daniely16}
A.~Daniely.
\newblock Complexity theoretic limitations on learning halfspaces.
\newblock In {\em Proceedings of the 48\textsuperscript{th} Annual Symposium on Theory of Computing, {STOC} 2016}, pages 105--117, 2016.

\bibitem[DIKZ25]{DiakonikolasIKZ25}
I.~Diakonikolas, G.~Iakovidis, D.~M. Kane, and N.~Zarifis.
\newblock Robust learning of multi-index models via iterative subspace approximation.
\newblock In {\em 66th {IEEE} Annual Symposium on Foundations of Computer Science, {FOCS} 2025, Sydney, Australia, December 14-17, 2025}, pages 2231--2239. {IEEE}, 2025.

\bibitem[DIM26]{diakonikolas2026polynomial}
I.~Diakonikolas, G.~Iakovidis, and M.~Ma.
\newblock Polynomial-time robust multiclass linear classification under gaussian marginals.
\newblock {\em arXiv preprint arXiv:2605.21428}, 2026.

\bibitem[DKLZ24]{DKLZ24}
I.~Diakonikolas, D.~M. Kane, S.~Liu, and N.~Zarifis.
\newblock Testable learning of general halfspaces with adversarial label noise.
\newblock In {\em The Thirty Seventh Annual Conference on Learning Theory}, volume 247 of {\em Proceedings of Machine Learning Research}, pages 1308--1335. {PMLR}, 2024.

\bibitem[DKM24]{diakonikolas2024active}
I.~Diakonikolas, D.~Kane, and M.~Ma.
\newblock Active learning of general halfspaces: Label queries vs membership queries.
\newblock {\em Advances in Neural Information Processing Systems}, 37:49180--49213, 2024.

\bibitem[DKPZ21]{DKPZ21}
I.~Diakonikolas, D.~M. Kane, T.~Pittas, and N.~Zarifis.
\newblock The optimality of polynomial regression for agnostic learning under gaussian marginals in the {SQ} model.
\newblock In {\em Proceedings of The 34\textsuperscript{th} Conference on Learning Theory, {COLT}}, 2021.

\bibitem[DKR23]{diakonikolas2023near}
I.~Diakonikolas, D.~Kane, and L.~Ren.
\newblock Near-optimal cryptographic hardness of agnostically learning halfspaces and relu regression under gaussian marginals.
\newblock In {\em International Conference on Machine Learning}, pages 7922--7938. PMLR, 2023.

\bibitem[DKS18]{DKS18a}
I.~Diakonikolas, D.~M. Kane, and A.~Stewart.
\newblock Learning geometric concepts with nasty noise.
\newblock In {\em Proceedings of the 50\textsuperscript{th} Annual {ACM} {SIGACT} Symposium on Theory of Computing, {STOC} 2018}, pages 1061--1073, 2018.

\bibitem[DKTZ22]{diakonikolas2022learning}
I.~Diakonikolas, V.~Kontonis, C.~Tzamos, and N.~Zarifis.
\newblock Learning general halfspaces with adversarial label noise via online gradient descent.
\newblock In {\em International Conference on Machine Learning}, pages 5118--5141. PMLR, 2022.

\bibitem[DMRT25a]{DiakonikolasMRT25}
I.~Diakonikolas, M.~Ma, L.~Ren, and C.~Tzamos.
\newblock Learning intersections of two margin halfspaces under factorizable distributions.
\newblock In N.~Haghtalab and A.~Moitra, editors, {\em The Thirty Eighth Annual Conference on Learning Theory, 30-4 July 2025, Lyon, France}, Proceedings of Machine Learning Research, pages 1472--1530. {PMLR}, 2025.

\bibitem[DMRT25b]{diakonikolas2025statistical}
I.~Diakonikolas, M.~Ma, L.~Ren, and C.~Tzamos.
\newblock Statistical query hardness of multiclass linear classification with random classification noise.
\newblock In {\em International Conference on Machine Learning}, pages 13693--13712. PMLR, 2025.

\bibitem[Eld13]{eldan2013thin}
R.~Eldan.
\newblock Thin shell implies spectral gap up to polylog via a stochastic localization scheme.
\newblock {\em Geometric and Functional Analysis}, 23(2):532--569, 2013.

\bibitem[Fel10]{Feldman:10ab}
V.~Feldman.
\newblock Distribution-specific agnostic boosting.
\newblock In {\em Proceedings of Innovations in Computer Science}, pages 241--250, 2010.

\bibitem[FJS81]{Friedman:1980tu}
J.~H. Friedman, M.~Jacobson, and W.~Stuetzle.
\newblock {Projection Pursuit Regression}.
\newblock {\em J. Am. Statist. Assoc.}, 76:817, 1981.

\bibitem[FS97]{FreundSchapire:97}
Y.~Freund and R.~Schapire.
\newblock A decision-theoretic generalization of on-line learning and an application to boosting.
\newblock {\em Journal of Computer and System Sciences}, 55(1):119--139, 1997.

\bibitem[Gav03]{gavinsky2003optimally}
D.~Gavinsky.
\newblock Optimally-smooth adaptive boosting and application to agnostic learning.
\newblock {\em Journal of Machine Learning Research}, 4(May):101--117, 2003.

\bibitem[GWBL20]{GWBL20}
S.~Garg, Y.~Wu, S.~Balakrishnan, and Z.~Lipton.
\newblock A unified view of label shift estimation.
\newblock In {\em Advances in Neural Information Processing Systems}, volume~33, pages 3290--3300. Curran Associates, Inc., 2020.

\bibitem[HL93]{HL93}
P.~Hall and K.-C. Li.
\newblock {On almost Linearity of Low Dimensional Projections from High Dimensional Data}.
\newblock {\em The Annals of Statistics}, 21(2):867 -- 889, 1993.

\bibitem[HL13]{hillar2013most}
C.~J. Hillar and L.-H. Lim.
\newblock Most tensor problems are np-hard.
\newblock {\em Journal of the ACM (JACM)}, 60(6):1--39, 2013.

\bibitem[Hub85]{Huber85-pp}
P.~J. Huber.
\newblock {Projection Pursuit}.
\newblock {\em The Annals of Statistics}, 13(2):435 -- 475, 1985.

\bibitem[Kea98]{Kearns:98}
M.~J. Kearns.
\newblock Efficient noise-tolerant learning from statistical queries.
\newblock {\em Journal of the ACM}, 45(6):983--1006, 1998.

\bibitem[KK09]{KalaiK09}
A.~Kalai and V.~Kanade.
\newblock Potential-based agnostic boosting.
\newblock In {\em Advances in Neural Information Processing Systems 22: 23\textsuperscript{rd} Annual Conference on Neural Information Processing Systems 2009}, pages 880--888, 2009.

\bibitem[KKM13]{kane2013learning}
D.~Kane, A.~Klivans, and R.~Meka.
\newblock Learning halfspaces under log-concave densities: Polynomial approximations and moment matching.
\newblock In {\em Conference on Learning Theory}, pages 522--545. PMLR, 2013.

\bibitem[KKMS08]{KKMS:08}
A.~Kalai, A.~Klivans, Y.~Mansour, and R.~Servedio.
\newblock Agnostically learning halfspaces.
\newblock {\em SIAM Journal on Computing}, 37(6):1777--1805, 2008.

\bibitem[KL25]{klartag2025thin}
B.~Klartag and J.~Lehec.
\newblock Thin-shell bounds via parallel coupling.
\newblock {\em arXiv preprint arXiv:2507.15495}, 2025.

\bibitem[KLT09]{KlivansLT09}
A.~R. Klivans, P.~M. Long, and A.~K. Tang.
\newblock Baum's algorithm learns intersections of halfspaces with respect to log-concave distributions.
\newblock In {\em 13\textsuperscript{th} International Workshop, {RANDOM} 2009}, pages 588--600, 2009.

\bibitem[KOS04]{KOS:04}
A.~Klivans, R.~O'Donnell, and R.~Servedio.
\newblock Learning intersections and thresholds of halfspaces.
\newblock {\em Journal of Computer \& System Sciences}, 68(4):808--840, 2004.

\bibitem[KS04]{KlivansServedio:04coltmargin}
A.~Klivans and R.~Servedio.
\newblock Learning intersections of halfspaces with a margin.
\newblock In {\em Proceedings of the 17\textsuperscript{th} Annual Conference on Learning Theory,}, pages 348--362, 2004.

\bibitem[KS08]{KS:08}
S.~Khot and R.~Saket.
\newblock On hardness of learning intersection of two halfspaces.
\newblock In {\em Proc.\ 40th Annual ACM Symposium on Theory of Computing (STOC)}, pages 345--354, 2008.

\bibitem[Lat11]{latalaa2011order}
R.~Lata{\l}aa.
\newblock Order statistics and concentration of lr norms for log-concave vectors.
\newblock {\em Journal of Functional Analysis}, 261:681--696, 2011.

\bibitem[Li91]{Li91}
K.-C. Li.
\newblock Sliced inverse regression for dimension reduction.
\newblock {\em Journal of the American Statistical Association}, 86(414):316--327, 1991.

\bibitem[Lot15]{lotz2015volume}
M.~Lotz.
\newblock On the volume of tubular neighborhoods of real algebraic varieties.
\newblock {\em Proceedings of the American Mathematical Society}, 143(5):1875--1889, 2015.

\bibitem[LV07]{lovasz2007geometry}
L.~Lov{\'a}sz and S.~Vempala.
\newblock The geometry of logconcave functions and sampling algorithms.
\newblock {\em Random Structures \& Algorithms}, 30(3):307--358, 2007.

\bibitem[LV18]{lee2018kannan}
Y.~T. Lee and S.~S. Vempala.
\newblock The kannan-lov$\backslash$'asz-simonovits conjecture.
\newblock {\em arXiv preprint arXiv:1807.03465}, 2018.

\bibitem[LWS18]{Lipton18a}
Z.~Lipton, Y.-X. Wang, and A.~Smola.
\newblock Detecting and correcting for label shift with black box predictors.
\newblock In {\em Proceedings of the 35th International Conference on Machine Learning}, volume~80 of {\em Proceedings of Machine Learning Research}, pages 3122--3130. PMLR, 10--15 Jul 2018.

\bibitem[Mil09]{milman2009role}
E.~Milman.
\newblock On the role of convexity in isoperimetry, spectral gap and concentration.
\newblock {\em Inventiones mathematicae}, 177(1):1--43, 2009.

\bibitem[MP68]{MinskyPapert:68}
M.~Minsky and S.~Papert.
\newblock {\em {P}erceptrons: an introduction to computational geometry}.
\newblock MIT Press, Cambridge, MA, 1968.

\bibitem[MT94]{MT:94}
W.~Maass and G.~Turan.
\newblock How fast can a threshold gate learn?
\newblock In S.~Hanson, G.~Drastal, and R.~Rivest, editors, {\em Computational Learning Theory and Natural Learning Systems}, pages 381--414. MIT Press, 1994.

\bibitem[NDRT13]{NDRT13}
N.~Natarajan, I.~S. Dhillon, P.~K. Ravikumar, and A.~Tewari.
\newblock Learning with noisy labels.
\newblock In C.~Burges, L.~Bottou, M.~Welling, Z.~Ghahramani, and K.~Weinberger, editors, {\em Advances in Neural Information Processing Systems}, volume~26, 2013.

\bibitem[Nov62]{Novikoff:62}
A.~Novikoff.
\newblock On convergence proofs on perceptrons.
\newblock In {\em Proceedings of the Symposium on Mathematical Theory of Automata}, volume XII, pages 615--622, 1962.

\bibitem[Pao12]{paouris2012small}
G.~Paouris.
\newblock Small ball probability estimates for log-concave measures.
\newblock {\em Transactions of the American Mathematical Society}, 364(1):287--308, 2012.

\bibitem[Ros58]{Rosenblatt:58}
F.~Rosenblatt.
\newblock The {P}erceptron: a probabilistic model for information storage and organization in the brain.
\newblock {\em Psychological Review}, 65:386--407, 1958.

\bibitem[RV23]{rubinfeld2023testing}
R.~Rubinfeld and A.~Vasilyan.
\newblock Testing distributional assumptions of learning algorithms.
\newblock In {\em Proceedings of the 55th Annual ACM Symposium on Theory of Computing}, pages 1643--1656, 2023.

\bibitem[SB14]{SB-book}
S.~Shalev{-}Shwartz and S.~Ben{-}David.
\newblock {\em Understanding Machine Learning - From Theory to Algorithms}.
\newblock Cambridge University Press, 2014.

\bibitem[Ser03]{Servedio:03jmlr}
R.~Servedio.
\newblock Smooth boosting and learning with malicious noise.
\newblock {\em Journal of Machine Learning Research}, 4:633--648, 2003.

\bibitem[Sha04]{shadrin2004twelve}
A.~Shadrin.
\newblock Twelve proofs of the markov inequality.
\newblock {\em Approximation theory: a volume dedicated to Borislav Bojanov}, pages 233--298, 2004.

\bibitem[She09]{Sherstov:09}
A.~Sherstov.
\newblock {The intersection of two halfspaces has high threshold degree}.
\newblock In {\em Proc.\ 50th IEEE Symposium on Foundations of Computer Science (FOCS)}, 2009.

\bibitem[SKP{\etalchar{+}}20]{Song-etal-survey20}
H.~Song, M.~Kim, D.~Park, Y.~Shin, and J.-G. Lee.
\newblock Learning from noisy labels with deep neural networks: A survey.
\newblock {\em IEEE Transactions on Neural Networks and Learning Systems}, pages 1--29, 2020.

\bibitem[STC00]{CristianiniShaweTaylor:00}
J.~Shawe-Taylor and N.~Cristianini.
\newblock {\em An introduction to support vector machines}.
\newblock Cambridge University Press, 2000.

\bibitem[Tie23]{Tiegel23}
S.~Tiegel.
\newblock Hardness of agnostically learning halfspaces from worst-case lattice problems.
\newblock In {\em The Thirty Sixth Annual Conference on Learning Theory, {COLT} 2023}, volume 195 of {\em Proceedings of Machine Learning Research}, pages 3029--3064. {PMLR}, 2023.

\bibitem[Val84]{Valiant:84}
L.~Valiant.
\newblock A theory of the learnable.
\newblock {\em Communications of the ACM}, 27(11):1134--1142, 1984.

\bibitem[Vap98]{Vapnik:98}
V.~Vapnik.
\newblock {\em Statistical Learning Theory}.
\newblock Wiley-Interscience, New York, 1998.

\bibitem[Vem10a]{Vempala10a}
S.~Vempala.
\newblock Learning convex concepts from gaussian distributions with {PCA}.
\newblock In {\em 51\textsuperscript{th} Annual {IEEE} Symposium on Foundations of Computer Science, {FOCS}}, pages 124--130, 2010.

\bibitem[Vem10b]{Vempala10}
S.~Vempala.
\newblock A random-sampling-based algorithm for learning intersections of halfspaces.
\newblock {\em J. {ACM}}, 57(6):32:1--32:14, 2010.

\bibitem[WLT18]{WLT18}
R.~Wang, T.~Liu, and D.~Tao.
\newblock Multiclass learning with partially corrupted labels.
\newblock {\em IEEE Transactions on Neural Networks and Learning Systems}, 29(6):2568--2580, 2018.

\bibitem[Xia08]{Xia08}
Y.~Xia.
\newblock A multiple-index model and dimension reduction.
\newblock {\em Journal of the American Statistical Association}, 103(484):1631--1640, 2008.

\bibitem[XTLZ02]{xia2002adaptive}
Y.~Xia, H.~Tong, W.~K. Li, and L.~Zhu.
\newblock An adaptive estimation of dimension reduction space.
\newblock {\em Journal of the Royal Statistical Society Series B: Statistical Methodology}, 64(3):363--410, 2002.

\end{thebibliography}
\bibliographystyle{halpha-abbrv}

\newpage

\appendix

\section*{Appendix}

\section{Probability Lemmas}\label{app probability}


\begin{fact}[Theorem 5.14 in \cite{lovasz2007geometry}]\label{fact density}
    Let $f(\x)$ be the density function of an isotropic logconcave distribution in $\R^d$. For every $\x$, with $\norm{\x} \le 1/9$, it holds $f(\x) \ge 2^{-9d\norm{\x}} f(0)$ and $f(x) \ge 2^{-7d}$.
\end{fact}


\begin{fact}[Lemma 5.7 in \cite{lovasz2007geometry}]\label{fact tail logconcave}
    Let $\D_{\x}$ be an isotropic logconcave distribution over $\R^d$ and let $\w \in \s^{d-1}$ be a unit vector. There is a constant $C>0$ such that for every $t>0$, $\Pr_{\x \sim D_\x}[\abs{\w \cdot \x}>t]<C e^{-t}.$
\end{fact}

\begin{fact}[\cite{paouris2012small}]\label{fact small-ball logconcave}
    Let $\D_{\x}$ be an isotropic logconcave distribution over $\R^d$, then for fixed constant $\eps \in (0,1)$
    \begin{align*}
        \Pr_{\x \sim D_\x}\left[ \norm{\x}_2 < \eps \sqrt{d}\right] \le \exp\left(-\Omega(d)\right).
    \end{align*}
\end{fact}

\begin{fact}[Theorem 13 in \cite{lee2018kannan}]\label{fact lipschitz logconcave}
    Let $\D_{\x}$ be an isotropic logconcave distribution over $\R^d$ and let $g: \R^d \to \R$ be an $L$-Lipschitz function, then for $t>0$,
    \begin{align*}
        \Pr_{\x \sim D_\x}\left[ \abs{ g(\x) - \E_{\x \sim D_\x}[g(\x)] }>Lt\right] \le \exp\left(-\Omega(\psi_D t)\right),
    \end{align*}
    where $\psi_D$ is the Cheeger constant of $\D_\x$.
\end{fact}
As an implication of Fact~\ref{fact lipschitz logconcave}, we have the following tail bound for the $L_p$ norm of a random vector drawn from an isotropic logconcave distribution.

\begin{fact}[Theorem 8 in \cite{latalaa2011order}]\label{fact norm logconcave}
    Let $\D_\x$ be an isotropic logconcave distribution over $\R^d$. For $p \ge 2$ and $t>1$, it holds
    \begin{align*}
        \Pr_{\x \sim D_\x}\left[\norm{\x}_p>C(p)td^{1/p}\right] \le \exp(-td^{1/p}),
    \end{align*}
    where $C(p)$ is a universal constant that only depends on $p$.
\end{fact}

\begin{fact}[Theorem 1.1 in \cite{arutyunyan2018deviation}]
    Let $\D_\x$ be an isotropic logconcave distribution over $\R^d$ and $p(\x): \R^d \to \R$ be a polynomial with degree at most $k$. Then, for $t>1$, it holds 
    \begin{align*}
        \Pr_{\x \sim D_\x}[\abs{p(\x)}>\sqrt{e}(4t)^k\norm{p}_2] \le e^{-t}.
    \end{align*}
\end{fact}

\begin{fact}[Theorem 8 in \cite{carbery2001distributional}]\label{fact small ball polynomial}
   There is a universal constant $C>0$ such that for $\D_\x$, an isotropic logconcave distribution over $\R^d$ and for $p(\x):\R^d \to \R$, a polynomial with degree $k$ and for $\eps>0$, it holds
    \begin{align*}
        \Pr_{\x \sim D_\x}\left[\abs{p(\x)}<\eps\right] \le Ck \frac{\eps^{1/k}}{\E_{\x \sim D_\x}[p^2(\x)]^{1/2k}}.
    \end{align*}
\end{fact}

\begin{fact}\label{fact error logconcave}
    Let $D$ be a distribution on $\R^d \times \{0,\pm 1\}$, such that $D_\x$ is an isotropic logconcave distribution. There is a constant $C$ such that for $S(\x,y)$ an event over $\R^d \times \{0,\pm 1\}$ and $\w \in \s^{d-1}$ a unit vector, it holds
\begin{align*}
    \E_{(\x,y)\sim D}\left[\abs{\w\cdot \x} \mathbf{1}\{S(\x,y)\} \right]\le C \Pr_{(\x,y)\sim D}[S(\x,y)] \log\left(\frac{e}{\Pr_{(\x,y)\sim D}[S(\x,y)]}\right).
\end{align*}
\end{fact}

\begin{proof}[Proof of Fact~\ref{fact error logconcave}]
    Let $Z:= \w\cdot \x$. Then, $Z$ is also an isotropic logconcave distribution. For $z \in \R$, we denote by $z_+:= \max \{z,0\}$. Notice that for every $z,s \in \R$, we have $\abs{z} \le s + (\abs{z}-s)_+$. This implies for every $s \in \R$, we have 
    \begin{align*}
     \E_{(\x,y)\sim D}\left[\abs{\w\cdot \x} \mathbf{1}\{S(\x,y)\} \right] &\le \Pr_{(\x,y)\sim D}[S(\x,y)]s + \E_{(\x,y)\sim D}[(\abs{Z}-s)_+]   \\
     & = \Pr_{(\x,y)\sim D}[S(\x,y)]s + \int_s^\infty \Pr_{Z}[\abs{Z}>r]dr \\
     & \le \Pr_{(\x,y)\sim D}[S(\x,y)]s + Ce^{-s},
    \end{align*}
    where the last inequality follows Fact~\ref{fact tail logconcave}. By taking $s= \log(C/\Pr_{(\x,y)\sim D}[S(\x,y)])$, we get 
    \begin{align*}
    \E_{(\x,y)\sim D}\left[\abs{\w\cdot \x} \mathbf{1}\{S(\x,y)\} \right]\le C \Pr_{(\x,y)\sim D}[S(\x,y)] \log\left(\frac{e}{\Pr_{(\x,y)\sim D}[S(\x,y)]}\right).
\end{align*}
\end{proof}

\begin{fact}\label{fact moment logconcave}
Let $D$ be an isotropic logconcave distribution over $\R^d$ and $\u^{(1)},\dots,\u^{(k)}$ be $k$ unit vectors in $\R^d$. Then, 
    \begin{align*}
        \abs{\E_{\x \sim D}[\prod_{i=1}^k(\u^{(i)}\cdot \x)]} \le (2k)^k. 
    \end{align*}
\end{fact}

\begin{proof}[Proof of Fact~\ref{fact moment logconcave}]
By H\"{o}lder's inequality and Lemma 33 in \cite{lee2018kannan}, we have
    \begin{align*}
    \abs{\E_{\x\sim D_\x} [\prod_{i = 1}^k(\u^{(i)}\cdot\x)]} \le \prod_{i=1}^k\left(\E_{\x\sim D_\x}[\abs{\u^{(i)}\cdot\x}^{k}]^{1/k}\right) \le \prod_{\ell \not\in S} 2k = \left(2k\right)^{k}.
\end{align*}
\end{proof}

\begin{fact}[Theorem 7 in \cite{carbery2001distributional}]\label{fact polynomial growth}
There is a constant $C>0$ such that for every logconcave distribution $D$ and for every polynomial $p(\x)$ with degree $k$, $0<r<q<\infty$
    \begin{align*}
        (\E_{\x \sim D}[\abs{p(\x)}^{q/k}])^{1/q} \le C \frac{\max(q,1)}{\max(r,1)} (\E_{\x \sim D}[\abs{p(\x)}^{r/k}])^{1/r}.
    \end{align*}
    In particular $\E_{\x \sim D}[p(\x)^{\ell}] \le (C\ell)^{\ell k} \E_{\x \sim D}[p(\x)^{2}]^{\ell/2}$, $\forall \ell > 2$, $\E_{\x \sim D}[p(\x)^{\ell}] \le (C\ell)^{\ell k} \E_{\x \sim D}[p(\x)]^{\ell}$, $\forall \ell > 1$
\end{fact}

\begin{fact}\label{fact coefficient}
    Let $D$ be a $1/2$-isotropic logconcave distribution over $\R$ and let $p(x)=\sum_{0}^\ell c_i x^i: \R \to \R$ be a degree $\ell$ polynomial, then 
    \begin{align*}
        \sum_{i=0}^\ell \abs{c_i} \le (C\ell)^\ell \E_{x\sim D}[\abs{p(x)}].
    \end{align*}
\end{fact}

\begin{proof}[Proof of Fact~\ref{fact coefficient}]
Let $\mu = \E_{x \sim D}$ and let $D(x)$ be the density function of $D$.
By Theorem 5.14 in \cite{lovasz2007geometry}, we know that $D(\mu)>2^{-8}$ and for every $x = \mu + z$, $\abs{z}<1/9$, we have $D(x)>2^{-9}$.
Let $I = [\mu-1/9,\mu+1/9]$. This implies that 
\begin{align*}
    \int_{I}\abs{p(x)}dx \le 2^9 \int_{I}\abs{p(x)}D(x) dx \le 2^9 \E_{x\sim D}[\abs{p(x)}].
\end{align*}
We center $I$, by taking $y = x - \mu$, and $I' = [-1/9,1/9]$.
Then, we have 
\begin{align*}
    \int_{I'}\abs{p(y+\mu)}dy = \int_{I}\abs{p(x)}dx \le 2^9 \E_{x\sim D}[\abs{p(x)}].
\end{align*}
Notice that $c_j = p^{(j)}(0)/(j!)$. By Markov's inequality \cite{shadrin2004twelve}, we have $\abs{p^{(j)}(0)} \le (Ck)^k \max_{[-1,1]}\abs{p(x)} \le (C'k)^k\int_{[-1,1]}\abs{p(x)}dx$. This gives
\begin{align*}
        \sum_{i=0}^\ell \abs{c_i} \le (C\ell)^\ell \E_{x\sim D}[\abs{p(x)}] \;.
    \end{align*}
\end{proof}


\begin{fact}\label{fact log concave l1 growth}

Let $D_\x$ be a $1/2$ isotropic logconcave distribution on $\mathbb{R}^d$. 
Let $p(\x)=\sum_{\card{\alpha}\le k} c_\alpha \x^\alpha$
be a polynomial of degree at most $k$, where each $\alpha$ is a multi index. 
Then
\begin{align*}
\sum_{\abs{\alpha} 
\le k} 
\left( \abs{c_\alpha}^2\E_{\x \sim D_\x}
[\x^{2\alpha}]\right)^{1/2}
\le
C_{d,k}\,\E_{\x \sim D_\x} [\abs{p(\x)}],    
\end{align*}
where $C_{d,k} \le 
\binom{d+k}{k}\,(Ck)^{Cdk}\,d^{k/2}\,2^{7d}.$
\end{fact}

\begin{proof}[Proof of Fact~\ref{fact log concave l1 growth}]
The case $k=0$ is immediate, so assume $k\ge 1$.
We first assume that $D_\x$ has mean zero.
Its covariance matrix $\bm\Sigma$ satisfies
\[
\frac12 I\preceq\bm\Sigma\preceq\frac32 I.
\]
Let $f$ denote the density of $D_\x$, and let
$M_f=\sup_{\x\in\R^d}f(\x)$.
By Theorem~5.14(a),(b) of Lov\'asz--Vempala, the source of
Fact~\ref{fact density}, applied to
$\bm\Sigma^{-1/2}\x$, we have
\[
f(\v)\ge 2^{-d}f(0)
\quad\text{whenever }\norm{\v}_2\le 1/18,
\qquad
M_f\le 2^{2d+4}f(0).
\]
Let $r=d^{-1/2}/18$ and $Q_r=[-r,r]^d$.
Thus, for every $\v\in Q_r$,
$
\frac{M_f}{f(\v)}\le 2^{3d+4}.
$

We first show that
\begin{equation}\label{eq:local-polynomial-bound}
\sup_{\v\in Q_r}|p(\v)|
\le
2^{5d}(4e)^k
\E_{\x\sim D_\x}[|p(\x)|].
\end{equation}
Fix $\v\in Q_r$ and $t\in[1/2,1]$.
By logconcavity,
\[
f((1-t)\v+t\x)
\ge
f(\v)^{1-t}f(\x)^t
\ge
\left(\frac{f(\v)}{M_f}\right)^{1-t}f(\x).
\]
Consequently, the density of $(1-t)\v+t\x$, for
$\x\sim D_\x$, is pointwise at most
\[
t^{-d}\left(\frac{M_f}{f(\v)}\right)^{1-t}f
\le
2^{(5/2)d+2}f
\le
2^{5d}f.
\]
It follows that
\begin{equation}\label{eq:contracted-polynomial-bound}
\E_{\x\sim D_\x}
\bigl[|p((1-t)\v+t\x)|\bigr]
\le
2^{5d}\E_{\x\sim D_\x}[|p(\x)|].
\end{equation}

Let $t_j=1/2+j/(2k)$, for $j=0,\ldots,k$.
Lagrange interpolation at these points gives, for every
univariate polynomial $P$ of degree at most $k$,
\[
P(0)=\sum_{j=0}^k\lambda_jP(t_j),
\qquad
\lambda_j=
\prod_{\substack{0\le i\le k\\i\ne j}}
\frac{-t_i}{t_j-t_i}.
\]
Moreover,
\[
\sum_{j=0}^k|\lambda_j|
\le
(2k)^k\sum_{j=0}^k\frac{1}{j!(k-j)!}
=
\frac{(4k)^k}{k!}
\le
(4e)^k.
\]
Applying this identity to
$P(t)=p((1-t)\v+t\x)$, taking expectations over
$\x\sim D_\x$, and using
\eqref{eq:contracted-polynomial-bound}, we obtain
\[
|p(\v)|
\le
\sum_{j=0}^k|\lambda_j|
\E_{\x\sim D_\x}
\bigl[|p((1-t_j)\v+t_j\x)|\bigr]
\le
2^{5d}(4e)^k\E_{\x\sim D_\x}[|p(\x)|].
\]
This proves \eqref{eq:local-polynomial-bound}.

We next rescale $Q_r$ to $[-1,1]^d$.
For $\u\in\R^d$, define $q(\u)=p(r\u)$.
Then
$
q(\u)=
\sum_{|\alpha|\le k}c_\alpha r^{|\alpha|}\u^\alpha,
$
and \eqref{eq:local-polynomial-bound} implies
\begin{equation}\label{eq:rescaled-polynomial-integral}
\int_{[-1,1]^d}|q(\u)|\,d\u
\le
2^d\sup_{\v\in Q_r}|p(\v)|
\le
2^{6d}(4e)^k
\E_{\x\sim D_\x}[|p(\x)|].
\end{equation}

Let $M\in\R^{(k+1)\times(k+1)}$ be the matrix
\[
M_{ij}=\int_{-1}^1 t^{i+j}\,dt,
\qquad i,j=0,\ldots,k.
\]
This matrix is positive definite, hence invertible.
For $i=0,\ldots,k$, define
$
\psi_i(t)=\sum_{j=0}^k(M^{-1})_{ij}t^j.
$
Then
\[
\int_{-1}^1\psi_i(t)t^j\,dt
=
\sum_{m=0}^k(M^{-1})_{im}M_{mj}
=
\delta_{ij}.
\]
Also,
$
\|\psi_i\|_\infty\le (C_1k)^{C_1k},
$
where the norm is over $[-1,1]$ and $C_1>0$ is universal.
For completeness, if $P_j$ is the $j$th Legendre polynomial
normalized by $P_j(1)=1$, then
\[
\psi_i(t)
=
\sum_{j=0}^k
\frac{2j+1}{2}\bigl([s^i]P_j(s)\bigr)P_j(t).
\]
The bounds $\|P_j\|_\infty\le1$ and
$|[s^i]P_j(s)|\le4^j$ give the claimed estimate.

For a multi-index $\alpha=(\alpha_1,\ldots,\alpha_d)$,
define
$
\Psi_\alpha(\u)=\prod_{m=1}^d\psi_{\alpha_m}(u_m).
$
For $|\alpha|,|\beta|\le k$, we have
\[
\int_{[-1,1]^d}\u^\beta\Psi_\alpha(\u)\,d\u
=
\prod_{m=1}^d
\int_{-1}^1u_m^{\beta_m}\psi_{\alpha_m}(u_m)\,du_m
=
\delta_{\alpha\beta}.
\]
Therefore,
$
c_\alpha r^{|\alpha|}
=
\int_{[-1,1]^d}q(\u)\Psi_\alpha(\u)\,d\u.
$
Since
$\|\Psi_\alpha\|_\infty\le(C_1k)^{C_1dk}$,
\eqref{eq:rescaled-polynomial-integral} gives
\[
|c_\alpha|r^{|\alpha|}
\le
(C_1k)^{C_1dk}2^{6d}(4e)^k
\E_{\x\sim D_\x}[|p(\x)|].
\]
There are $N=\binom{d+k}{k}$ such multi-indices.
Since $r<1$ and $|\alpha|\le k$, it follows that
\begin{align}
\sum_{|\alpha|\le k}|c_\alpha|
&\le
r^{-k}\sum_{|\alpha|\le k}|c_\alpha|r^{|\alpha|}
\nonumber\\
&\le
N(C_1k)^{C_1dk}2^{6d}(4e)^k r^{-k}
\E_{\x\sim D_\x}[|p(\x)|]
\nonumber\\
&\le
N(C_2k)^{C_2dk}d^{k/2}2^{6d}
\E_{\x\sim D_\x}[|p(\x)|],
\label{eq:centered-coefficient-bound}
\end{align}
where $C_2>0$ is universal.

To conclude the proof, let $D_\x$ have an arbitrary mean
$\bm\mu$ satisfying $\|\bm\mu\|_2\le1/2$.
Set $\bz=\x-\bm\mu$, and write
$
\widetilde p(\bz)=p(\bz+\bm\mu)
=
\sum_{|\beta|\le k}a_\beta\bz^\beta.
$
The distribution of $\bz$ is mean-zero and logconcave,
with covariance between $I/2$ and $3I/2$.
Applying \eqref{eq:centered-coefficient-bound} gives
\[
\sum_{|\beta|\le k}|a_\beta|
\le
N(C_2k)^{C_2dk}d^{k/2}2^{6d}
\E_{\x\sim D_\x}[|p(\x)|].
\]
Since
$
p(\x)=
\sum_{|\beta|\le k}a_\beta(\x-\bm\mu)^\beta,
$
the binomial theorem implies
\begin{align*}
\sum_{|\alpha|\le k}|c_\alpha|
\le
\sum_{|\beta|\le k}|a_\beta|
\prod_{j=1}^d(1+|\mu_j|)^{\beta_j}\le
(3/2)^k\sum_{|\beta|\le k}|a_\beta|.
\end{align*}

Finally, Fact~\ref{fact moment logconcave}, applied after
centering and scaling each one-dimensional marginal,
together with Holder's inequality, gives
$
\left(\E_{\x\sim D_\x}[\x^{2\alpha}]\right)^{1/2}
\le (C_3k)^k, |\alpha|\le k,
$
for a universal constant $C_3>0$.
Combining the preceding bounds, we obtain
\begin{align*}
\sum_{|\alpha|\le k}
\left(
|c_\alpha|^2
\E_{\x\sim D_\x}[\x^{2\alpha}]
\right)^{1/2}
\le
(C_3k)^k\sum_{|\alpha|\le k}|c_\alpha|\le
\binom{d+k}{k}
(Ck)^{Cdk}d^{k/2}2^{7d}
\E_{\x\sim D_\x}[|p(\x)|],
\end{align*}
for a sufficiently large universal constant $C>1$.
This concludes the proof.
\end{proof}

\section{Omitted Details from Section~\ref{sec boosting}}\label{app boosting}
In this section, we supply the details omitted from
Section~\ref{sec boosting}. We prove the following consequence of
Theorem~3.5 in~\cite{Feldman:10ab}, making explicit the smoothness
property used in Fact~\ref{th distribution specific boosting}.

\begin{fact}[Corollary of Theorem~3.5 in~\cite{Feldman:10ab}]
Let $0\le a<1/2$ and $0<\beta\le 1/2$, and let
$\H\subseteq\{\pm1\}^{\R^d}$. Suppose that $\A$ is an algorithm
such that, for every distribution $D$ over $\R^d\times\{\pm1\}$
satisfying $\opt_D(\H)\le 1/2-a$, it draws $m$ examples, runs in
time $T$, and outputs a Boolean hypothesis $g$ satisfying
$\err_D(g)\le 1/2-\beta$ with probability at least $99\%$.
Here $\opt_D(\H):=\inf_{h\in\H}\err_D(h)$.

Then there is an algorithm $\mathcal B$ that, for every distribution
$D$ over $\R^d\times\{\pm1\}$ and every $\eps>0$, makes
$\widetilde O(\beta^{-2}/\eps)$ calls to $\A$, runs in time
$\poly(T,\beta^{-1},\eps^{-1})$, and outputs a Boolean hypothesis
$\hat h$ with the following simultaneous guarantees, with
probability at least $99\%$:
\[
  \err_D(\hat h)\le \frac{\opt_D(\H)}{1-2a}+\eps,
\]
and every call to $\A$ is on a distribution $D'$ obtained by
accepting an example $(\x,y)\sim D$ with an explicit probability
$q(\x,y)\in[0,1]$ satisfying $\E_D[q]\ge\eps$.
In particular,
\[
  D'_\x(d\x)=w(\x)D_\x(d\x)
  \qquad\text{for some }w:\R^d\to[0,1/\eps].
\]
\end{fact}

\begin{proof}[Proof of Fact~\ref{th distribution specific boosting}]
We will be consistent with the notations used in \cite{Feldman:10ab}.
We may assume $0<\eps<1$, and fix $\sign(0)=1$.
For $a\in\R$, define $P_1(a)=a$ if $\abs{a}\le1$ and
$P_1(a)=\sign(a)$ otherwise. Furthermore, define $R(a)=a^2$ if
$\abs{a}\le1$ and $R(a)=2\abs{a}-1$ otherwise.
We list several properties of $P_1$ and $R$, which are proved in
\cite{Feldman:10ab}. For every Boolean function $f$ and distribution
$D$, it holds
\begin{enumerate}
    \item $\E_D[R(f)]=1$ and, for every real-valued function $\phi$,
    $\E_D[R(f-\phi)]\ge0$.

    \item For $h,g:\R^d\to[-1,1]$, if
    $\langle P_1(f-h),g\rangle_D\ge2\gamma$, then
    \[
        \E_D[R(f-(h+\gamma g))]
        \le\E_D[R(f-h)]-3\gamma^2.
    \]

    \item For every real-valued $h:\R^d\to\R$,
    $\E_D[R(f-P_1(h))]\le\E_D[R(f-h)]$.
\end{enumerate}

We describe the algorithm used in \cite{Feldman:10ab} here.
Denote by $f(\x)=\E_D[y\mid\x]$, for each $\x\in\R^d$.
We first prove the claim when $f$ is Boolean; the general case is
discussed below. We describe the algorithm with exact expectations
and condition on all promised weak-learning calls succeeding;
estimation and confidence amplification are discussed at the end.
We initialize $h_0\equiv0$. In each time step $t\ge0$, we maintain
a function $h_t:\R^d\to[-1,1]$ and, whenever $N_t>0$, define a
distribution $D_t$ as follows:
\[
    (D_t)_\x(\x)
    =D_\x(\x)\abs{P_1(f(\x)-h_t(\x))}/N_t,
    \qquad
    N_t=\norm{P_1(f-h_t)}_{D,1},
\]
and $D_t(y=f(\x)\mid\x)=1$.

Set $\Delta'=\eps$. Before calling $\A$, if
$
    \langle P_1(f-h_t),-\sign(h_t)\rangle_D>\eps/2,
$
we perform the balancing update
$
    h_{t+1}=P_1\bigl(h_t-(\eps/4)\sign(h_t)\bigr)
$
and proceed to the next time step.
Otherwise, if $\err_D(\sign(h_t))\le\Delta'$, we output
$\sign(h_t)$ and terminate.
If neither condition holds, we run $\A$ over $D_t$.
If a weak hypothesis $g_t:\R^d\to\{\pm1\}$ satisfying
$\err_{D_t}(g_t)\le1/2-\beta$ is returned, we update
$h_{t+1}=P_1(h_t+\beta N_tg_t)$.
If no such hypothesis is returned, we output $\sign(h_t)$
and terminate.

In particular, since
$N_t\ge\err_D(\sign(h_t))$, every call to $\A$ satisfies
$N_t>\Delta'=\eps$.
Notice that this implies there is a weight function
$w_t:\R^d\to[0,1/\eps]$ such that
$(D_t)_\x(\x)=D_\x(\x)w_t(\x)$.

For each time step $t$ at which a weak hypothesis is used, we have
\begin{align*}
    \E_D[P_1(f-h_t)g_t]
    =\E_D[\abs{P_1(f-h_t)}fg_t]
    =\E_{D_t}[fg_t]N_t
    \ge2\beta N_t.
\end{align*}
Since
$N_t=\E_D[\abs{P_1(f-h_t)}]\ge\E_D[R(f-h_t)]/3$,
this gives
\begin{align*}
    \E_D[R(f-h_{t+1})]
    &\le\E_D[R(f-h_t)]-3(\beta N_t)^2\\
    &\le\E_D[R(f-h_t)](1-\beta^2N_t)\\
    &\le\E_D[R(f-h_t)](1-\beta^2\Delta').
\end{align*}

Each balancing update also decreases the potential, by at least
$3\eps^2/16$, by the second and third properties above.
Thus there are $O(\eps^{-2})$ balancing updates, and these updates
make no calls to $\A$.
The displayed multiplicative decrease implies that after at most
$
    O\bigl(\beta^{-2}(\Delta')^{-1}\log(2/\Delta')\bigr)
$
weak-learner updates, the potential is at most $\Delta'$.
By the definition of $R$, this implies
$
    \err_D(\sign(h_t))
    \le\E_D[R(f-h_t)]\le\Delta',
$
so the algorithm terminates before another weak-learning call.

To conclude the proof, it remains to analyze termination when a
weak hypothesis is not returned. On the success event specified
above, this implies that, for every $h^*\in\H$,
\[
    \E_{D_t}[fh^*]<2a,
    \qquad\text{equivalently}\qquad
    \E_D[h^*P_1(f-h_t)]<2aN_t.
\]
This gives, for every $h^*\in\H$,
\begin{align*}
    N_t
    =\E_D[fP_1(f-h_t)]
    =\E_D[h^*P_1(f-h_t)]
      +\E_D[(f-h^*)P_1(f-h_t)]
    \le2aN_t+2\err_D(h^*).
\end{align*}
Taking the infimum over $h^*\in\H$, we obtain
$N_t\le2\opt_D(\H)/(1-2a)$.
Moreover, the balancing test ensures
$\langle P_1(f-h_t),-\sign(h_t)\rangle_D\le\eps/2$.
Consequently,
\begin{align*}
    \err_D(\sign(h_t))
    =\frac12\E_D[\abs{f-\sign(h_t)}]
    =\frac12\E_D[P_1(f-h_t)(f-\sign(h_t))]
\le\frac{N_t+\eps/2}{2}
    \le\frac{\opt_D(\H)}{1-2a}+\frac\eps4.
\end{align*}
Together with the stopping rule $\Delta'=\eps$, this proves the
desired error guarantee in both termination cases.

For a general distribution $D$, apply the same analysis on the
space of labeled pairs $(\x,y)$, with Boolean target $f(\x,y)=y$,
while every hypothesis still depends only on $\x$.
The weak learner receives $\x$ as its input and $y$ as its label.
The reweighted distribution is sampled by accepting $(\x,y)\sim D$
with probability $\abs{P_1(y-h_t(\x))}$.
Its feature marginal has weight
\[
    w_t(\x)
    =\frac{\E_D[\abs{P_1(y-h_t(\x))}\mid\x]}{N_t}
    \le\frac1{N_t}\le\frac1\eps,
    \qquad
    N_t=\E_D[\abs{P_1(y-h_t(\x))}].
\]
All preceding identities and potential inequalities remain valid.

Finally, the tests can be implemented with fresh estimates and
constant slack. Estimate the balancing correlation, the error,
and $N_t$ to accuracy $\eps/8$. Use an estimated balancing
threshold $3\eps/8$ with balancing step $\eps/8$, and stop when
the estimated error is at most $7\eps/8$.
After balancing, the true balancing correlation is at most
$\eps/2$. The identity
\[
    2\err_D(\sign(h_t))
    =N_t+\E_D[P_1(y-h_t(\x))(-\sign(h_t(\x)))]
\]
then ensures $N_t>\eps$ whenever the error test continues.
Estimate each weak hypothesis's correlation under $D_t$ to
accuracy $\beta/4$, accept it when the estimate is at least
$7\beta/4$, and use step size
$\beta(\widehat N_t-\eps/8)$.
Every promised successful output passes this test, and every
accepted output has true correlation at least $3\beta/2$.
Since $3N_t/4\le\widehat N_t-\eps/8\le N_t$ on a call,
these changes affect the potential decreases only by constant
factors. Repeating $\A$ independently a logarithmic number of
times per round, allocating failure probabilities to all fresh
estimates, and imposing sufficiently large polynomial caps on
the updates and rejection-sampling requests gives all the
claimed guarantees simultaneously with probability at least
$99\%$. Invocations of $\A$ may also be capped at their stated
time and sample bounds. The total number of calls to $\A$ is
$\widetilde O(\beta^{-2}/\eps)$ and the running time is
$\poly(T,\beta^{-1},\eps^{-1})$.
\end{proof}

\section{Omitted Details from Section~\ref{sec application}}\label{app application}

\subsection{Omitted Proof of Extended Theorem 6 in \cite{kane2013learning}} 

In this section, we give the proof of the following extension of 
Theorem 6 in \cite{kane2013learning}.

\begin{theorem}[Extension of Theorem 6 in \cite{kane2013learning}]\label{th extension}
    Let $D_\x$ be an isotropic logconcave distribution over $\R^d$ and let $f: \R^d \to \{0,1\}$. For every $\eps,\delta \in (0,1)$, there is a pair of polynomials $p^l,p^u$ of degree $\ell = \exp(O(d^4/\delta^4)+d^2\log^2(d/\eps))$, such that 
    \begin{align*}
        p^l(\x) \le f(\x)  \le p^u(\x), \  \E_{\x \sim D_\x}[|p^u(\x)- p^l(\x)|] \le \eps+ 2 \Pr_{\x\sim D_\x}[S_\delta], 
    \end{align*}
    where $S_\delta=\{\u \in \R^d \mid \exists \v \in \R^d: \norm{\u-\v}_2\le \delta, f(\u) \neq f(\v)\}$.
\end{theorem}

We will make use of the following facts proved in \cite{kane2013learning}. 

\begin{fact}\label{fact norm}
    Let $D_\x$ be an isotropic logconcave distribution over $\R^d$, for every $R>0$, $\Pr_{\x\sim D_\x}[\norm{\x}_2>R] \le 2^{O(d)}\exp(-R/32)$.
\end{fact}

\begin{fact}\label{fact taylor}
    Let $f(\x): \R^d \to \R$ be $c$-smooth, which means there is some $c>0$ such that  for 
    every $k$ and for every $\u \in \s^{d-1}$, $\nabla^k_\u f$,
    the $k$-th order derivative of $f$ in $\u$ direction satisfies $\norm{\nabla^k_\u f}_\infty \le c^k$. 
    Then for every positive $k,N$, there is a degree $k2^N$-polynomial $p$ such that $\abs{(f(\x))^N-p(\x)} \le E(\x) = \sum_{j=0}^{N-1} 2^N (2c\norm{\x}_2)^{2^jk}/(2^jk)!$.
    Furthermore, for any isotropic logconcave distribution $D_\x$, $\norm{E(\x)}_{D_\x,1}\le 2^N2^{O(d)}2^{-k}$.
\end{fact}

\begin{proof}[Proof of \Cref{th extension}]

We construct the polynomial $p^u$ and $p^l$ can be constructed by symmetry. Recall that $S_\delta=\{\u \in \R^d \mid \exists \v \in \R^d: \norm{\u-\v}_2\le \delta, f(\u) \neq f(\v)\}$.
    The construction of $p^u$ will depend on several parameters that will be determined later: $c>0$, a small constant, $k,N \in \Z_+$, large even integers, and $R\ge 1$ a suitably large number.
    
    For $c>0$, we define $\rho_c(\x) = \prod_{i=1}^d \text{sinc}(c\x_i/\sqrt{d})$, where $\text{sinc}(z) = \sin(z)/z$ for $z\neq 0$, and $\text{sinc}(0)=1$.
    By \cite{kane2013learning}, $\rho_c(\x)$ is $c$-smooth, and when $\norm{\x}_2 \le c^{-1}\sqrt{d}$, we have $\rho_c(\x) = \exp(-\Theta(c^2\norm{\x}^2_2/d))$.
    For an even integer $N \in \Z_+$, we define $r(\x) = C\rho_c(\x)^N$, where $C>0$ is a normalization constant such that $\int_{\R^d} r(\x)d\x = 1$. 
    By \cite{kane2013learning}, $C=O(c^dN^{d/2})$.
    Following the construction of \cite{kane2013learning}, we choose $c=1/256$.

    For $R>0$,
    we define $f^*(\x) = 1$ if $\norm{\x}_2<R+\delta/2$ and $\x \in S_{\delta/2} \cup \{\x \mid f(\x)=1\}$, and $f^*(\x) = 0$ otherwise.
    We define $s(\x):=r \ast f^*$:
    \begin{align*}
        s(\x) = \int_{\R^d} f^*(\u)r(\x-\u)d\u.
    \end{align*}
    We notice that $0\le s(\x)\le 1$ for every $\x$.
    For every $\x$ with $\norm{\x}_2 \le R$ and $f(\x) = 1$, we have 
    \begin{align*}
        s(\x)
        = \int_{\u: \norm{\u-\x}_2<\delta/2}
            f^*(\u)r(\x-\u)d\u
        +
        \int_{\u: \norm{\u-\x}_2\ge\delta/2}
            f^*(\u)r(\x-\u)d\u.
    \end{align*}
    Following \cite{kane2013learning}, we choose
    $N=\max\{O(d^4/\delta^4),O(d^2\log^2(d/\eps))\}$
    as a sufficiently large even integer so that
    $
        \int_{\u: \norm{\u}_2\ge\delta/2}r(\u)d\u<\eps/16.
    $
    Such a choice implies
    $\int_{\u: \norm{\u-\x}_2<\delta/2}r(\x-\u)d\u>1-\eps/16$.
    We claim that for every $\u$ with
    $\norm{\u-\x}_2<\delta/2$, we have $f^*(\u)=1$.
    For any such $\u$, either $f(\u)=1$ or $\u\in S_{\delta/2}$.
    Since $\norm{\u}_2<R+\delta/2$, by definition of $f^*$, we have $f^*(\u)=1$.
    This implies that $s(\x)\ge 1-\eps/16$ if $\norm{\x}_2\le R$ and $f(\x)=1$.
    So, for every $\x$ with $\norm{\x}_2\le R$, we have
    $f(\x)\le s(\x)+\eps/16$.

    We next design a polynomial $p(\x)$ to approximate $s(\x)$.
    Since $\rho_c(\x)$ is $c$-smooth, by Fact~\ref{fact taylor}, for every $\u\in\R^d$, there is a polynomial $p_\u(\x)$ with degree at most $k2^N$ such that
    \begin{align*}
        \abs{p_\u(\x)-r(\x+\u)}\le CE(\x),
    \end{align*}
    where $E$ is the nonnegative polynomial in Fact~\ref{fact taylor}.

    Define
    $p(\x)=\int_{\B(R+\delta/2)}p_{-\u}(\x)f^*(\u)d\u$.
    Since $f^*$ is supported on $\B(R+\delta/2)$, we have
    \begin{align*}
        \abs{p(\x)-s(\x)}
        \le
        \abs{\int_{\B(R+\delta/2)}
            (p_{-\u}(\x)-r(\x-\u))f^*(\u)d\u}
        \le AE(\x),
    \end{align*}
    where
    $A:=C\operatorname{Vol}(\B(R+\delta/2))\le C(O(R))^d$.
    By choosing
    $k=O(N+d\log N+d\log R+\log(1/\eps))$
    as a sufficiently large even integer, Fact~\ref{fact taylor} gives
    $A\norm{E}_{D_\x,1}\le\eps/16$.

    We choose
    \[
        P(\x):=
        p(\x)+AE(\x)+\eps/16
        +(\norm{\x}_2/R)^{2t}
    \]
    for some positive integer $t$ that will be chosen later.
    If $\norm{\x}_2\le R$, then
    $P(\x)\ge s(\x)+\eps/16\ge f(\x)$.
    If $\norm{\x}_2>R$, then $P(\x)\ge 1$, since
    $p(\x)+AE(\x)\ge s(\x)\ge 0$.
    So, $p^u:=P$ satisfies $p^u(\x)\ge f(\x)$ for every $\x\in\R^d$.

    It remains to upper bound $\norm{p^u-f}_{D_\x,1}$.
    We have
    \begin{align*}
        \norm{p^u-f}_{D_\x,1}
        &\le
        \norm{s-f}_{D_\x,1}
        +2A\norm{E}_{D_\x,1}
        +\eps/16
        +\norm{(\norm{\x}_2/R)^{2t}}_{D_\x,1}.
    \end{align*}

    We notice that unless $\norm{\x}_2>R$ or $\x\in S_\delta$, we have
    $\abs{s(\x)-f(\x)}<\eps/16$.
    Indeed, if $\norm{\x}_2\le R$ and $\x\notin S_\delta$,
    then for every $\u$ with $\norm{\u-\x}_2<\delta/2$,
    we have $f(\u)=f(\x)$ and $\u\notin S_{\delta/2}$:
    otherwise, the triangle inequality would give a point within
    distance $\delta$ of $\x$ with a different function value.
    Since $\norm{\u}_2<R+\delta/2$, this implies
    $f^*(\u)=f(\x)$, and hence
    \[
        \abs{s(\x)-f(\x)}
        \le
        \int_{\u:\norm{\u-\x}_2\ge\delta/2}r(\x-\u)d\u
        <\eps/16.
    \]
    Otherwise, $\abs{s(\x)-f(\x)}\le 1$.
    This implies that
    \[
        \norm{s-f}_{D_\x,1}
        \le
        \Pr_{\x\sim D_\x}[\x\in S_\delta]
        +\Pr_{\x\sim D_\x}[\norm{\x}_2>R]
        +\eps/16.
    \]

    By integrating the logconcave norm-tail bound
    $\Pr_{\x\sim D_\x}[\norm{\x}_2>u]\le 2^{O(d)}e^{-u/32}$,
    we obtain
    \[
        \left(\E_{\x\sim D_\x}[\norm{\x}_2^{2t}]\right)^{1/(2t)}
        \le C_0(d+t)
    \]
    for a universal constant $C_0\ge 1$.
    We choose
    $
        t=\left\lceil\log_2(16/\eps)\right\rceil,
        R=2C_0(d+t),
    $
    so that
    $
        \norm{(\norm{\x}_2/R)^{2t}}_{D_\x,1}
        \le 2^{-2t}\le\eps/16.
    $
    By Markov's inequality, this also gives
    $\Pr_{\x\sim D_\x}[\norm{\x}_2>R]\le\eps/16$.
    Therefore,
    \begin{align*}
        \norm{p^u-f}_{D_\x,1}
        \le
        \Pr_{\x\sim D_\x}[\x\in S_\delta]
        +6\eps/16
        \le
        \Pr_{\x\sim D_\x}[\x\in S_\delta]
        +\eps/2.
    \end{align*}

    Applying the same construction to $1-f$ gives an upper polynomial
    $\widetilde p^u$ for $1-f$; define $p^l:=1-\widetilde p^u$.
    Since $S_\delta(1-f)=S_\delta(f)$, we obtain
    $p^l\le f\le p^u$ pointwise and
    \[
        \E_{\x\sim D_\x}[p^u(\x)-p^l(\x)]
        \le
        2\Pr_{\x\sim D_\x}[\x\in S_\delta]+\eps.
    \]
    Both polynomials have degree at most
    $        \max\{k2^N,2t\}
        =
        \exp\left(
            O\left(d^4/\delta^4+d^2\log^2(d/\eps)\right)
        \right).
    $
    This concludes the proof.
\end{proof}


\subsection{Omitted Details from Theorem~\ref{thm:low-dim-PTFs}} 
In this section, we provide omitted details from
Theorem~\ref{thm:low-dim-PTFs}. To do this, we will show that
under any isotropic logconcave distribution, the probability
that a point is $\delta$-close to the decision boundary is small.
Such a result can be derived from the work
\cite{lotz2015volume}.

\begin{fact}[Theorem 1.1 in \cite{lotz2015volume}]
    Let $p(\x)$ be a degree $m\ge 1$ polynomial in $\R^d$ and
    let $S=\{\x\in\R^d\mid p(\x)=0\}$.
    Assume that $\nabla p(\x)\neq 0$ for every $\x\in S$.
    Let $D'$ be the uniform distribution over a ball
    $\B^{d-1}(\u,R)$. Then, for every $0<\delta\le R$,
    \begin{align*}
        \Pr_{\x\sim D'}
        \left[\inf_{\v\in S}\norm{\x-\v}_2\le\delta\right]
        &\le
        4\sum_{i=0}^{d-1}
        \binom{d}{1+i}
        (2m\delta/R)^{1+i}
        (1+\delta/R)^{d-1-i}\\
        &\le
        8dm2^{d-1}(m+1)^{d-1}\delta/R.
    \end{align*}
\end{fact}

\begin{lemma}\label{lm surfact area ptf}
    Let $p(\x)$ be a polynomial of degree at most $m\ge 1$
    in $\R^d$ and let $D_\x$ be an isotropic logconcave
    distribution over $\R^d$.
    For $\delta\in(0,1/2)$, let
    \[
    S_\delta=
    \{\u\in\R^d\mid
    \exists\v\in\R^d:
    \norm{\u-\v}_2\le\delta,\ 
    \sign(p(\u))\neq\sign(p(\v))\}.
    \]
    Then, for some universal constant $C>1$,
    \[
    \Pr_{D_\x}[S_\delta]
    \le
    (Cdm)^d\delta
    \left(\log\frac e\delta\right)^{d-1}.
    \]
\end{lemma}

\begin{proof}[Proof of Lemma~\ref{lm surfact area ptf}]
    If $p$ is constant, the statement is immediate.
    Otherwise, let $S=\{\x:p(\x)=0\}$.
    By continuity,
    $
    S_\delta\subseteq
    \{\x:\operatorname{dist}(\x,S)\le\delta\}.
    $

    To handle a possibly singular zero set $S$, choose positive
    regular values $a_j\downarrow 0$ of $p^2$, and let
    $S_j=\{\x:p(\x)^2=a_j\}$.
    The preceding fact applies to each $S_j$ with degree
    at most $2m$.
    Every $\v\in S$ satisfies
    $\operatorname{dist}(\v,S_j)\to 0$: every neighborhood
    of $\v$ contains a point where $p^2>0$, and continuity
    along the connecting segment gives the claim.
    Applying the fact with radius $\delta+\eta$, using
    Fatou's lemma, and then letting $\eta\downarrow 0$
    therefore gives the same bound for $S$, with $m$
    replaced by $2m$.
    This replacement is absorbed into the constants below.

    Let $f$ denote the density of $D_\x$.
    We use the standard bounds
    \[
    \sup_{\x\in\R^d}f(\x)\le(C_0\sqrt d)^d,
    \qquad
    \vol(\B^{d-1}(0,1))\le(C_0/\sqrt d)^d,
    \]
    where $C_0>0$ is universal.
    Thus, for every fixed $R>\sqrt d$, we have
    \begin{align*}
        \Pr_{\x\sim D_\x}[\x\in S_\delta]
        &\le
        \Pr_{\x\sim D_\x}[\norm{\x}_2>R]
        +
        \Pr_{\x\sim D_\x}
        [\norm{\x}_2\le R,\ \x\in S_\delta]\\
        &\le
        \Pr_{\x\sim D_\x}[\norm{\x}_2>R]
        +
        \sup_{\x\in\R^d}f(\x)\,
        \vol(S_\delta\cap\B^{d-1}(0,R))\\
        &\le
        \Pr_{\x\sim D_\x}[\norm{\x}_2>R]
        +
        C_1^dR^{d-1}(m+1)^{d-1}dm\delta,
    \end{align*}
    for a universal constant $C_1>1$.

    By choosing $R=\sqrt d\log(2e/\delta)$, the logconcave
    tail bound gives
    \[
    \Pr_{\x\sim D_\x}[\norm{\x}_2>R]
    \le \exp(1-R/\sqrt d)=\delta/2.
    \]
    Consequently,
    \begin{align*}
        \Pr_{D_\x}[S_\delta]
        \le
        \frac{\delta}{2}
        +
        C_1^d
        \left(\sqrt d\log\frac{2e}{\delta}\right)^{d-1}
        (m+1)^{d-1}dm\delta
        \le
        (Cdm)^d\delta
        \left(\log\frac e\delta\right)^{d-1},
    \end{align*}
    for some universal constant $C>1$.
\end{proof}



\end{document}